\documentclass{article}

\PassOptionsToPackage{numbers,sort&compress}{natbib}
\usepackage[main,final]{neurips_2026}

\usepackage[utf8]{inputenc}
\usepackage[T1]{fontenc}
\usepackage{microtype}
\usepackage{amsmath,amssymb,amsthm,mathtools}
\usepackage{mathrsfs}
\usepackage{bm}
\usepackage{enumitem}
\usepackage{xcolor}
\usepackage{hyperref}
\usepackage{booktabs}
\usepackage{multirow}
\usepackage{array}
\usepackage{graphicx}
\usepackage{caption}
\usepackage{subcaption}
\usepackage{float}
\usepackage{algorithm}
\usepackage{algorithmic}
\usepackage{url}
\usepackage{thmtools}
\usepackage{thm-restate}

\hypersetup{
  colorlinks=true,
  linkcolor=blue!50!black,
  citecolor=blue!50!black,
  urlcolor=blue!50!black,
  pdftitle={Deep Minds and Shallow Probes},
  pdfauthor={Su Hyeong Lee, Risi Kondor}
}
\allowdisplaybreaks

\newtheorem{theorem}{Theorem}[section]

\newtheorem{lemma}[theorem]{Lemma}
\newtheorem{corollary}[theorem]{Corollary}
\newtheorem{proposition}[theorem]{Proposition}
\newtheorem{definition}[theorem]{Definition}
\newtheorem{remark}[theorem]{Remark}

\newcommand{\R}{\mathbb{R}}
\newcommand{\C}{\mathbb{C}}
\newcommand{\E}{\mathbb{E}}

\newcommand{\Aff}{\operatorname{Aff}}
\newcommand{\GL}{\operatorname{GL}}
\newcommand{\Poly}[1]{\mathcal{P}_{\le #1}}
\newcommand{\norm}[1]{\left\lVert #1 \right\rVert}
\newcommand{\ip}[2]{\left\langle #1,#2 \right\rangle}
\newcommand{\Span}{\operatorname{span}}
\newcommand{\rank}{\operatorname{rank}}
\newcommand{\diag}{\operatorname{diag}}
\newcommand{\Ker}{\operatorname{ker}}
\newcommand{\Image}{\operatorname{Im}}

\newcommand{\Hom}{\operatorname{Hom}}

\newcommand{\eps}{\varepsilon}
\newcommand{\pmci}[2]{#1{\scriptsize$\,\pm\,$#2}}

\title{Deep Minds and Shallow Probes}
\author{%
  Su Hyeong Lee \\
  Department of Statistics \\
  University of Chicago
  \And
  Risi Kondor \\
  Departments of Statistics and Computer Science \\
  University of Chicago
}
\date{}

\begin{document}
\maketitle
{\renewcommand{\thefootnote}{\fnsymbol{footnote}}\footnotetext[1]{Corresponding author: \texttt{sulee@uchicago.edu}.}}

\begin{abstract}
Neural representations are not unique objects. Even when two systems realize the same downstream computation, their hidden coordinates may differ by reparameterization. A probe family intended to reveal structure already present in a representation should therefore be stable under the relevant representation symmetries rather than be tied to a particular basis. 
We prove that, under standard nondegeneracy assumptions, behaviorally equivalent LLMs have hidden representations related by an invertible affine transformation. The resulting symmetry principle singles out a unique hierarchy of shallow coordinate-stable probes, with linear probes as its degree-1 member.
We also show that a natural object for cross-model probe transfer is a shared probe-visible quotient--the representation modulo directions invisible to the probe family--rather than the full hidden state. Experiments on synthetic and real-world tasks support both predictions, showing where degree-2 probes help beyond linear ones and how quotient-based transfer enables coverage-aware monitor portability across model families. These results point toward a broader geometric representation theory of neural probing, with coverage-aware monitor transfer as a concrete operational consequence.
\end{abstract}

\vspace{-12pt}

\section{Introduction}

Probing has become one of the standard tools for studying learned representations in deep networks. In the simplest and most widely used setting, a probe is a lightweight classifier or regressor trained on frozen activations in order to predict a target property \citep{belinkov2022probing,conneau2018cram,hewitt-liang-2019-designing,pimentel2020pareto,pimentel-etal-2020-information,white-etal-2021-non}. Linear probes are especially common because they are hard to confuse with a large task-solving network in their own right \citep{hewitt-liang-2019-designing,pimentel2020pareto,belinkov2022probing}. This shallowness is a major practical virtue.

At the same time, neural representations are not unique objects. Even when two systems realize the same computation, their hidden coordinates may differ by reparameterizations. A probing family intended to reveal structure already present in the representation should therefore be stable under such reparameterizations, so that probe success is not an artifact of an arbitrary choice of coordinates. This raises a basic design question: what group action $G$ should arise naturally in neural representations, and how do $G$-stability and shallowness constrain admissible probe families?  Clearly, for an arbitrary universal approximator function class, a positive result may be driven by the probe's expressive power rather than by the representation under study \citep{hewitt-liang-2019-designing,pimentel2020pareto}.

Despite the widespread use of probes (Appendix~\ref{app:probing_related_work}), there is still no general framework, to our knowledge, for deriving probe families from first principles. Linear probes give one important answer, but they are not the only possible shallow family, and higher-order probes are often introduced without a clear invariance principle. Our goal is to replace this ad hoc design choice with a structural criterion that systematically derives admissible probe families from the induced group actions.

This paper isolates one clean principle that sharply constrains the answer. We study representation symmetries in the tractable exact setting commonly used across language and classification models, the final readout layer (i.e., head), where we prove that equivalent computations realize affine changes of hidden coordinates. Once one also insists that the probe remain shallow--formalized as a finite-dimensional linear score space--the admissible spaces are sharply constrained: the only nonzero affine-invariant scalar score spaces are bounded-degree polynomial spaces. Linear probes are thus not an isolated engineering default, but emerge as the degree-1 member of a unique polynomial hierarchy.

At degree~2, this classification naturally induces low-rank Canonical Polyadics (CP)~\citep{kolda2009tensor, dubey2022scalable} as a structured family, as CP rank is preserved by coordinate change whereas monomial sparsity is not (Proposition~\ref{prop:low_rank_invariance}). Across models it identifies a canonical transfer object: not the full hidden state, but only the directions a probe bank can see, namely the \emph{probe-visible quotient} \(Z(V)=H/K(V)\) (Theorem~\ref{thm:common_abstract_space}).

The experiments track these claims in order. Synthetic tasks establish that the degree hierarchy is tight, and on a cross-token agreement task, CP rank-1 improves over linear heads by 16.8--20.0pp AUROC across Pythia scales. 
Quotient alignment transfers safety monitors from Qwen-7B~\citep{yang2024qwen25} to Mistral-7B~\citep{jiang2023mistral} with no target labels while making coverage failures visible. 
These results point toward a geometric theory of neural probe design, in which different architectural or mechanistic representation symmetries give rise to different admissible shallow probe families. Additional related work appears in Appendix~\ref{app:probing_related_work}, and a detailed guide to the experimental appendices appears in Appendix~\ref{app:experiment_narrative}.

\paragraph{Related work.}
The modern probing literature begins with the use of linear classifiers on intermediate representations to measure what is linearly recoverable at each layer \citep{alain2016intermediate,conneau2018cram,liu-etal-2019-linguistic,tenney-etal-2019-bert,belinkov2022probing}. Subsequent work emphasized the importance of controlling probe complexity and of distinguishing information present in the representation from information learned by the probe itself \citep{hewitt-liang-2019-designing,pimentel2020pareto,pimentel-etal-2020-information,voita-titov-2020-information,belinkov2022probing}. Structural probing made this issue especially concrete: syntax is not naturally captured by a single linear direction, but can be expressed by a quadratic form after a learned linear change of basis \citep{hewitt2019structural,white-etal-2021-non}. Later work generalized this observation by proposing nonlinear structural probes with the same parameter count as the original quadratic probe \citep{white-etal-2021-non}. Information-theoretic views of probing likewise argue that richer probe classes reveal more of the structure stored in representations \citep{pimentel-etal-2020-information,pimentel2020pareto,voita-titov-2020-information}.

On the representation-comparison side, model stitching and representation equivalence studies ask when two hidden spaces can be linked by a simple map. Early work in vision used stitching layers to study representation equivalence empirically \citep{lenc2015equivalence,bansal2021stitching,balogh2023functional}. More recent work in language models shows that affine maps between residual streams can transfer probes, steering vectors, and sparse autoencoder parameters between models of different sizes in the same family \citep{chen2025transferring,bansal2021stitching,balogh2023functional}. At the same time, a recent critique of task-loss stitching argues that direct matching of activations can be a more trustworthy criterion than optimizing downstream task loss alone \citep{balogh2025stitch,raghu2017svcca,kornblith2019similarity}. We provide a theoretical foundation for this work by showing that models realizing the same family of linear concepts share a canonical probe-visible space. This space captures the structure accessible to the probes and provides a principled target for cross-model alignment.

\vspace{-5pt}
\section{Affine Invariance and the Polynomial Hierarchy}
\label{sec:probe_classification}
\vspace{-5pt}

Our goal is to develop a methodology that is not tied to a particular hand-specified attribute, but instead applies broadly to probes for arbitrary target properties. Thus, we deliberately leave the target property abstract: a probe is a mechanism for detecting \textit{some} property of a model, and we seek conditions on probes that hold uniformly across the \textit{arbitrary} choice of property. Rather than asking what properties a probe should detect (which should be as general as possible), we thus ask what structural conditions a probe must satisfy in order to be a coherent property-detector regardless of whichever property is being probed for.

\vspace{-5pt}
\subsection{A motivating thought experiment}
\label{sec:thought_experiment}
We formalize this principle through a simple thought experiment. Suppose two language models $A$ and $B$ produce identical outputs on all inputs. Since the models are behaviorally indistinguishable, any property that can be read off from $A$'s behavior can therefore be read off from $B$'s behavior, and vice versa. We should thus expect an ideal probe trained on $A$ to admit a counterpart trained on $B$. Yet the two models $A$ and $B$ may differ arbitrarily in architecture, initialization, and learned weights. The question then becomes: for so broad a class, what structure can we extract?

A natural place to begin is with architectural invariants. Although their internal representations may look quite different, both models ultimately pass their final hidden states through an affine readout map into vocabulary-logit space. Thus, at the final layer, the readout provides a common architectural interface through which model computations become behaviorally observable. Restricting our attention to this readout, we establish the following: if $A$ and $B$ are output-equivalent and share a linear readout structure, their penultimate representations are necessarily related by an affine transformation (Section~\ref{sec:readout}, Appendix~\ref{app:readout_proofs}). Consequently, any probe family intended to characterize properties of neural models, which includes final-layer representations, should be stable under this affine reparameterization. This symmetry principle is the starting point for our classification of stable admissible probe families.

\vspace{-5pt}
\subsection{Why affine? The readout-layer symmetry}
\label{sec:readout}
\vspace{-5pt}

For readers using different notational conventions, we provide a detailed setup in Appendix~\ref{app:setup}. Due to space, we provide only a concise overview of the main arguments here. Consider two realizations $(\gamma, \Lambda)$ and $(\gamma_*, \Lambda_*)$ of a bias-free linear readout, where $\gamma: D \to \R^d$ maps inputs to hidden states and $\Lambda \in \R^{n \times d}$ is the readout matrix. Here, \(D\) is the set of inputs being compared, and \(\gamma(x)\) is the final hidden representation computed by the network for input \(x\in D\). The matrix \(\Lambda\) is simply the network's final linear head, mapping that representation to the output-score vector \(\Lambda\gamma(x)\). Under standard nondegeneracy assumptions (Appendix~\ref{app:setup}), equality of the readout scores (i.e., behavioral equivalence of the neural net, meaning identical output scores for the same inputs) determines the hidden coordinates up to a unique invertible linear transformation (Theorem~\ref{thm:linear_head_equivalence}; full proof in Appendix~\ref{app:setup}).

\begin{samepage}
\begin{restatable}[Linear-head equivalence]{theorem}{thmLinearHead}
\label{thm:linear_head_equivalence}
Let \(\gamma,\gamma_*:D\to\R^d\) and \(\Lambda,\Lambda_*\in\R^{n\times d}\) define two bias-free readouts. Assume that \(\gamma(D)\) spans \(\R^d\) and both readout matrices have rank \(d\). If
\begin{equation}
\label{eq:pairing_preserved_linear}
\Lambda_*\gamma_*(x)=\Lambda\gamma(x)\qquad\text{for every }x\in D,
\end{equation}
then there is a unique \(A\in\GL(d)\) such that
\[
\gamma_*(x)=A^{-\top}\gamma(x),\qquad \Lambda_*=\Lambda A^\top.
\]
\end{restatable}
\end{samepage}

For softmax heads, adding the same value to every logit for a given input leaves the output probabilities unchanged. Adding the same vector $c$ to all readout vectors produces exactly this shared change in the logits, giving $\Lambda_* = \Lambda A^\top + \mathbf{1}c^\top$ (Proposition~\ref{prop:softmax_affine_symmetry}), whereas equality of the scores themselves gives $\Lambda_* = \Lambda A^\top$ (Theorem~\ref{thm:linear_head_equivalence}). A robust approximate variant of the theorem generalizing the thought experiment in Section~\ref{sec:thought_experiment} shows that approximate output agreement forces approximate linear alignment, with error controlled by the readout conditioning (Theorem~\ref{thm:robust_alignment}).

These results identify the \emph{linear} part of readout equivalence. We next allow the final readout to include a bias. A translation of the hidden coordinates by \(b\) can then be compensated by changing the readout biases, leaving every logit unchanged. Writing \(h=\gamma(x)\), the score for output \(i\) is \(s_i(h)=\ip{\lambda_i}{h}+\beta_i\), where \(\beta_i\) is its bias. We can express this as a linear readout by appending a constant \(1\) to the hidden representation and including the bias in the readout vector:
\[
\tilde h=(h,1),\qquad
\tilde\lambda_i=(\lambda_i,\beta_i),\qquad
s_i(h)=\ip{\tilde\lambda_i}{\tilde h}.
\]
Applying Theorem~\ref{thm:linear_head_equivalence} to these augmented representations in \(\R^{d+1}\) gives a linear change of coordinates. Because the added coordinate is always \(1\), this transformation can contribute a constant translation to the original hidden coordinates, yielding \(h\mapsto A^{-\top}h+b\) (Corollary~\ref{cor:affine_head_equiv} in Appendix~\ref{app:readout_proofs}).

Thus, models that are behaviorally equivalent in the sense of Section~\ref{sec:thought_experiment} must have final-layer representations which are related by an affine transformation. We therefore take the affine group $\Aff(d)$ as the symmetry relevant for probing.

A practitioner may train a probe on a single fixed checkpoint, but fixing that checkpoint does not by itself make its coordinates intrinsic to the represented information. To see why coordinate stability matters, let \(h'=g\circ h\), where \(g\in\Aff(d)\), and suppose a probe \(f\) reads a property from \(h\). The corresponding probe in the new coordinates is
\[
f'=f\circ g^{-1},\qquad f'(h'(x))=f(h(x)).
\]
The probe changes with the coordinates, while the score it assigns to each input remains unchanged. Closure of the probe family under this transformation therefore guarantees that every realizable score has an equivalent realization after reparameterization. Without this closure, a restriction on the probe family can make a property appear accessible in one coordinate system but inaccessible in another, even though the represented information is unchanged. For probes intended to support coordinate-independent claims about a representation, this motivates stability under affine precomposition. We now ask which finite-dimensional probe spaces satisfy this requirement.

\vspace{-5pt}
\subsection{Deep minds and shallow probes}
\label{subsec:probe_classification}
\vspace{-5pt}

A core methodological aim of probing is to keep the probe \emph{shallow}. If the probe family is too rich, probe success may primarily reflect the probe's own learning capacity. Finite-dimensionality provides a clean mathematical way to enforce bounded capacity: a finite-dimensional linear family has a finite number of effective degrees of freedom, cannot represent arbitrary decision boundaries, and is far from a universal approximator.

This point matters because affine invariance by itself is too weak. Infinite-dimensional affine-invariant spaces abound. For example, the full space \(C(\R^n)\) is affine-invariant but far too expressive to support a faithful notion of a shallow readout. The theorem below therefore combines two requirements:
\begin{itemize}[leftmargin=2em]
    \item \textbf{symmetry:} the family should be stable under affine changes of representation coordinates; and
    \item \textbf{simplicity:} the family should live in a finite-dimensional linear space (i.e., be shallow).
\end{itemize}

Let \(g:\R^n\to\R^n\) be an affine map of the form \(g(x)=Ax+b, A\in\GL(n),\; b\in\R^n.\) We let affine maps act on functions by \emph{precomposition}: \((g\cdot f)(x):=f(g(x)).\)

\begin{definition}[Affine-invariant scalar probe space]
A linear subspace \(V\subset C(\R^n)\) is \emph{affine-invariant} if
\(f\in V\text{ and } g\in\Aff(n) \) imply \(f\circ g\in V.\)
\end{definition}

Whenever topological arguments are used in this section and in Appendix~\ref{app:probe_proofs}, the space \(C(\R^n)\) is equipped with the compact-open topology, equivalently the topology of uniform convergence on compact subsets of \(\R^n\). 

\begin{restatable}[Classification of scalar affine-invariant probe spaces]{theorem}{thmScalarClassification}
\label{thm:scalar_affine_classification}
Let \(V\subset C(\R^n)\) be a finite-dimensional affine-invariant linear subspace. Then either \(V=\{0\}\), or there exists an integer \(\ell\ge 0\) such that
\(V=\Poly{\ell}(\R^n),\) that is, \(V\) is exactly the space of all real polynomials on \(\R^n\) of total degree at most \(\ell\).
Conversely, the zero space and each \(\Poly{\ell}(\R^n)\) are finite-dimensional and affine-invariant.
\end{restatable}

In other words, once one insists on finite-dimensionality and affine invariance, the only nonzero scalar probe nonlinearities are bounded-degree polynomials. 
Polynomiality under related invariance assumptions has precedents in the theory of finite-dimensional translation-invariant characterizations \citep{almira2025characterization}. The new point is the connection to probing: the theorem turns a representational symmetry of neural readouts into a probe-design principle, explaining linear probes as degree $1$ and higher-order polynomial probes as the only finite-dimensional affine-stable extensions.
We now present the following corollary.

\begin{restatable}[Vector-valued probe families]{corollary}{corVectorValued}
\label{cor:vector_valued_probe_family}
Let \(\mathcal F\subset C(\R^n,\R^k)\) be a finite-dimensional affine-invariant linear subspace under the same precomposition action. Then there exists \(\ell\ge 0\) such that every coordinate of every \(f\in\mathcal F\) is a real polynomial of total degree at most \(\ell\), or equivalently,
\(\mathcal F\subset \Poly{\ell}(\R^n,\R^k).\)
In particular, vector-valued admissible shallow probes must have polynomial coordinates of uniformly bounded degree.
\end{restatable}
The scalar and vector-valued cases differ in one important respect. In the scalar case, affine invariance determines the space completely: every nonzero finite-dimensional affine-invariant space is exactly \(\Poly{\ell}(\R^n)\). In the vector-valued case, affine invariance fixes the dependence on the input \(x\), but it does not determine which output directions in \(\R^k\) are allowed. A vector-valued affine-invariant family is therefore only enforced to be a linear subspace of \(\Poly{\ell}(\R^n,\R^k)\).

\vspace{-5pt}
\subsection{Parameter counts and structured polynomial probes}
\vspace{-5pt}

A general polynomial probe \(f:\R^n\to\R^k\) of total degree at most \(\ell\) can be written as
\(
 f(x)=\sum_{|\alpha|\le \ell} c_{\alpha} x^{\alpha},
\)
where \(\alpha=(\alpha_1,\dots,\alpha_n)\) is a multi-index and \(x^{\alpha}=x_1^{\alpha_1}\cdots x_n^{\alpha_n}\). The number of free scalar coefficients in a vector-valued degree-\(\ell\) polynomial probe is
\(k\binom{n+\ell}{\ell},\)
which is large enough that unrestricted high-degree probes quickly become impractical. This motivates structured subclasses. One especially natural class is the family of low-rank homogeneous tensor probes~\citep{kolda2009tensor, dubey2022scalable}.
\begin{definition}[Rank-\(R\) homogeneous CP probe]
Fix an integer \(m\ge 1\). A scalar homogeneous polynomial of degree \(m\) is a \emph{rank-\(R\) CP probe} if for \( \alpha_r\in\R\) and \(v_{r,j}\in\R^n,\) it has the functional form
\(p(x)=\sum_{r=1}^R \alpha_r \prod_{j=1}^m \ip{v_{r,j}}{x}\).
\end{definition}

\begin{restatable}[Low-rank structure is stable under linear reparameterization]{proposition}{propLowRank}
\label{prop:low_rank_invariance}
For fixed degree \(m\) and rank bound \(R\), the family of rank-\(R\) homogeneous CP probes is closed under precomposition by \(\GL(n)\). In contrast, monomial sparsity is not preserved in general under precomposition by \(\GL(n)\).
\end{restatable}
This proposition is the reason CP is used throughout the empirical analysis. Low-rank tensor structure is compatible with the symmetry principle: precomposition by an invertible linear map preserves the CP-rank bound. Monomial sparsity, by contrast, depends on the chosen basis. To recover full \(\Aff(n)\)-stability one adds lower-degree terms, giving the \emph{affine-completed} parameterization
\begin{equation}
\label{eq:affine_completed_cp}
f(x) = \sum_{r=1}^{R}\alpha_r\prod_{j=1}^{m}(\ip{v_{r,j}}{x}+b_{r,j}) + g(x), \qquad g\in\Poly{m-1}(\R^n).
\end{equation}
Since each factor \(\ip{v_{r,j}}{Ax+t}+b_{r,j}\) is again affine in \(x\), the affine-completed family is closed under \(\Aff(n)\), not just \(\GL(n)\). The degree-2 instance is the parameterization used in the experiments. Appendix~\ref{subsubsec:exact_reparam} verifies the distinction directly: analytically transported full quadratics reproduce scores to machine precision, while sparse monomial probes fail under the same coordinate change. Appendix~\ref{subsubsec:exp_a} gives the regression version of the same comparison, and Appendix~\ref{subsubsec:cp_rank} quantifies the CP compression--accuracy trade-off.

\vspace{-10pt}

\section{A Shared Abstract Space for Transferable Probes}
\label{sec:shared_space}

\vspace{-5pt}

The preceding section used behavioral equivalence as a symmetry principle. We now use the same thought experiment in Section~\ref{sec:thought_experiment} to address a second question. Suppose models $A$ and $B$ agree on their behavioral outputs. If so, in what sense do they represent the same concepts? More precisely, can we identify a shared space of probe-detectable properties that is independent of the hidden-state coordinates of either model? For affine probe families, the answer is affirmative. We construct an abstract concept dual space $\mathcal{C}^*$ and show that the probe-visible quotient of each model maps canonically to this same space. Consequently, the concepts detected by affine probes in $A$ and $B$ are not merely related by an arbitrary change of basis; they are two coordinate presentations of a single underlying object.

This gives a coordinate-free account of cross-model concept equivalence. Rather than comparing hidden states directly, we compare their images in the shared probe-visible quotient. The resulting canonical isomorphism identifies which probe-detectable concepts are the same across behaviorally equivalent models, while avoiding any commitment to either model's internal coordinates. We also prove approximate analogues of these results. When two models only approximately agree in their output behavior, their probe-visible quotients remain approximately aligned under quantitative nondegeneracy conditions (Appendix~\ref{app:shared_space_proofs}, Theorems~\ref{thm:approx_shared_space}-\ref{cor:approx_margin_transfer}). Thus, the exact canonical isomorphism obtained in the idealized setting has a stable approximate counterpart, which supports the use of these quotient spaces for comparing and transferring concepts across realistic, non-identical models.

\vspace{-5pt}
\subsection{Concept families and probe-visible quotients}
\label{subsec:concept_families}
\vspace{-5pt}

Let \(X\) be any set of inputs. For \(i\in\{1,2\}\), let \(h_i:X\to H_i\) be a representation map into a finite-dimensional real vector space \(H_i\) (e.g., a representation space of model $i$). Let \(V_i\subset H_i^*\) be a finite-dimensional space of linear probes. The concept family realized by \((h_i,V_i)\) is
\[
 \mathcal C_i:=\{\ell\circ h_i : \ell\in V_i\}\subset \R^X.
\]
To rule out redundant probe directions and ensure that distinct probes correspond to distinct concepts on the represented inputs, we assume that no nonzero probe in \(V_i\) vanishes on every represented input, so that the evaluation map \(E_i:V_i\to \mathcal C_i,\) \(E_i(\ell)=\ell\circ h_i\) is injective.

\begin{definition}[Probe-invisible subspace and probe-visible quotient]
\label{def:probe_invisible}
For a probe space \(V\subset H^*\), define the probe-invisible subspace
\(K(V):=\bigcap_{\ell\in V} \Ker \ell.\)
The corresponding \emph{probe-visible quotient} is
\(Z(V):=H/K(V).\)
\end{definition}

The quotient removes exactly the directions on which every probe in \(V\) is blind.

\begin{restatable}[Evaluations span the dual of a finite-dimensional concept family]{lemma}{lemEvaluations}
\label{lem:evaluations_span}
Let \(\mathcal C\subset \R^X\) be finite-dimensional. For each \(x\in X\), define the evaluation functional
\(\operatorname{ev}_x:\mathcal C\to\R,\) \(\operatorname{ev}_x(c)=c(x).\) Then the set \(\{\operatorname{ev}_x:x\in X\}\) spans \(\mathcal C^*\).
\end{restatable}

\begin{restatable}[Common abstract probe space]{theorem}{thmCommonSpace}
\label{thm:common_abstract_space}
Assume that two pairs \((h_1,V_1)\) and \((h_2,V_2)\) realize the same finite-dimensional concept family,
\(\mathcal C_1=\mathcal C_2=:\mathcal C,
\) and that the evaluation maps \(E_1\) and \(E_2\) are injective. Then the following statements hold.
\begin{enumerate}[leftmargin=2em]
    \item For each \(i\in\{1,2\}\), there exists a unique linear map \(Q_i:H_i\to \mathcal C^*\)
    satisfying
    \(Q_i(v)(E_i(\ell))=\ell(v)
    \text{ for all }v\in H_i\text{ and all }\ell\in V_i.
    \) In particular,
    \(Q_i(h_i(x))=\operatorname{ev}_x
    \text{ for all }x\in X.
    \)
    \item The kernel of \(Q_i\) is exactly the probe-invisible subspace \(K(V_i)\).
    \item The induced map \(\overline Q_i: Z(V_i)=H_i/K(V_i)\to \mathcal C^*\)
    is a linear isomorphism.
\end{enumerate}
Consequently the probe-visible quotients of the two models are canonically isomorphic to the same abstract space \(\mathcal C^*\). In particular,
\(H_1/K(V_1)\cong \mathcal C^* \cong H_2/K(V_2).\) Because \(\mathcal C\) is finite-dimensional, the canonical bidual identification also gives
\(\mathcal C\cong (\mathcal C^*)^*.\)
\end{restatable}

\begin{remark} The interpolation property \(Q_i(h_i(x))=\operatorname{ev}_x\) is a consequence of the theorem, but by itself it does not determine \(Q_i\) on all of \(H_i\) unless one adds an extra spanning assumption such as \(\Span\,h_i(X)=H_i\). The displayed formula with \(E_i(\ell)\) is the canonical and fully well-posed definition.
\end{remark}

\paragraph{Interpretation.}
The theorem says that linear probes do not require the full hidden states of two models to align. They only require the \emph{probe-visible quotients} to align. Once a finite-dimensional concept family is fixed, both models project to the same abstract space of evaluations. This is the precise sense in which a low-complexity linear map can link two different representation spaces.
In practice, if a bank of $k$ linear probes is stacked into a weight matrix $W \in \mathbb{R}^{k \times d}$, then the probe-invisible subspace is exactly $\operatorname{ker}(W)$, and the quotient $H / K(V)$ can be realized by coordinates on the probe-visible directions. A numerically stable way to obtain these coordinates is to compute an SVD $W=U^\prime \Sigma^\prime {V^\prime}^{\top}$ and retain the right singular directions associated with the thresholded nonzero singular values. This yields the quotient projection used in our experiments: it removes directions invisible to the probe bank while preserving the directions on which the probes actually depend.

\begin{restatable}[Transporting probes through the shared space]{corollary}{corTransporting}
\label{cor:transporting_probes}
Under the hypotheses of Theorem~\ref{thm:common_abstract_space}, every concept \(c\in\mathcal C\) corresponds to a unique linear functional on the common space \(\mathcal C^*\). Thus a probe trained in one model can be transferred to the other model by expressing it as a linear functional on \(\mathcal C^*\) and pulling it back through the quotient map of the target model.
\end{restatable}

\paragraph{Approximate finite-bank realization.}
Theorem~\ref{thm:common_abstract_space} is exact and basis-free, but deployment uses a \emph{finite} bank of learned probes together with an SVD realization of its visible quotient. Appendix~\ref{app:shared_space_proofs} translates the exact theorem into that setting. The resulting bounds say that quotient-coordinate error is controlled by concept-family mismatch divided by the smallest visible singular value of the source bank (Theorem~\ref{thm:approx_shared_space}), and that the induced transport inherits the conditioning of both banks (Proposition~\ref{prop:approx_transport_conditioning}). Two practical failure modes therefore matter: the source and target may disagree on the bank-visible concepts, or the bank itself may be poorly conditioned. Near-duplicate probe directions are problematic not because they increase probe count, but because they barely enlarge the visible span while shrinking \(\sigma_{\min}^+(W)\). This is why the transfer experiments emphasize coverage and diversity of the probe bank rather than raw bank size.

\vspace{-5pt}
\section{Experiments}
\label{sec:experiments}
\vspace{-2pt}

The experiments are organized around the two main consequences of the theory: the polynomial degree hierarchy with its coordinate-stable structured families, and quotient-based transfer. Appendix~\ref{app:experiment_details} gives full implementation details, while Appendix~\ref{app:extended_experiments} expands each block with additional diagnostics. Results are averaged over 5 seeds \(\{42,137,256,314,999\}\). AUROC is the primary classification metric; balanced accuracy is reported where it is more informative.

\begin{figure}[t]
\centering
\includegraphics[width=\textwidth]{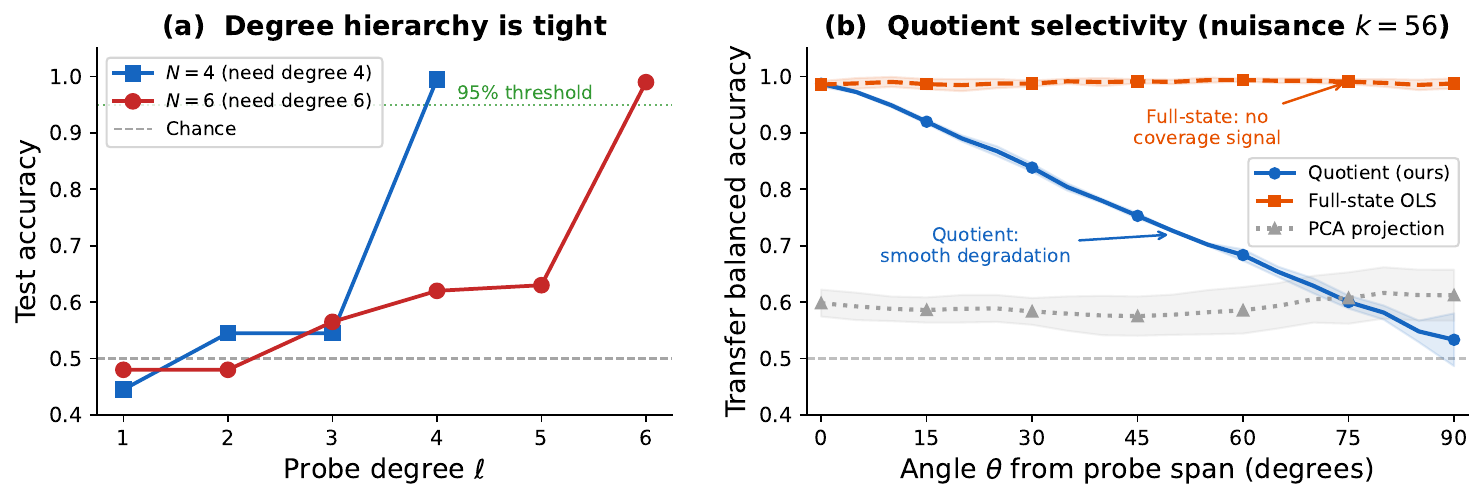}
\caption{Two predictions of the paper's framework. (a)~The polynomial degree hierarchy is tight: on circular $N$-parity, degree $<N$ is at chance while degree $N$ succeeds (Theorem~\ref{thm:scalar_affine_classification}). (b)~Quotient transfer is selective: as a held-out concept rotates from in-span ($\theta{=}0^\circ$) to out-of-span ($\theta{=}90^\circ$), quotient transfer degrades smoothly while full-state OLS stays at ${\sim}0.99$ with no coverage signal. PCA collapses under nuisance. Full setup and per-angle table in Appendix~\ref{subsec:continuum} (Table~\ref{tab:theta_sweep}). Shaded regions: $\pm 1$ std over 5 seeds.}
\label{fig:main}
\vspace{-12pt}
\end{figure}

\vspace{-5pt}
\subsection{The polynomial degree hierarchy}
\label{subsec:degree_hierarchy}
\vspace{-5pt}

\paragraph{Synthetic degree validation.} We begin with constructions whose required degree is known exactly. In a random affine embedding of \(a,b\in\{-1,+1\}\) into \(\R^{64}\), the XOR/equality target \(y=\mathbf{1}[a=b]\) is genuinely degree~2: a linear probe stays at chance (\(0.501\pm0.010\) balanced accuracy), whereas a full quadratic probe reaches \(1.000\pm0.000\) across 5 seeds, with no degradation under affine reparameterization (Appendix~\ref{app:degree_validation_appendix}). On circular \(N\)-parity (Figure~\ref{fig:main}a), the hierarchy is tight at every degree: degree \(N-1\) is at chance for every \(N\in\{2,\dots,8\}\), while degree \(N\) reaches \(0.985\)--\(0.995\) on held-out points. This is not a sample-efficiency artifact; degree \(N-1\) cannot even interpolate the \(2N\) training points. Appendix~\ref{app:degree_validation_appendix} also sharpens the XOR claim via $\Poly{2}\setminus\Poly{1}$ and the affine-stability of $\Poly{1}$, and recovers the minimum probe degree on five Boolean tasks of known degree across multiple layers of Pythia-160m and Pythia-410m. Appendix~\ref{subsubsec:exp_a} provides the regression analogue.

\begin{table}[h]
\vspace{-10pt}
\centering
\small
\caption{Cross-token score-space composition (AUROC, mean $\pm$ std over 10 resamples). On subject--verb agreement, a degree-2 cross-token target by construction, CP rank-1 outperforms the linear head by 16.8--20.0pp across all three Pythia scales; the gap persists on a harder XOR variant. Full per-target breakdown in Appendix~\ref{subsec:cross_tok_appendix}.}
\label{tab:cross_tok_main}
\begin{tabular}{ll cccc}
\toprule
Model & Target & Linear & Full quad & CP-1 & Gap (CP-1 -- lin) \\
\midrule
70m & agreement & \pmci{.733}{.005} & \pmci{.819}{.013} & \pmci{.901}{.004} & $+$16.8pp \\
70m & agree\_XOR\_v.past & \pmci{.635}{.009} & \pmci{.684}{.016} & \pmci{.798}{.007} & $+$16.3pp \\
\midrule
160m & agreement & \pmci{.740}{.003} & \pmci{.855}{.007} & \pmci{.930}{.002} & $+$19.0pp \\
160m & agree\_XOR\_v.past & \pmci{.612}{.010} & \pmci{.689}{.014} & \pmci{.800}{.006} & $+$18.8pp \\
\midrule
410m & agreement & \pmci{.755}{.004} & \pmci{.872}{.011} & \pmci{.955}{.003} & $+$20.0pp \\
410m & agree\_XOR\_v.past & \pmci{.614}{.009} & \pmci{.672}{.013} & \pmci{.803}{.011} & $+$18.9pp \\
\bottomrule
\end{tabular}
\vspace{-12pt}
\end{table}

\paragraph{Cross-token agreement.} Subject--verb agreement is a degree-2 cross-token target by construction because agreement in number is determined by the product of number features evaluated at the subject and verb tokens. To probe this, we train per-token linear probes on Universal Dependencies English-EWT, concatenate the subject and verb scores into a 60-dimensional score vector $S$ (30 probes $\times$ 2 tokens), and fit composition heads on $S$. Because a polynomial head on linear-probe scores is itself a polynomial probe on the original hidden state, the classification theorem applies to these heads just as it does to ordinary readouts (Section~\ref{sec:probe_classification}). All pairs use cross-sentence pairing for both classes to avoid sentence-identity confounds.

CP rank-1 outperforms the linear head on agreement by 16.8--20.0pp across all three Pythia scales, and by 16.3--18.9pp on the harder XOR variant where labels combine agreement with verb tense. Full quadratic on the 60D score space gains over linear but lags CP-1 by 7--13pp, consistent with capacity-driven overfitting: full-quadratic training AUROC saturates at $\geq 0.995$ on every model scale, yielding a 5--11$\times$ larger train-test gap than CP rank-1 (Appendix~\ref{subsec:cross_tok_appendix}, Table~\ref{tab:cross_tok_traintest}). Raw-activation baselines on full $2d$-dimensional hidden states (linear concat and bilinear; Appendix~\ref{subsec:cross_tok_appendix}) underperform score-space CP-1 by 20--32pp on agreement, confirming that composition through primitive probe scores is both more effective and orders of magnitude cheaper than fitting interaction probes in the full hidden space. The complete per-target breakdown including AND targets and 3-way conjunctions appears in Appendix~\ref{subsec:cross_tok_appendix}. Degree-2 interactions also arise in safety monitoring. We define policy mismatch as the XOR of prompt-side risk and response-side refusal -- i.e., harmful compliance or benign over-refusal. On this 2D score-space target, CP rank-1 and full quadratic reach AUROC $0.976$ and $0.974$ against $0.906$ for the linear head, with non-overlapping bootstrap intervals (Appendix~\ref{app:degree2_composition}, Table~\ref{tab:degree2_composition}).

\paragraph{Affine invariance and basis stability.}
Proposition~\ref{prop:low_rank_invariance} predicts a sharp dichotomy under affine reparameterization: the complete degree-2 family is closed under $\Aff(d)$, so a fitted full quadratic can be transported analytically to any new basis without retraining, while monomial sparsity is basis-dependent and a sparse representation becomes dense under coordinate change. Appendix~\ref{subsubsec:exact_reparam} verifies this on a degree-2 regression target $y = \ip{a}{z}\ip{b}{z} + \ip{c}{z} + \eps$ in $\R^{64}$: across 20 random affine transforms, the analytically transported full quadratic reproduces test scores to machine precision (maximum error $1.14\times 10^{-11}$), while a top-50-monomial sparse probe with its sparsity pattern frozen in the original basis collapses to $R^2 = -18.3 \pm 3.5$.

\vspace{-5pt}
\subsection{Quotient-space transfer}
\label{subsec:transfer_experiments}
\vspace{-5pt}

\paragraph{Synthetic quotient transfer (stability and selectivity).}
To test quotient transfer's stability and selectivity (Theorem~\ref{thm:common_abstract_space}), we embed a shared latent concept $c\in\R^8$ into source ($d_s=64$) and target ($d_t=128$) spaces with growing nuisance:
\begin{align}
\label{eq:quotient_embed}
    h_s &= A_s\, c + B_s\, n_s + \eps_s, &\quad A_s\in\R^{64\times 8},\;\; B_s\in\R^{64\times k_s},\;\; n_s\sim\mathcal{N}(0,I_{k_s}), \\
    h_t &= A_t\, c + B_t\, n_t + \eps_t, &\quad A_t\in\R^{128\times 8},\;\; B_t\in\R^{128\times k_t},\;\; n_t\sim\mathcal{N}(0,I_{k_t}),
\end{align}
with $\sigma=0.01$ noise and $k_s=k_t\in\{0,8,24,56\}$.
We train 5 primitive probes on the source, build the quotient via SVD, and transfer to the target via Ridge alignment on 5{,}000 unlabeled paired synthetic activations.
Two in-span concepts and one out-of-span concept are held out.
We compare four alignment methods: quotient Ridge, full-state OLS, PCA projection, and random projection.

\vspace{-10pt}

\begin{table}[h]
\centering
\small
\caption{Transfer balanced accuracy as nuisance dimension grows (0 to 56, 5 seeds). Quotient transfer is both \emph{stable} (in-span accuracy stays at 0.99) and \emph{selective} (out-of-span correctly rejected at chance). PCA collapses under nuisance; full-state OLS transfers everything including out-of-span concepts.}
\label{tab:exp_h}
\begin{tabular}{l cc cc}
\toprule
& \multicolumn{2}{c}{In-span concept} & \multicolumn{2}{c}{Out-of-span concept} \\
\cmidrule(lr){2-3} \cmidrule(lr){4-5}
Method & $k{=}0$ & $k{=}56$ & $k{=}0$ & $k{=}56$ \\
\midrule
Quotient (Ridge) & \pmci{0.997}{.002} & \pmci{0.991}{.004} & \pmci{0.484}{.020} & \pmci{0.519}{.023} \\
Full-state OLS & \pmci{0.996}{.001} & \pmci{0.990}{.003} & \pmci{0.995}{.001} & \pmci{0.989}{.006} \\
PCA projection & \pmci{0.996}{.001} & \pmci{0.573}{.036} & \pmci{0.996}{.001} & \pmci{0.585}{.022} \\
Random projection & \pmci{0.927}{.029} & \pmci{0.574}{.018} & \pmci{0.903}{.065} & \pmci{0.567}{.027} \\
\bottomrule
\end{tabular}
\vspace{-4pt}
\end{table}

Table~\ref{tab:exp_h} highlights the two properties the theory predicts. First, \textbf{stability}: quotient transfer stays above \(0.99\) accuracy even as nuisance dimension grows \(7\times\), whereas PCA collapses because it keeps variance-dominant directions rather than probe-relevant ones. Second, \textbf{selectivity}: the quotient rejects out-of-span concepts at chance (\(\text{bacc}\in[0.48,0.52]\)), while full-state OLS transfers them at 0.989.

That selectivity is the practical value of the quotient. Since \(Z(V)=H/K(V)\) removes directions orthogonal to every probe weight vector, concepts lying largely outside the bank are projected away and transfer is correctly expected to deteriorate. On real data the relationship is only approximate--some low-ISF concepts still transfer because they correlate with bank members (Section~\ref{subsec:continuum})--but the quotient provides a deployment-time coverage diagnostic that full-state alignment does not.

\paragraph{Safety monitor portability across model families.}
The central practical question is whether a monitor trained in one model family can be moved to another with no target labels. We train source probes on Qwen-2.5-7B-Instruct (hereafter Qwen-7B) for toxicity, jailbreaking, harmful content, sentiment, and moderation, align source and target quotients on \emph{unlabeled} paired activations, and evaluate only on held-out target labels. Table~\ref{tab:track_k} reports the full per-target comparison, including Qwen-2.5-14B (Qwen-14B), Qwen-2.5-Coder-7B (Qwen-Coder), and a random-initialization control.

\begin{table}[h]
\vspace{-10pt}
\centering
\footnotesize
\caption{Zero-label safety monitor transfer (AUROC). Source: Qwen-2.5-7B-Instruct, 11-probe bank (5 core concepts + 6 BeaverTails harm categories, $k{=}11$, cond${}\approx 5.8$). Bold entries highlight cross-architecture transfer. 95\% bootstrap CI half-widths ($\pm$); source AUROC is in-model.}
\label{tab:track_k}
\begin{tabular}{l c ccccc}
\toprule
Concept & Source & Qwen-3B & Qwen-14B & Coder & Mistral & Random \\
\midrule
Toxicity & 0.976 & \pmci{0.952}{.01} & \pmci{0.955}{.01} & \pmci{0.971}{.01} & \textbf{\pmci{0.926}{.01}} & \pmci{0.580}{.03} \\
Jailbreaking & 0.974 & \pmci{0.932}{.02} & \pmci{0.926}{.03} & \pmci{0.838}{.04} & \textbf{\pmci{0.669}{.05}} & \pmci{0.522}{.06} \\
Harmful (BT) & 0.791 & \pmci{0.772}{.03} & \pmci{0.781}{.03} & \pmci{0.793}{.03} & \textbf{\pmci{0.773}{.03}} & \pmci{0.632}{.03} \\
Sentiment & 0.975 & \pmci{0.966}{.01} & \pmci{0.974}{.01} & \pmci{0.967}{.01} & \textbf{\pmci{0.972}{.01}} & \pmci{0.511}{.04} \\
Moderation & 0.890 & \pmci{0.836}{.05} & \pmci{0.857}{.04} & \pmci{0.867}{.04} & \textbf{\pmci{0.820}{.05}} & \pmci{0.512}{.07} \\
\bottomrule
\end{tabular}
\vspace{-5pt}
\end{table}
Most importantly, the transfer extends across architecture families. Qwen-7B\(\to\)Mistral reaches AUROC 0.67--0.97 across the five concepts with the well-conditioned 11-probe bank, while the random-initialization control is near chance on four of the five. Within-family transfer is stronger still (Qwen-3B: 0.77--0.97; Qwen-14B: 0.78--0.97). Cross-architecture transfer is nontrivial but heterogeneous across concepts. Sentiment and several toxicity/harmfulness monitors remain strong, but some transfers are substantially weaker; for example, Qwen-7B\(\to\)Mistral jailbreaking attains AUROC 0.669, the lowest across the five concepts. We therefore interpret quotient alignment as evidence of partial portability of bank-supported monitor directions across model families, not as a claim that every transferred monitor is sufficient for deployment without additional validation.

A practically important consequence is label efficiency. Zero-label quotient transfer achieves 0.93--0.97 AUROC on toxicity and sentiment, matching or exceeding probes trained from scratch on 500 target labels (Table~\ref{tab:label_eff} in Appendix~\ref{app:safety_transfer_table}). Even adding 25 target labels for adaptation rarely improves on the zero-label baseline. Appendix~\ref{subsec:safety_transfer_additional} clarifies the conditions behind Table~\ref{tab:track_k} by showing that performance improves with more unlabeled alignment data. The approximate finite-bank theory in Appendix~\ref{app:shared_space_proofs} provides a quantitative explanation of these patterns.

\vspace{-10pt}

\begin{table}[h]
\centering
\small
\caption{Real-data selectivity: quotient vs.\ full-state OLS transfer on held-out concepts (AUROC, mean{\scriptsize$\,\pm\,$std} over 5 seeds). Full-state transfers held-out concepts at 0.77--0.95, often indistinguishable from genuine coverage. The quotient correctly degrades on low-ISF concepts.}
\label{tab:selectivity_main}
\begin{tabular}{lc cc cc}
\toprule
& & \multicolumn{2}{c}{Quotient} & \multicolumn{2}{c}{Full-state OLS} \\
\cmidrule(lr){3-4} \cmidrule(lr){5-6}
Concept & ISF & Qwen-3B & Mistral & Qwen-3B & Mistral \\
\midrule
\multicolumn{6}{l}{\emph{In-bank (should transfer):}} \\
\quad toxicity & --- & \pmci{.959}{.004} & \pmci{.938}{.013} & \pmci{.959}{.004} & \pmci{.938}{.013} \\
\quad sentiment & --- & \pmci{.965}{.001} & \pmci{.972}{.002} & \pmci{.965}{.001} & \pmci{.972}{.002} \\
\midrule
\multicolumn{6}{l}{\emph{Held-out (should degrade):}} \\
\quad privacy\_violation & 0.014 & \pmci{.681}{.062} & \pmci{.677}{.057} & \textbf{\pmci{.952}{.002}} & \textbf{\pmci{.929}{.018}} \\
\quad sexually\_explicit & 0.020 & \pmci{.760}{.027} & \pmci{.777}{.039} & \textbf{\pmci{.872}{.009}} & \textbf{\pmci{.903}{.015}} \\
\quad overall harmful & 0.124 & \pmci{.758}{.005} & \pmci{.719}{.027} & \pmci{.770}{.000} & \pmci{.768}{.000} \\
\bottomrule
\end{tabular}
\vspace{-10pt}
\end{table}

\paragraph{Real-data coverage diagnostics: quotient vs.\ full-state.}

The synthetic experiment established selectivity in a controlled setting. We further investigate if the same distinction survives on safety data. We build a probe bank from 8 concepts (6 BeaverTails harm categories plus toxicity and sentiment) and hold out 3 concepts that are \emph{not} in the bank: privacy violation, sexually explicit content, and overall harmful/safe. For each held-out concept we compute its \emph{in-span fraction} (ISF)--the fraction of the probe direction lying in the quotient span--and transfer via both quotient Ridge and full-state OLS.

Table~\ref{tab:selectivity_main} captures the epistemic advantage of quotient transfer. On in-bank concepts, quotient and full-state alignment are identical. On held-out concepts with low ISF, however, full-state OLS still transfers at 0.77--0.95--numbers that can look indistinguishable from genuine coverage--whereas the quotient drops to 0.67--0.77. In other words, full-state alignment can mask a coverage failure that the quotient exposes. Low ISF is not a perfect predictor, but it is an informative one, and it comes from the same quotient construction rather than from a separate post-hoc diagnostic. Appendix~\ref{subsec:continuum} shows that this ISF--transfer relationship extends smoothly across a leave-one-out analysis of 13 concepts.

\paragraph{Coverage-aware abstention as a deployment benchmark.}
To measure selectivity as a deployment criterion, we evaluate 21 concepts (11 in-bank + 10 held-out) against three bank conditions on three target architectures (Appendix~\ref{app:coverage_abstention}). Define the \emph{silent-failure rate} as the fraction of concepts with $\text{ISF} < 0.05$ whose full-state transferred AUROC reaches $\ge 0.75$ on every target architecture. These are concepts that full-state transports successfully on all three targets without any label-free signal that the transfer is credible -- exactly the cases the quotient operator should flag as uncovered.

\begin{table}[h]
\vspace{-10pt}
\centering
\footnotesize
\setlength{\tabcolsep}{4pt}
\caption{Coverage-aware abstention at $\gamma{=}0.05$ on a 21-concept evaluation pool. Silent-failure rate is $\Pr[\text{ISF} < \gamma \text{ and } \text{AUROC}_{\text{fs}} \ge 0.75 \text{ on every target}]$: the fraction of concepts where full-state transports above quality threshold on all three target architectures despite no span-based coverage certificate. Paired $(\text{quot.}{-}\text{f.s.})$ AUROC gap is computed on deployed concepts. 95\% concept-pool bootstrap, $5{,}000$ draws.}
\label{tab:coverage_abstention_main}
\resizebox{\textwidth}{!}{%
\begin{tabular}{l c ccc}
\toprule
Bank (\# concepts) & Silent-fail rate & \multicolumn{3}{c}{Paired $\Delta(\text{quot.}{-}\text{f.s.})$ on deployed concepts} \\
\cmidrule(lr){3-5}
 & & Qwen-2.5-3B & Qwen-2.5-14B & Mistral-7B \\
\midrule
Full (11)        & $0.238\ [0.048, 0.429]$ & $-0.022\ [-0.043, -0.005]$ & $-0.027\ [-0.049, -0.008]$ & $-0.011\ [-0.032, +0.004]$ \\
Safety-only (10) & $0.286\ [0.095, 0.476]$ & $-0.024\ [-0.046, -0.006]$ & $-0.030\ [-0.053, -0.008]$ & $-0.013\ [-0.034, +0.004]$ \\
Core-only (5)    & $0.381\ [0.190, 0.571]$ & $-0.060\ [-0.101, -0.025]$ & $-0.066\ [-0.105, -0.033]$ & $-0.043\ [-0.080, -0.010]$ \\
\bottomrule
\end{tabular}%
}
\vspace{-5pt}
\end{table}

All three bank conditions have silent-failure rates whose CIs exclude zero, rising monotonically as the bank shrinks (0.238 at the full bank, 0.381 at core-only). The paired (quotient $-$ full-state) AUROC gap on concepts that pass the threshold is negative across targets and banks: 1--3pp with the full bank and 4--7pp with the five-probe core bank -- the cost of coverage-aware abstention. The benchmark gives a single decision-theoretic quantity for coverage-aware deployment. Practitioners using the full bank trade a few AUROC points on in-span concepts for the elimination of roughly a quarter of the apparently-successful monitors that would otherwise deploy without a coverage certificate.

\paragraph{Behavioral monitor portability.}
We extend the evaluation to inference-time output control. The transferred bank is used to score sampled completions rather than only to classify cached activations. Used as a zero-label best-of-$k{=}8$ reranker on 461 LLM-judged prompts, the bank score substantially reduces the harmful compliance rate (HCR; fraction of harmful prompts whose chosen completion is substantive) on both within-family and cross-architecture targets at modest benign false-refusal rates (BFR; fraction of benign prompts whose chosen completion is a refusal) (Appendix~\ref{app:reranking}, Appendix~\ref{app:refusal_audit}).

\begin{table}[h]
\vspace{-10pt}
\centering
\small
\caption{Behavioral reranking outcomes, zero target labels. Best-of-$k{=}8$ reranking with the source-trained 11-probe bank applied through quotient alignment, evaluated by an LLM judge on 461 prompts. 95\% bootstrap CIs over prompts ($1{,}000$ resamples). Full setup in Appendix~\ref{app:reranking}, Table~\ref{tab:reranking}.}
\label{tab:reranking_main}
\begin{tabular}{l ccc c}
\toprule
Target & Baseline HCR & Reranked HCR & BFR & $\Delta$HCR \\
\midrule
Qwen-3B  & $39.5\%$ $[32.3, 47.2]$ & $19.8\%$ $[13.7, 26.1]$ & $6.0\%$ $[3.0, 9.5]$ & $19.7$pp $[12.4, 27.3]$ \\
Mistral  & $73.8\%$ $[66.5, 80.7]$ & $58.4\%$ $[50.9, 65.8]$ & $5.0\%$ $[2.0, 8.5]$ & $15.4$pp $[9.3, 21.7]$ \\
\bottomrule
\end{tabular}
\vspace{-5pt}
\end{table}

Both reductions are achieved without target supervision, using the source-trained bank through quotient alignment. HCR falls by roughly half on Qwen-3B, which belongs to the same model family as the source. Mistral starts from a higher baseline HCR and shows a smaller reduction. This pattern is consistent with stronger transfer between models that share an architecture.

\vspace{-5pt}
\section{Discussion and Conclusion}
\vspace{-5pt}

In this paper, we describe how probing functions can be constrained by representation symmetries rather than chosen ad hoc. Under affine reparameterizations, this yields a unique polynomial hierarchy, identifies structured low-rank probes as principled while excluding basis-dependent monomial-sparse constructions, and shows that cross-model transfer should act on the probe-visible quotient rather than the full hidden state. Our experiments support this picture. Transfer succeeds when the relevant concept survives in the appropriate quotient, and degrades when it depends on basis-specific or otherwise nontransferable structure.

These results also point beyond the affine setting considered here. Our analysis was motivated by the exact symmetry of the final readout layer, but the relevant symmetries deeper in a network need not remain exactly affine. Middle layers may instead obey weaker, richer, or only approximate symmetry classes, which could in turn admit richer probe families. This suggests several promising directions for future work: identifying approximate symmetry groups for neural representations, extending quotient transfer to broader nonlinear concept families, and designing probe banks that are simultaneously expressive, well-conditioned, and minimally redundant. More broadly, our results suggest that probing and cross-model monitor transfer are governed by the same geometric question. Which properties of a neural representation survive natural group actions, and which survive only in a particular basis?

\begin{ack}
Su Hyeong Lee gratefully acknowledges support from the Kwanjeong Educational Foundation Scholarship and the ML Alignment and Theory Scholars (MATS) Fellowship. Su Hyeong Lee was also supported by the Center for AI Safety Compute Grant (Open-Weight Model Safeguards) and a BlueDot Impact Rapid Grant.
\end{ack}

\bibliographystyle{unsrtnat}
\bibliography{references}

\newpage
\tableofcontents
\newpage

\appendix

\section{Related Work}\label{app:probing_related_work}

\subsection{Probing in safety, alignment, and deployment}

Within safety and alignment, probes are often used to separate a model's internal state from its overt behavior. The eliciting-latent-knowledge line asks whether frozen activations contain truth-relevant or policy-relevant information even when the model's verbal answer is unreliable or strategically distorted. Methods in this family include unsupervised search for truth directions, supervised lie-detection from hidden states, studies of the linear geometry of truthfulness, benchmark/model-organism work on quirky or systematically untruthful models, and layerwise affine readouts such as the tuned lens \citep{burns2022discovering,azaria2023internal,marks2023geometry,mallen2023quirky,belrose2023tuned}. Related work studies whether models internally encode instruction-following success, again using simple directions in activation space that can generalize beyond the prompts used for training \citep{heo2024instruction}. The common premise is that the representation may contain safety-relevant structure that is easier to read than to elicit behaviorally.

Probe-like techniques are also increasingly used as control tools rather than purely diagnostic ones, and recent work explicitly explores nonlinear safety monitors beyond linear probes \citep{oldfield2026beyond}. Representation engineering and activation steering identify low-dimensional directions associated with properties such as truthfulness, harmlessness, sentiment, or refusal, then intervene on those directions at inference time or regularize against them during training \citep{zou2023representation,li2023iti,turner2023activation,panickssery2024caa,arditi2024refusal,papadatos2024sycophancy}. In parallel, safety work on deception and hidden goals uses probes and related internal-state analyses as part of broader auditing pipelines: models can strategically deceive under pressure, alignment audits can recover implanted hidden objectives, and deceptive instructions can induce measurable representational flips that remain linearly decodable \citep{scheurer2023strategic,marks2025auditing_hidden,long2025flip}. This literature is close in spirit to ours, but the emphasis is different: we ask which probe families can be interpreted as reading properties of the representation itself, rather than artifacts of a particular parameterization or basis.

Finally, the same design pattern appears in deployment systems, even when it is described as moderation, monitoring, or safeguards rather than as probing. OpenAI describes production use of moderation models, safety classifiers, and reasoning monitors in its moderation paper and system cards \citep{markov2023holistic,openai2023gpt4system,openai2024gpt4osystem,openai2025o3o4system}. Anthropic's Constitutional Classifiers make this connection especially explicit: safeguards are trained from natural-language constitutions and used as stand-alone jailbreak defenses, while follow-up work introduces production-grade classifier cascades and explicitly ensembles external classifiers with linear probes \citep{sharma2025constitutional,cunningham2026constitutional}. For safety monitoring, these small readout models are attractive because they are cheap to retrain, easy to layer on top of a changing base model, and amenable to red-teaming and policy iteration. This operational setting also sharpens the transfer question studied in our paper: if monitors must move across model versions or families, the relevant object to align is the probe-visible structure, not the entire hidden state.

\subsection{Nonlinear probing}

Probing has become one of the standard tools for studying learned representations in deep networks. In the simplest and most widely used setting, a probe is a lightweight classifier or regressor trained on frozen activations in order to predict a target property \citep{belinkov2022probing,conneau2018cram,hewitt-liang-2019-designing}. Linear probes are especially common because they are easy to optimize, easy to compare across layers, and hard to confuse with a large task-solving network in their own right \citep{hewitt-liang-2019-designing,pimentel2020pareto,belinkov2022probing}. This simplicity is a major practical virtue.

At the same time, there is now considerable evidence that many representations may not be easily exhausted by linear structure \citep{hewitt2019structural,white-etal-2021-non,pimentel2020pareto,pimentel-etal-2020-information}. Structural probing for syntax already requires a quadratic geometry \citep{hewitt2019structural,white-etal-2021-non,belinkov2022probing}; more generally, richer nonlinear probes can reveal information that a linear readout misses \citep{pimentel2020pareto,hewitt-liang-2019-designing,belinkov2022probing,pimentel-etal-2020-information,white-etal-2021-non}. The central methodological question is therefore not whether nonlinear probes should ever be used, but which nonlinear families are principled. If we allow an arbitrary universal approximator as a probe, a positive result can be driven by the probe's own expressive power rather than by the representation under study \citep{hewitt-liang-2019-designing,pimentel2020pareto}.

This paper isolates one clean principle that sharply constrains the answer. Neural representations are not unique objects: even when two systems realize the same task-relevant computation, their hidden coordinates may differ by reparameterizations. A probing family intended to reveal structure already present in the representation should therefore be stable under such reparameterizations. We formalize this as affine invariance under precomposition. We then ask which finite-dimensional probe spaces survive this requirement.

\subsection{Parameter symmetries and identifiability}
\label{app:parameter_symmetries}

Neural networks can realize the same function through different parameterizations, with consequences for optimization, identifiability, and representation comparison \citep{zhao2026symmetry}. Related identifiability results establish conditions under which a network's input--output function determines its realization up to symmetries associated with its activation functions and architecture \citep{vlacic2021affine,vlacic2022identifiability}. Hidden-neuron permutation symmetries also guide the design of equivariant architectures for processing network weights \citep{navon2023equivariant}. Recent work on gauge symmetries in Transformers characterizes function-preserving attention reparameterizations under architectural and nondegeneracy assumptions and develops the geometry of the corresponding parameter-space quotient \citep{wang2026complete,wang2026gauge}.

The identifiability literature asks when a network's function determines the parameters and architecture that realize it. Our work instead asks what hidden representational structure is shared across different realizations of the same behavior, and how probes can reveal that structure. We obtain three complementary results. Under standard nondegeneracy assumptions, behaviorally equivalent models have final-layer representations related by an invertible affine transformation (Corollary~\ref{cor:affine_head_equiv}, Appendix~\ref{app:affine_head_equivalence}). Requiring finite-dimensional linear spaces of continuous scalar probe scores to respect this coordinate freedom yields a polynomial hierarchy (Theorem~\ref{thm:scalar_affine_classification}, Section~\ref{subsec:probe_classification}). Finally, a shared family of linear concepts determines a canonical probe-visible space for cross-model transfer (Theorem~\ref{thm:common_abstract_space}, Section~\ref{subsec:concept_families}). These results connect behavioral equivalence to a principled choice of probe family and a precise account of what representations have in common.

\subsection{Toward a broader explanation of affine alignment}
\label{app:broader_affine_alignment}

The proof strategy underlying our readout-equivalence results (Theorem~\ref{thm:linear_head_equivalence}, Proposition~\ref{prop:softmax_affine_symmetry}, and Corollary~\ref{cor:affine_head_equiv}) may generalize to provide a broader theoretical explanation for why representations learned by different neural networks are often related by affine transformations. The strategy uses agreement of observable readouts to constrain the relationship between the underlying representations. Our results establish this connection at the final readout layer, where equivalent realizations are related by an invertible linear transformation, with an affine extension for heads with biases. Theorem~\ref{thm:robust_alignment} further shows that approximate logit agreement implies approximate linear alignment, with an error bound controlled by the conditioning of the readout. Together, these results suggest that affine relationships between representations may arise as a consequence of shared behavioral constraints. Extending this explanation to upstream layers would require identifying how the downstream computation imposes analogous constraints on intermediate representations. We view this as a promising direction for generalizing our proof strategy.

\section{Detailed Proofs for Section~\ref{sec:probe_classification}}
\label{app:readout_proofs}
\label{app:probe_proofs}

This appendix follows the logic of Section~\ref{sec:probe_classification} in order: readout-layer symmetries first, the classification of affine-stable probe spaces second, and structured degree-2 families last. Each proof is followed by a short proof sketch that records the main idea without the surrounding bookkeeping.

\subsection{Readout-layer results}\label{app:setup}

This subsection proves Theorem~\ref{thm:linear_head_equivalence} from Section~\ref{sec:readout} and states its normalized, softmax, and approximate extensions.

Let \(D=\{x_1,\dots,x_N\}\) be a finite dataset. A final-layer realization consists of
\begin{itemize}[leftmargin=2em]
    \item a representation map \(\gamma:D\to\R^d\), and
    \item readout vectors \(\lambda_1,\dots,\lambda_n\in\R^d\).
\end{itemize}
For a linear head, the score of class \(i\) on input \(x\) is
\[
 s_i(x)=\ip{\gamma(x)}{\lambda_i}.
\]
For a softmax head, the probability of class \(i\) is
\[
 p_i(x)=\frac{\exp(\ip{\gamma(x)}{\lambda_i})}{\sum_{j=1}^n \exp(\ip{\gamma(x)}{\lambda_j})}.
\]

Then, we must now compare the two formal objects of interest:
\[
(\gamma,\Lambda) = (\gamma,\lambda_1,\dots,\lambda_n)
\qquad\text{and}\qquad
(\gamma_*,\Lambda_*) = (\gamma_*,\lambda_1^*,\dots,\lambda_n^*).
\]

Two realizations $(\gamma,\Lambda)$ and $(\gamma_*,\Lambda_*)$ are related by maps $g$ and $f$ satisfying $\gamma_*(x)=g(\gamma(x))$ and $\lambda_i^*=f(\lambda_i)$. The realization is \emph{nondegenerate} if $\{\gamma(x)\}$ and $\{\lambda_i\}$ each span $\R^d$. A collection $\{\lambda_1,\dots,\lambda_n\}$ \emph{affinely spans} $\R^d$ if $\{\lambda_i-\lambda_1:2\le i\le n\}$ spans $\R^d$.

\begin{restatable}[Normalized linear heads]{corollary}{corNormalizedHeads}
\label{cor:normalized_linear_heads}
Assume the hypotheses of Theorem~\ref{thm:linear_head_equivalence} except that the observed scores are normalized:
\[
\frac{\ip{\gamma_*(x)}{\lambda_i^*}}{H_*(x)}=\frac{\ip{\gamma(x)}{\lambda_i}}{H(x)}
\qquad\text{for all }x\in D\text{ and all }i,
\]
where \(H(x)\neq 0\) and \(H_*(x)\neq 0\) for all \(x\in D\). Then there exists a unique \(A\in\GL(d)\) and a scalar \(\rho(x)\neq 0\) for each \(x\in D\) such that
\[
\gamma_*(x)=\rho(x)A^{-\top}\gamma(x),
\qquad
\lambda_i^*=A\lambda_i.
\]
\end{restatable}

\begin{restatable}[Affine maps and linear independence]{lemma}{lemAffineIndep}
\label{lem:affine_linear_independence}
Let \(T:\R^d\to\R^d\) be an affine map of the form \(T(v)=Av+b\) with \(A\in\GL(d)\). Then \(T\) preserves \emph{every} linearly independent subset of \(\R^d\) if and only if \(b=0\).
\end{restatable}

\begin{restatable}[When the softmax shift disappears]{corollary}{corSoftmaxShift}
\label{cor:softmax_shift_disappears}
Under the hypotheses of Proposition~\ref{prop:softmax_affine_symmetry}, assume in addition that the induced map on readout vectors preserves every linearly independent subset of \(\R^d\). Then \(c=0\), so the readout relation reduces to the purely linear form
\[
\gamma_*(x)=A^{-\top}\gamma(x),
\qquad
\lambda_i^*=A\lambda_i.
\]
\end{restatable}

\begin{restatable}[Softmax-head equivalence]{proposition}{propSoftmax}
\label{prop:softmax_affine_symmetry}
Assume that $\{\gamma(x):x\in D\}$ spans $\R^d$, that $\{\lambda_1,\dots,\lambda_n\}$ affinely spans $\R^d$, and that the corresponding transformed difference vectors also span $\R^d$. Suppose that for every $x\in D$ and every $i$ we have
\[
\frac{\exp(\ip{g(\gamma(x))}{f(\lambda_i)})}{\sum_{j=1}^n \exp(\ip{g(\gamma(x))}{f(\lambda_j)})}
=
\frac{\exp(\ip{\gamma(x)}{\lambda_i})}{\sum_{j=1}^n \exp(\ip{\gamma(x)}{\lambda_j})}.
\]
Then there exist a unique $A\in\GL(d)$ and a unique $c\in\R^d$ such that $\gamma_*(x)=A^{-\top}\gamma(x)$ for all $x\in D$ and $\lambda_i^*=A\lambda_i+c$ for all $i$, equivalently $\Lambda_*=\Lambda A^\top + \mathbf{1}c^\top$.
\end{restatable}

\begin{restatable}[Robust linear alignment]{theorem}{thmRobustAlignment}
\label{thm:robust_alignment}
Let $\Lambda_1,\Lambda_2\in\R^{n\times d}$ have full column rank with $\sigma_d(\Lambda_2)\ge \sigma_0>0$. Let $\gamma_1,\gamma_2:D\to\R^d$ be two representation maps. If $\frac{1}{|D|}\sum_{x\in D}\norm{\Lambda_1\gamma_1(x)-\Lambda_2\gamma_2(x)}_2^2 \le \eps^2$, then with $A:=\Lambda_2^{\dagger}\Lambda_1$ we have $\frac{1}{|D|}\sum_{x\in D}\norm{\gamma_2(x)-A\gamma_1(x)}_2^2 \le \eps^2/\sigma_0^2$.
\end{restatable}

\subsubsection{Proof of Theorem~\ref{thm:linear_head_equivalence}}

\thmLinearHead*

\begin{proof}
Because \(\Lambda\) has rank \(d\), choose indices \(i_1,\ldots,i_d\) such that \(\lambda_{i_1},\ldots,\lambda_{i_d}\) are linearly independent. The corresponding transformed vectors are also independent. Indeed, if \(\sum_r a_r\lambda_{i_r}^*=0\), score equality implies
\[
\left\langle\gamma(x),\sum_r a_r\lambda_{i_r}\right\rangle
=\left\langle\gamma_*(x),\sum_r a_r\lambda_{i_r}^*\right\rangle=0
\quad\text{for every }x\in D.
\]
Since \(\gamma(D)\) spans \(\R^d\), this forces \(\sum_r a_r\lambda_{i_r}=0\), hence every \(a_r=0\).

Define the \(d\times d\) matrices
\[
 \Lambda_d :=
 \begin{bmatrix}
 \lambda_{i_1}^\top\\
 \vdots\\
 \lambda_{i_d}^\top
 \end{bmatrix},
 \qquad
 V_d :=
 \begin{bmatrix}
 (\lambda_{i_1}^*)^\top\\
 \vdots\\
 (\lambda_{i_d}^*)^\top
 \end{bmatrix}.
\]
Because their rows are linearly independent, both matrices are invertible. Now fix any \(x\in D\). For each row index \(r\in\{1,\dots,d\}\), equation \eqref{eq:pairing_preserved_linear} gives
\[
 \ip{\gamma_*(x)}{\lambda_{i_r}^*}
 =
 \ip{\gamma(x)}{\lambda_{i_r}}.
\]
Stacking these \(d\) equations into one matrix equation yields
\[
 V_d\, \gamma_*(x) = \Lambda_d\,\gamma(x).
\]
Since \(V_d\) is invertible, we may solve for \(\gamma_*(x)\):
\begin{equation}
\label{eq:g_linear_formula}
 \gamma_*(x) = V_d^{-1}\Lambda_d\,\gamma(x).
\end{equation}
Thus every transformed hidden state is obtained by applying the same linear matrix
\[
 G:=V_d^{-1}\Lambda_d.
\]
In particular,
\[
 \gamma_*(x)=G\gamma(x)
 \qquad\text{for all }x\in D.
\]
The matrix \(G\) is a product of two invertible matrices, so \(G\in\GL(d)\). We now show uniqueness. Suppose \(G'\in\R^{d\times d}\) also satisfies
\[
 \gamma_*(x)=G'\gamma(x)
 \qquad\text{for all }x\in D.
\]
Subtracting the two expressions gives
\[
 (G-G')\gamma(x)=0
 \qquad\text{for all }x\in D.
\]
Because the represented points \(\{\gamma(x):x\in D\}\) span \(\R^d\), the only linear map that vanishes on all of them is the zero map. Hence \(G-G'=0\), so \(G=G'\).

Now, fix any class index \(i\in\{1,\dots,n\}\). Equation \eqref{eq:pairing_preserved_linear} and the identity \(\gamma_*(x)=G\gamma(x)\) give
\[
 \ip{G\gamma(x)}{\lambda_i^*} = \ip{\gamma(x)}{\lambda_i}
 \qquad\text{for all }x\in D.
\]
Using the transpose to move \(G\) from the first argument to the second, this becomes
\[
 \ip{\gamma(x)}{G^\top \lambda_i^*} = \ip{\gamma(x)}{\lambda_i}
 \qquad\text{for all }x\in D.
\]
Subtracting the right-hand side from the left-hand side gives
\[
 \ip{\gamma(x)}{G^\top \lambda_i^*-\lambda_i}=0
 \qquad\text{for all }x\in D.
\]
Again the represented points span \(\R^d\). The only vector orthogonal to a spanning set is the zero vector. Therefore
\[
 G^\top \lambda_i^*-\lambda_i = 0.
\]
Equivalently,
\[
 G^\top \lambda_i^*=\lambda_i,
 \qquad\text{so}
 \qquad
 \lambda_i^*=G^{-\top}\lambda_i.
\]
Define
\[
 A:=G^{-\top}.
\]
Then \(A\in\GL(d)\), and for every \(i\) we have
\[
 \lambda_i^*=A\lambda_i.
\]
Since \(A=G^{-\top}\), we have \(G=A^{-\top}\). Therefore the representation relation becomes
\[
 \gamma_*(x)=G\gamma(x)=A^{-\top}\gamma(x)
 \qquad\text{for all }x\in D.
\]
Finally, if \(\Lambda\) is the readout matrix with rows \(\lambda_i^\top\), then the transformed readout matrix has rows
\[
 (\lambda_i^*)^\top = (A\lambda_i)^\top = \lambda_i^\top A^\top.
\]
Thus
\[
 \Lambda_* = \Lambda A^\top.
\]
This proves the theorem.
\end{proof}

\paragraph{Proof sketch.}
The proof is nothing more than a careful use of bilinearity. Once a basis of readout vectors is fixed, preserving every pairing with that basis uniquely determines the transformed representation. After that, preserving pairings with a spanning set of represented points uniquely determines the transformed readout vectors. The two directions constrain each other until only a single invertible linear change of coordinates remains.

\subsubsection{Proof of Corollary~\ref{cor:normalized_linear_heads}}

\corNormalizedHeads*

\begin{proof}
For each \(x\in D\), define the scalar
\[
 \rho(x):=\frac{H_*(x)}{H(x)}.
\]
The assumptions guarantee that both denominators are nonzero, so \(\rho(x)\) is well-defined and nonzero for every \(x\in D\). The normalized equality says
\[
 \frac{\ip{\gamma_*(x)}{\lambda_i^*}}{H_*(x)}
 =
 \frac{\ip{\gamma(x)}{\lambda_i}}{H(x)}
 \qquad\text{for all }x\in D\text{ and all }i.
\]
Multiply both sides by \(H_*(x)\):
\[
 \ip{\gamma_*(x)}{\lambda_i^*}
 =
 \frac{H_*(x)}{H(x)}\ip{\gamma(x)}{\lambda_i}
 =
 \rho(x)\ip{\gamma(x)}{\lambda_i}.
\]
Now divide both sides by \(\rho(x)\):
\[
 \ip{\rho(x)^{-1}\gamma_*(x)}{\lambda_i^*}
 =
 \ip{\gamma(x)}{\lambda_i}.
\]
Thus, if we define a new transformed representation on the dataset by
\[
 \widehat\gamma(x):=\rho(x)^{-1}\gamma_*(x),
\]
then the pair \((\widehat\gamma,\lambda_1^*,\dots,\lambda_n^*)\) preserves the same bilinear scores as \((\gamma,\lambda_1,\dots,\lambda_n)\).

To apply Theorem~\ref{thm:linear_head_equivalence}, we must also know that \(\widehat\gamma\) depends only on the represented point \(\gamma(x)\), not on the choice of \(x\). Suppose \(\gamma(x)=\gamma(y)\). Then for every \(i\),
\[
 \ip{\widehat\gamma(x)-\widehat\gamma(y)}{\lambda_i^*}
 =
 \ip{\gamma(x)-\gamma(y)}{\lambda_i}
 =0.
\]
Because the readout vectors \(\lambda_i^*\) span \(\R^d\), the only vector orthogonal to all of them is the zero vector. Hence
\[
 \widehat\gamma(x)=\widehat\gamma(y).
\]
So there is a well-defined map \(\widehat g\) on the represented points with \(\widehat g(\gamma(x))=\widehat\gamma(x)\), and the hypotheses of Theorem~\ref{thm:linear_head_equivalence} are satisfied. Therefore there exists a unique \(A\in\GL(d)\) such that
\[
 \widehat\gamma(x)=A^{-\top}\gamma(x),
 \qquad
 \lambda_i^*=A\lambda_i.
\]
Substituting back \(\widehat\gamma(x)=\rho(x)^{-1}\gamma_*(x)\), we obtain
\[
 \gamma_*(x)=\rho(x)A^{-\top}\gamma(x),
 \qquad
 \lambda_i^*=A\lambda_i.
\]
This is exactly the claimed formula.
\end{proof}

\paragraph{Proof sketch.}
Normalization introduces one extra scalar degree of freedom per input. Once that scalar is peeled off, the remaining structure is identical to the unnormalized linear-head case.

\subsubsection{Proof of Proposition~\ref{prop:softmax_affine_symmetry}}

\propSoftmax*

\begin{proof}
Fix an input \(x\in D\). For convenience write
\[
 a_i(x):=\ip{g(\gamma(x))}{f(\lambda_i)},
 \qquad
 b_i(x):=\ip{\gamma(x)}{\lambda_i}.
\]
The hypothesis says
\[
 \frac{e^{a_i(x)}}{\sum_{j=1}^n e^{a_j(x)}}
 =
 \frac{e^{b_i(x)}}{\sum_{j=1}^n e^{b_j(x)}}
 \qquad\text{for all }i.
\]
Multiply both sides by the two denominators:
\[
 e^{a_i(x)}\sum_{j=1}^n e^{b_j(x)}
 =
 e^{b_i(x)}\sum_{j=1}^n e^{a_j(x)}.
\]
The factor
\[
 c(x):=\frac{\sum_{j=1}^n e^{a_j(x)}}{\sum_{j=1}^n e^{b_j(x)}}
\]
does not depend on \(i\). Therefore
\[
 e^{a_i(x)} = c(x)e^{b_i(x)}
 \qquad\text{for all }i.
\]
Taking logarithms gives
\[
 a_i(x)=b_i(x)+\kappa(x),
 \qquad
 \kappa(x):=\log c(x).
\]
So for all \(x\in D\) and all \(i\),
\begin{equation}
\label{eq:logits_shifted}
 \ip{g(\gamma(x))}{f(\lambda_i)} = \ip{\gamma(x)}{\lambda_i}+\kappa(x).
\end{equation}
The key point is that \(\kappa(x)\) depends on the input \(x\), but not on the class index \(i\).

Now, fix the base class index \(1\). For any \(i\in\{2,\dots,n\}\), subtract the equation for class \(1\) from the equation for class \(i\). The common shift \(\kappa(x)\) cancels, and we obtain
\[
 \ip{g(\gamma(x))}{f(\lambda_i)-f(\lambda_1)}
 =
 \ip{\gamma(x)}{\lambda_i-\lambda_1}
 \qquad\text{for all }x\in D.
\]
By assumption, the difference vectors \(\{\lambda_i-\lambda_1:2\le i\le n\}\) span \(\R^d\), and the corresponding transformed difference vectors also span \(\R^d\). Therefore Theorem~\ref{thm:linear_head_equivalence} applies to the difference family. It yields a unique matrix \(A\in\GL(d)\) such that
\begin{equation}
\label{eq:softmax_g_linear}
 g(\gamma(x))=A^{-\top}\gamma(x)
 \qquad\text{for all }x\in D,
\end{equation}
and
\begin{equation}
\label{eq:softmax_diff_linear}
 f(\lambda_i)-f(\lambda_1)=A(\lambda_i-\lambda_1)
 \qquad\text{for all }i.
\end{equation}

We now convert the relation on differences into an affine formula. Define
\[
 c:=f(\lambda_1)-A\lambda_1.
\]
Then for every \(i\), equation \eqref{eq:softmax_diff_linear} implies
\[
 f(\lambda_i)
 =
 f(\lambda_1)+A(\lambda_i-\lambda_1)
 =
 A\lambda_i + \bigl(f(\lambda_1)-A\lambda_1\bigr)
 =
 A\lambda_i + c.
\]
Thus
\[
 \lambda_i^*=f(\lambda_i)=A\lambda_i+c
 \qquad\text{for all }i.
\]
This proves the affine form of the transformed readout vectors.

Next, we show that the same matrix \(A\) gives the representation relation. Note that this has already been proved in \eqref{eq:softmax_g_linear}; namely,
\[
 \gamma_*(x)=g(\gamma(x))=A^{-\top}\gamma(x)
 \qquad\text{for all }x\in D.
\]
The row vector corresponding to class \(i\) in the transformed readout matrix is
\[
 (\lambda_i^*)^\top = (A\lambda_i+c)^\top = \lambda_i^\top A^\top + c^\top.
\]
Therefore the transformed readout matrix is
\[
 \Lambda_* = \Lambda A^\top + \mathbf 1 c^\top,
\]
where \(\mathbf 1\in\R^n\) is the all-ones vector. Suppose another pair \((\widetilde A,\widetilde c)\) also satisfies
\[
 \gamma_*(x)=\widetilde A^{-\top}\gamma(x)
 \qquad\text{for all }x\in D,
\]
and
\[
 \lambda_i^*=\widetilde A\lambda_i+\widetilde c
 \qquad\text{for all }i.
\]
Then for every \(x\in D\),
\[
 (A^{-\top}-\widetilde A^{-\top})\gamma(x)=0.
\]
Because the represented points span \(\R^d\), it follows that
\[
 A^{-\top}=\widetilde A^{-\top},
\]
hence \(A=\widetilde A\). Once \(A\) is fixed, the formula for class \(1\) gives
\[
 c=\lambda_1^*-A\lambda_1,
\]
so \(c\) is also unique. This proves the proposition.
\end{proof}

\paragraph{Proof sketch.}
Softmax equality preserves only \emph{differences} of logits. Once the common shift is removed by subtracting one class from another, the problem collapses to the linear-head theorem. The only new symmetry is the class-independent affine shift on readout vectors.

\subsubsection{Proof of Lemma~\ref{lem:affine_linear_independence}}

\lemAffineIndep*

\begin{proof}
If \(b=0\), then \(T(v)=Av\). Since \(A\in\GL(d)\), the map \(T\) is an invertible linear map. Invertible linear maps preserve all linear relations, and therefore preserve linear independence. Assume \(b\neq 0\). Because \(A\) is invertible, the equation
\[
 Av+b=0
\]
has the unique solution
\[
 v_0=-A^{-1}b.
\]
Since \(b\neq 0\) and \(A\) is invertible, we have \(v_0\neq 0\). Now consider the singleton set
\[
 S=\{v_0\}.
\]
A singleton \(\{v_0\}\) with \(v_0\neq 0\) is linearly independent. But its image is
\[
 T(v_0)=Av_0+b = A(-A^{-1}b)+b = -b+b = 0.
\]
So
\[
 T(S)=\{0\}.
\]
A singleton containing the zero vector is linearly dependent. Therefore \(T\) does not preserve every linearly independent set.
\end{proof}

\paragraph{Proof sketch.}
A nonzero translation always moves one nonzero point to the origin. Since the origin itself is linearly dependent, global preservation of linear independence forces the translation part to vanish.

\subsubsection{Proof of Corollary~\ref{cor:softmax_shift_disappears}}

\corSoftmaxShift*

\begin{proof}
By Proposition~\ref{prop:softmax_affine_symmetry}, the transformed readout map has the form
\[
 f(\lambda)=A\lambda+c
\]
for some \(A\in\GL(d)\) and some \(c\in\R^d\). The additional assumption says that this affine map preserves every linearly independent subset of \(\R^d\). By Lemma~\ref{lem:affine_linear_independence}, that can happen only if the translation part is zero. Hence \(c=0\). Substituting \(c=0\) into the formulas of Proposition~\ref{prop:softmax_affine_symmetry} gives
\[
 \lambda_i^*=A\lambda_i,
 \qquad
 \gamma_*(x)=A^{-\top}\gamma(x).
\]
This is exactly the claimed conclusion.
\end{proof}

\paragraph{Proof sketch.}
The softmax symmetry allows a common shift across classes. The extra linear-independence-preservation assumption is stronger than softmax equivalence and rules out that shift.

\subsubsection{Proof of Theorem~\ref{thm:robust_alignment}}

\thmRobustAlignment*

\begin{proof}
Let
\[
 A:=\Lambda_2^{\dagger}\Lambda_1.
\]
We must show that
\[
 \frac{1}{|D|}\sum_{x\in D}\norm{\gamma_2(x)-A\gamma_1(x)}_2^2
 \le
 \frac{\eps^2}{\sigma_0^2}.
\]

We use the fact that \(\Lambda_2\) has full column rank. For a full-column-rank matrix, the Moore-Penrose pseudoinverse satisfies
\[
 \Lambda_2^{\dagger}\Lambda_2 = I_d.
\]
Therefore, for each \(x\in D\),
\[
 \gamma_2(x)
 =
 \Lambda_2^{\dagger}\Lambda_2\gamma_2(x).
\]
Substitute this identity into the representation error:
\begin{align*}
 \gamma_2(x)-A\gamma_1(x)
 &=
 \Lambda_2^{\dagger}\Lambda_2\gamma_2(x)-\Lambda_2^{\dagger}\Lambda_1\gamma_1(x)\\
 &=
 \Lambda_2^{\dagger}\bigl(\Lambda_2\gamma_2(x)-\Lambda_1\gamma_1(x)\bigr).
\end{align*}
Taking norms and using the operator norm bound \(\norm{Mv}_2\le \norm{M}_{\mathrm{op}}\norm{v}_2\), we get
\[
 \norm{\gamma_2(x)-A\gamma_1(x)}_2
 \le
 \norm{\Lambda_2^{\dagger}}_{\mathrm{op}}
 \cdot
 \norm{\Lambda_2\gamma_2(x)-\Lambda_1\gamma_1(x)}_2.
\]
Square both sides:
\[
 \norm{\gamma_2(x)-A\gamma_1(x)}_2^2
 \le
 \norm{\Lambda_2^{\dagger}}_{\mathrm{op}}^2
 \cdot
 \norm{\Lambda_2\gamma_2(x)-\Lambda_1\gamma_1(x)}_2^2.
\]
Average over \(x\in D\):
\[
 \frac{1}{|D|}\sum_{x\in D}\norm{\gamma_2(x)-A\gamma_1(x)}_2^2
 \le
 \norm{\Lambda_2^{\dagger}}_{\mathrm{op}}^2
 \cdot
 \frac{1}{|D|}\sum_{x\in D}\norm{\Lambda_2\gamma_2(x)-\Lambda_1\gamma_1(x)}_2^2.
\]
By assumption, the average output discrepancy is at most \(\eps^2\). Thus it remains only to bound \(\norm{\Lambda_2^{\dagger}}_{\mathrm{op}}\).

For a full-column-rank matrix, the operator norm of the pseudoinverse is the reciprocal of the smallest singular value:
\[
 \norm{\Lambda_2^{\dagger}}_{\mathrm{op}} = \frac{1}{\sigma_d(\Lambda_2)}.
\]
Since \(\sigma_d(\Lambda_2)\ge \sigma_0\), we obtain
\[
 \norm{\Lambda_2^{\dagger}}_{\mathrm{op}}^2 \le \frac{1}{\sigma_0^2}.
\]
Combining the two inequalities gives
\[
 \frac{1}{|D|}\sum_{x\in D}\norm{\gamma_2(x)-A\gamma_1(x)}_2^2
 \le
 \frac{\eps^2}{\sigma_0^2}.
\]
This is the desired bound.
\end{proof}

\paragraph{Proof sketch.}
The readout matrix acts as an observation operator. If that operator is well-conditioned, then small output error can only come from small hidden-space error in the observed directions. The pseudoinverse converts the output discrepancy back into a hidden-state discrepancy, and the smallest singular value determines how much that conversion can amplify error.

\subsubsection{Affine-head equivalence via homogenization}
\label{app:affine_head_equivalence}

The readout results above address bias-free bilinear heads and establish a $\GL(d)$ symmetry on hidden states. To extend to the full affine group $\Aff(d)$, we pass to affine heads via the standard homogenization trick.

\begin{corollary}[Affine-head equivalence via homogenization]
\label{cor:affine_head_equiv}
Let the logits be affine in the hidden state:
\[
s_i(x) = \ip{\gamma(x)}{\lambda_i} + \beta_i.
\]
Define augmented hidden states and readouts by
\[
\tilde{\gamma}(x) := \bigl(\gamma(x),\, 1\bigr) \in \R^{d+1},
\qquad
\tilde{\lambda}_i := \bigl(\lambda_i,\, \beta_i\bigr) \in \R^{d+1}.
\]
Apply Theorem~\ref{thm:linear_head_equivalence} (or Proposition~\ref{prop:softmax_affine_symmetry} for softmax) to the augmented realization in $\R^{d+1}$. Under the corresponding nondegeneracy assumptions, there exist $A \in \GL(d)$ and $b \in \R^d$ such that
\[
\gamma_*(x) = A^{-\top}\gamma(x) + b.
\]
For exact affine logits,
\[
\lambda_i^* = A\lambda_i,
\qquad
\beta_i^* = \beta_i - \ip{A\lambda_i}{b}.
\]
For softmax heads, there is additionally a common shift $(\bar{c}, \tau) \in \R^d \times \R$ such that
\[
\lambda_i^* = A\lambda_i + \bar{c},
\qquad
\beta_i^* = \beta_i - \ip{A\lambda_i}{b} + \tau,
\]
which adds the same quantity $\ip{\bar{c}}{\gamma_*(x)} + \tau$ to every class logit and is absorbed by softmax normalization.
\end{corollary}

\begin{proof}
The augmented scores satisfy $s_i(x) = \ip{\tilde{\gamma}(x)}{\tilde{\lambda}_i}$. By Theorem~\ref{thm:linear_head_equivalence} applied in $\R^{d+1}$, there exists $\tilde{A} \in \GL(d{+}1)$ with $\tilde{\gamma}_*(x) = \tilde{A}^{-\top}\tilde{\gamma}(x)$ and $\tilde{\lambda}_i^* = \tilde{A}\tilde{\lambda}_i$. Since both the original and equivalent augmented hidden states lie in the affine chart $\{u_{d+1} = 1\}$, the map $\tilde{A}^{-\top}$ must preserve this chart. Writing $\tilde{A}^{-\top}$ in block form:
\[
\tilde{A}^{-\top} = \begin{pmatrix} A^{-\top} & b \\ 0 & 1 \end{pmatrix},
\qquad\text{so}\qquad
\tilde{A} = \begin{pmatrix} A & 0 \\ -b^\top A & 1 \end{pmatrix}.
\]
Unpacking $\tilde{\gamma}_*(x) = \tilde{A}^{-\top}\tilde{\gamma}(x)$ gives $\gamma_*(x) = A^{-\top}\gamma(x) + b$. Unpacking $\tilde{\lambda}_i^* = \tilde{A}\tilde{\lambda}_i$ gives $\lambda_i^* = A\lambda_i$ and $\beta_i^* = -b^\top A\lambda_i + \beta_i = \beta_i - \ip{A\lambda_i}{b}$. The softmax case adds the common shift from Proposition~\ref{prop:softmax_affine_symmetry} applied in $\R^{d+1}$, which decomposes into a $\R^d$ component $\bar{c}$ (acting on readout vectors) and a scalar $\tau$ (acting on biases).
\end{proof}

\paragraph{Remark.}
The translation $b$ in the hidden-state transformation arises from the bias coordinate $\beta_i$, not from the softmax common-logit shift. This distinction matters: the softmax shift is absorbed by normalization, while the translation $b$ is a genuine degree of freedom in the hidden-space reparameterization. Together, $A$ and $b$ generate the full affine group $\Aff(d)$, justifying the affine-invariance requirement in the probe classification theorem.

\subsection{Probe classification results}

We now prove the probe classification theorem. The proof is fully self-contained and uses only elementary linear algebra, mollification, and a constructive lemma about homogeneous polynomials.

\subsubsection{Preliminary lemmas}

\begin{lemma}[Finite-dimensional subspaces of \(C(\R^n)\) are closed]
\label{lem:finite_dim_closed}
Every finite-dimensional linear subspace of \(C(\R^n)\), equipped with the compact-open topology, is closed.
\end{lemma}

\begin{proof}
Let \(W\subset C(\R^n)\) be finite-dimensional, and choose a basis \(w_1,\dots,w_m\) of \(W\). Define
\[
 T:\R^m\to W,
 \qquad
 T(a_1,\dots,a_m)=\sum_{j=1}^m a_j w_j.
\]
Because the basis vectors are linearly independent, \(T\) is a linear bijection. We will use the fact that the compact-open topology on \(C(\R^n)\) is the topology of uniform convergence on compact subsets. Fix the standard exhaustion
\[
 K_q:=\{x\in\R^n:\norm{x}_2\le q\},
 \qquad q=1,2,3,\dots.
\]
This topology is metrized by, for example,
\[
 d(f,g):=\sum_{q=1}^{\infty} 2^{-q}\min\Bigl(1,\sup_{x\in K_q}|f(x)-g(x)|\Bigr).
\]
Hence a subset of \(C(\R^n)\) is closed if and only if it is sequentially closed. We therefore prove that \(W\) is sequentially closed. Let \((u_k)\subset W\) converge in \(C(\R^n)\) to some \(u\in C(\R^n)\). Write
\[
 u_k=\sum_{j=1}^m a_{k,j} w_j.
\]
We claim that the coefficient vectors
\[
 a_k:=(a_{k,1},\dots,a_{k,m})\in\R^m
\]
are bounded. Suppose they were unbounded. Then after passing to a subsequence we could assume \(\norm{a_k}_2\to\infty\). Define normalized coefficients
\[
 b_k:=\frac{a_k}{\norm{a_k}_2}.
\]
The sequence \((b_k)\) lies on the unit sphere of \(\R^m\), so by compactness it has a convergent subsequence, still denoted \((b_k)\), with limit \(b=(b_1,\dots,b_m)\) and \(\norm{b}_2=1\).

Now divide the identity for \(u_k\) by \(\norm{a_k}_2\):
\[
 \frac{u_k}{\norm{a_k}_2}=\sum_{j=1}^m b_{k,j} w_j.
\]
Because \(u_k\to u\) in the compact-open topology, the left-hand side converges to \(0\). Passing to the limit on the right-hand side yields
\[
 \sum_{j=1}^m b_j w_j=0.
\]
Since \(w_1,\dots,w_m\) are linearly independent, all \(b_j\) must be zero. That contradicts \(\norm{b}_2=1\). Thus the coefficient vectors \((a_k)\) are bounded.

By Bolzano-Weierstrass, a subsequence of \((a_k)\) converges to some \(a=(a_1,\dots,a_m)\in\R^m\). By continuity of \(T\), the corresponding subsequence of \((u_k)\) converges to
\[
 T(a)=\sum_{j=1}^m a_j w_j\in W.
\]
But the original sequence \((u_k)\) converges to \(u\), and limits in a metric space are unique. Therefore
\[
 u=T(a)\in W.
\]
So every convergent sequence in \(W\) converges to a point of \(W\). Hence \(W\) is closed.
\end{proof}

\begin{lemma}[Translation invariance implies smoothness]
\label{lem:translation_implies_smooth}
Let \(V\subset C(\R^n)\) be a finite-dimensional translation-invariant subspace, meaning that
\[
 f\in V \implies f(\cdot+a)\in V
 \qquad\text{for every }a\in\R^n.
\]
Then every element of \(V\) is infinitely differentiable, and every partial derivative maps \(V\) to itself.
\end{lemma}

\begin{proof}
Fix \(f\in V\). Let \(\varphi\in C_c^{\infty}(\R^n)\) be a smooth compactly supported test function. Define the convolution
\[
 (f*\varphi)(x)=\int_{\R^n} f(x-y)\varphi(y)\,dy.
\]
It is standard that \(f*\varphi\in C^{\infty}(\R^n)\). We claim that \(f*\varphi\in V\).

For each \(y\in\R^n\), define the translate
\[
 T_y f(x):=f(x-y).
\]
By translation invariance, \(T_y f\in V\). The convolution can be written as
\[
 (f*\varphi)(x)=\int_{\R^n} (T_y f)(x)\varphi(y)\,dy.
\]
Approximate this integral by Riemann sums:
\[
 S_m(x):=\sum_{r=1}^{N_m} \alpha_{m,r} (T_{y_{m,r}}f)(x).
\]
Each \(S_m\) is a finite linear combination of elements of \(V\), so \(S_m\in V\). The Riemann sums converge to \(f*\varphi\) uniformly on compact subsets of \(\R^n\). Since \(V\) is finite-dimensional, Lemma~\ref{lem:finite_dim_closed} implies that \(V\) is closed in the compact-open topology. Therefore the limit \(f*\varphi\) also belongs to \(V\).

Now choose a standard mollifier family \((\varphi_\delta)_{\delta>0}\) with \(\varphi_\delta\in C_c^{\infty}(\R^n)\) and
\[
 f*\varphi_\delta \to f
\]
uniformly on compact subsets as \(\delta\to 0\). Each mollified function belongs to \(V\cap C^{\infty}(\R^n)\). Hence \(V\cap C^{\infty}(\R^n)\) is dense in \(V\). Because \(V\) is finite-dimensional, a dense linear subspace must equal the whole space. Therefore
\[
 V\subset C^{\infty}(\R^n).
\]
Now let \(f\in V\), and fix a coordinate direction \(e_j\). For \(h\neq 0\), the difference quotient is
\[
 \frac{f(\cdot+he_j)-f}{h}.
\]
Because \(f(\cdot+he_j)\in V\) and \(V\) is a linear space, every such difference quotient lies in \(V\). Since \(f\) is now known to be smooth, these difference quotients converge in the compact-open topology to the partial derivative \(\partial_j f\). By closedness of \(V\), the limit also lies in \(V\). Thus every partial derivative maps \(V\) to itself.
\end{proof}

\begin{lemma}[Vanishing high-order derivatives force a polynomial]
\label{lem:vanishing_high_derivatives_polynomial}
Let \(K\ge 1\), and let \(f\in C^K(\R^n)\). Assume that
\[
 \partial^{\alpha}f\equiv 0
 \qquad\text{for every multi-index }\alpha\text{ with }|\alpha|=K.
\]
Then \(f\) is a polynomial of total degree at most \(K-1\).
\end{lemma}

\begin{proof}
We apply the multivariate Taylor formula at the origin with integral remainder. For every \(x\in\R^n\),
\[
 f(x)
 =
 \sum_{|\alpha|\le K-1} \frac{\partial^{\alpha}f(0)}{\alpha!}x^{\alpha}
 +
 \sum_{|\alpha|=K} \frac{K}{\alpha!}x^{\alpha}
 \int_0^1 (1-t)^{K-1} \partial^{\alpha}f(tx)\,dt.
\]
By assumption, every derivative \(\partial^{\alpha}f\) of total order \(K\) vanishes identically, so every integral in the remainder term is zero. Therefore the remainder vanishes, and we are left with
\[
 f(x)=\sum_{|\alpha|\le K-1} \frac{\partial^{\alpha}f(0)}{\alpha!}x^{\alpha}.
\]
The right-hand side is a polynomial of total degree at most \(K-1\). Hence \(f\) is such a polynomial.
\end{proof}

\begin{lemma}[Joint generalized eigenspace decomposition for commuting operators]
\label{lem:joint_generalized_eigenspaces}
Let \(W\) be a finite-dimensional complex vector space, and let
\[
 T_1,\dots,T_n\in \operatorname{End}_{\C}(W)
\]
be pairwise commuting linear operators. Then there exists a finite set \(\Sigma\subset\C^n\) such that
\[
 W=\bigoplus_{\mu\in\Sigma} W_{\mu},
\]
where
\[
 W_{\mu}:=\bigcap_{j=1}^n \Ker\bigl(T_j-\mu_j I\bigr)^N
\]
for any integer \(N\ge \dim W\). On each \(W_{\mu}\), every operator \(T_j-\mu_j I\) is nilpotent.
\end{lemma}

\begin{proof}
We argue by induction on the number \(n\) of commuting operators. In the base case \(n=1\), for a single operator \(T_1\), the standard primary decomposition theorem gives
\[
 W=\bigoplus_{\lambda\in\operatorname{Spec}(T_1)} \Ker(T_1-\lambda I)^N
\]
for every \(N\ge \dim W\). This is exactly the desired statement.

Now, assume the lemma is known for \(n-1\) commuting operators, and consider \(T_1,\dots,T_n\). Apply the base case to \(T_1\):
\[
 W=\bigoplus_{\lambda\in\operatorname{Spec}(T_1)} G_{\lambda},
 \qquad
 G_{\lambda}:=\Ker(T_1-\lambda I)^N,
\]
where \(N\ge \dim W\). Because the operators commute, every \(T_j\) with \(j\ge 2\) preserves every generalized eigenspace \(G_{\lambda}\). Indeed, if \(v\in G_{\lambda}\), then
\[
 (T_1-\lambda I)^N(T_j v)=T_j(T_1-\lambda I)^N v=T_j(0)=0.
\]
So \(T_j v\in G_{\lambda}\).

Now fix one \(\lambda\). The restricted operators
\[
 T_2|_{G_{\lambda}},\dots,T_n|_{G_{\lambda}}
\]
are pairwise commuting endomorphisms of the finite-dimensional complex vector space \(G_{\lambda}\). By the inductive hypothesis, there is a finite set \(\Sigma_{\lambda}\subset\C^{n-1}\) such that
\[
 G_{\lambda}=\bigoplus_{(\mu_2,\dots,\mu_n)\in\Sigma_{\lambda}} G_{\lambda,\mu_2,\dots,\mu_n},
\]
where on each summand \(G_{\lambda,\mu_2,\dots,\mu_n}\) the operators
\[
 T_j-\mu_j I\qquad (j=2,\dots,n)
\]
are nilpotent, and also
\[
 G_{\lambda,\mu_2,\dots,\mu_n}
 \subset
 \Ker(T_1-\lambda I)^N.
\]
Hence on that same summand the operator \(T_1-\lambda I\) is nilpotent as well.

Collecting all these summands over all \(\lambda\) gives a direct-sum decomposition
\[
 W=\bigoplus_{\mu\in\Sigma} W_{\mu},
\]
where \(\Sigma\subset\C^n\) consists of tuples \(\mu=(\lambda,\mu_2,\dots,\mu_n)\). By construction,
\[
 W_{\mu}\subset \bigcap_{j=1}^n \Ker(T_j-\mu_j I)^N,
\]
and on \(W_{\mu}\) every \(T_j-\mu_j I\) is nilpotent.

Conversely, if a vector \(v\) lies in \(\bigcap_{j=1}^n \Ker(T_j-\mu_j I)^N\), then it lies in the \(\mu_1\)-generalized eigenspace of \(T_1\), and inside that space it lies in the joint generalized eigenspace corresponding to \((\mu_2,\dots,\mu_n)\) for the remaining operators. Hence it lies in the summand already denoted \(W_{\mu}\). Therefore
\[
 W_{\mu}=\bigcap_{j=1}^n \Ker(T_j-\mu_j I)^N.
\]
This completes the induction.
\end{proof}

\begin{lemma}[A nonzero \(\GL(n)\)-invariant subspace of homogeneous polynomials is the whole homogeneous space]
\label{lem:homogeneous_irreducibility_constructive}
Let \(\mathcal H_d\) denote the space of real homogeneous polynomials of total degree \(d\) on \(\R^n\). If \(U\subset \mathcal H_d\) is a nonzero linear subspace satisfying
\[
 p\in U,\; A\in\GL(n)
 \implies
 p\circ A\in U,
\]
then \(U=\mathcal H_d\).
\end{lemma}

\begin{proof}
We will prove this constructively. Choose any nonzero polynomial \(p\in U\). Because \(p\) is homogeneous of degree \(d\), it has an expansion
\[
 p(x)=\sum_{|\alpha|=d} c_{\alpha} x^{\alpha},
\]
where the sum runs over all multi-indices \(\alpha=(\alpha_1,\dots,\alpha_n)\) with total degree \(|\alpha|=d\).

At least one coefficient \(c_{\alpha}\) is nonzero. Consider the action of diagonal matrices
\[
 D(t)=\diag(t_1,\dots,t_n),
 \qquad t_j>0.
\]
Then
\[
 (p\circ D(t))(x)
 =
 p(t_1x_1,\dots,t_nx_n)
 =
 \sum_{|\alpha|=d} c_{\alpha} t^{\alpha} x^{\alpha},
\]
where
\[
 t^{\alpha}:=t_1^{\alpha_1}\cdots t_n^{\alpha_n}.
\]
Because \(U\) is \(\GL(n)\)-invariant, each polynomial \(p\circ D(t)\) lies in \(U\).

Let
\[
 S:=\{\alpha:|\alpha|=d\text{ and }c_{\alpha}\neq 0\}
\]
be the support of \(p\). This set is nonempty because \(p\neq 0\). Choose an integer \(L>d\), choose distinct positive numbers \(s_1,\dots,s_{|S|}\), and set
\[
 t^{(r)}:=\bigl(s_r, s_r^L, s_r^{L^2},\dots,s_r^{L^{n-1}}\bigr)\in (0,\infty)^n,
 \qquad r=1,\dots,|S|.
\]
For a multi-index \(\alpha=(\alpha_1,\dots,\alpha_n)\) with \(|\alpha|=d\), we then have
\[
 (t^{(r)})^{\alpha}=s_r^{e(\alpha)},\qquad e(\alpha):=\alpha_1+\alpha_2L+\cdots+\alpha_nL^{n-1}.
\]
Because each \(\alpha_j\) lies in \(\{0,1,\dots,d\}\) and \(L>d\), the base-\(L\) representation is unique, so distinct multi-indices \(\alpha\) give distinct exponents \(e(\alpha)\).

Restrict attention to the supported monomials, that is, to the columns indexed by \(S\). For \(r=1,\dots,|S|\), the polynomial \(p\circ D(t^{(r)})\) has the expansion
\[
 p\circ D(t^{(r)})
 =
 \sum_{\alpha\in S} c_{\alpha}(t^{(r)})^{\alpha}x^{\alpha}
 =
 \sum_{\alpha\in S} c_{\alpha}s_r^{e(\alpha)}x^{\alpha}.
\]
Therefore the coefficient matrix of the family
\[
 p\circ D(t^{(1)}),\dots,p\circ D(t^{(|S|)})
\]
with respect to the supported monomial basis \(\{x^{\alpha}:\alpha\in S\}\) is
\[
 \Bigl[c_{\alpha}s_r^{e(\alpha)}\Bigr]_{1\le r\le |S|,\; \alpha\in S}.
\]
This matrix factors as
\[
 \Bigl[s_r^{e(\alpha)}\Bigr]_{1\le r\le |S|,\; \alpha\in S}\cdot \diag(c_{\alpha})_{\alpha\in S}.
\]
The first factor is a Vandermonde matrix in the distinct numbers \(s_r\) and the distinct exponents \(e(\alpha)\), so it is invertible. The second factor is also invertible because every \(c_{\alpha}\) with \(\alpha\in S\) is nonzero. Hence the full coefficient matrix is invertible.

It follows that the span of these transformed polynomials is exactly the span of the supported monomials:
\[
 \Span\bigl\{p\circ D(t^{(r)}):1\le r\le |S|\bigr\}
 =
 \Span\bigl\{x^{\alpha}:\alpha\in S\bigr\}.
\]
In particular, every supported monomial \(x^{\alpha}\) belongs to \(U\). Therefore \(U\) contains at least one monomial.

Now, suppose
\[
 x^{\alpha}=x_1^{\alpha_1}\cdots x_n^{\alpha_n}\in U.
\]
Fix indices \(i\neq j\) with \(\alpha_i>0\). Consider the shear matrix \(S_{ij}(s)\in\GL(n)\) defined by
\[
 x_i\mapsto x_i+s x_j,
 \qquad
 x_r\mapsto x_r\;\text{ for }r\neq i.
\]
Because \(U\) is \(\GL(n)\)-invariant,
\[
 x^{\alpha}\circ S_{ij}(s)
 =
 (x_i+s x_j)^{\alpha_i} x_j^{\alpha_j} \prod_{r\neq i,j}x_r^{\alpha_r}
\in U
\]
for every real \(s\). Expand the binomial:
\[
 (x_i+s x_j)^{\alpha_i}
 =
 \sum_{m=0}^{\alpha_i} \binom{\alpha_i}{m} s^m x_i^{\alpha_i-m}x_j^m.
\]
Thus
\[
 x^{\alpha}\circ S_{ij}(s)
 =
 \sum_{m=0}^{\alpha_i}
 \binom{\alpha_i}{m}s^m
 x_i^{\alpha_i-m}x_j^{\alpha_j+m}
 \prod_{r\neq i,j}x_r^{\alpha_r}.
\]
Choose \(\alpha_i+1\) distinct values of \(s\). The resulting univariate Vandermonde matrix in the powers \(1,s,s^2,\dots,s^{\alpha_i}\) is invertible. Therefore each individual monomial
\[
 x_i^{\alpha_i-m}x_j^{\alpha_j+m}\prod_{r\neq i,j}x_r^{\alpha_r}
\]
belongs to \(U\). Thus, we may transfer any amount of degree from coordinate \(i\) to coordinate \(j\) while remaining inside \(U\).

Starting from the monomial \(x^{\alpha}\in U\), repeatedly apply the above to move all degree from coordinates \(2,3,\dots,n\) into coordinate \(1\). After finitely many steps, we obtain
\[
 x_1^d\in U.
\]
Consider the linear map
\[
 T(s_2,\dots,s_n):
 x_1\mapsto x_1+s_2x_2+\cdots+s_nx_n,
 \qquad
 x_r\mapsto x_r\;\text{ for }r\ge 2.
\]
Because \(U\) is \(\GL(n)\)-invariant,
\[
 x_1^d\circ T(s_2,\dots,s_n)
 =
 (x_1+s_2x_2+\cdots+s_nx_n)^d
 \in U.
\]
Expanding by the multinomial theorem gives
\[
 (x_1+s_2x_2+\cdots+s_nx_n)^d
 =
 \sum_{|\beta|=d}
 \binom{d}{\beta}
 s_2^{\beta_2}\cdots s_n^{\beta_n}
 x^{\beta}.
\]
Now choose \(M_d\) parameter tuples on the same moment-curve pattern as above, where \(M_d\) is the number of degree-\(d\) monomials. With respect to the full degree-\(d\) monomial basis, the resulting coefficient matrix has entries
\[
 \binom{d}{\beta} s_r^{e(\beta)}.
\]
This matrix factors as a Vandermonde matrix in the distinct exponents \(e(\beta)\) times the invertible diagonal matrix of multinomial coefficients \(\binom{d}{\beta}\). Hence it is invertible. Therefore every degree-\(d\) monomial \(x^{\beta}\) is a linear combination of these transformed polynomials, and therefore lies in \(U\).

Since the degree-\(d\) monomials form a basis of \(\mathcal H_d\), it follows that
\[
 U=\mathcal H_d.
\]
This proves the lemma.
\end{proof}

\paragraph{Proof sketch.}
The lemma says that once a nonzero homogeneous polynomial is available and the whole general linear group is allowed to move the coordinates, the orbit is so rich that it generates every monomial of that degree. The proof makes this explicit using diagonal scalings, shears, and Vandermonde inversion.

\subsubsection{Proof of Theorem~\ref{thm:scalar_affine_classification}}

\thmScalarClassification*

\begin{proof}
We begin with the degenerate case. If \(V=\{0\}\), then we are in case (1) of the theorem and there is nothing more to prove. For the rest of the argument assume
\[
 V\neq \{0\}.
\]
We will prove that then \(V=\Poly{\ell}(\R^n)\) for some \(\ell\ge 0\). Because the affine group contains all translations, the hypothesis implies that
\[
 f\in V \implies f(\cdot+a)\in V
 \qquad\text{for every }a\in\R^n.
\]
Hence \(V\) is translation-invariant. By Lemma~\ref{lem:translation_implies_smooth}, every function in \(V\) is smooth and every partial derivative preserves \(V\). Let
\[
 D_j:=\partial_j:V\to V
 \qquad (j=1,\dots,n).
\]
Because mixed partial derivatives commute on smooth functions, the operators \(D_1,\dots,D_n\) commute. Now complexify the space:
\[
 V_{\C}:=V\otimes_{\R}\C.
\]
The operators \(D_j\) extend \(\C\)-linearly to commuting endomorphisms of \(V_{\C}\). Let
\[
 m:=\dim_{\C} V_{\C}.
\]
By Lemma~\ref{lem:joint_generalized_eigenspaces}, there exists a finite set \(\Sigma\subset\C^n\) such that
\[
 V_{\C}=\bigoplus_{\mu\in\Sigma} V_{\mu},
\qquad
 V_{\mu}=\bigcap_{j=1}^n \Ker(D_j-\mu_j I)^m.
\]
In other words, every element of \(V_{\C}\) splits uniquely into finitely many components, each belonging to a joint generalized eigenspace labeled by a tuple \(\mu=(\mu_1,\dots,\mu_n)\in\Sigma\).

Now, fix one \(\mu\in\Sigma\), and let \(u\in V_{\mu}\). By definition,
\[
 (D_j-\mu_j I)^m u=0
 \qquad\text{for every }j=1,\dots,n.
\]
Define
\[
 v(x):=e^{-\langle \mu,x\rangle}u(x).
\]
We claim that for each \(j\),
\[
 \partial_j^m v=0.
\]
Indeed, for one derivative we have
\[
 \partial_j\bigl(e^{-\langle \mu,x\rangle}u(x)\bigr)
 =
 -\mu_j e^{-\langle \mu,x\rangle}u(x)+e^{-\langle \mu,x\rangle}\partial_j u(x)
 =
 e^{-\langle \mu,x\rangle}(D_j-\mu_j I)u(x).
\]
Applying the same identity repeatedly gives
\[
 \partial_j^m v
 =
 e^{-\langle \mu,x\rangle}(D_j-\mu_j I)^m u
 =0.
\]
Now set
\[
 K:=n(m-1)+1.
\]
Take any multi-index \(\alpha=(\alpha_1,\dots,\alpha_n)\) with \(|\alpha|=K\). Since \(K\) is larger than the sum of the bounds \((m-1)+\cdots+(m-1)\), at least one coordinate satisfies \(\alpha_j\ge m\). Therefore
\[
 \partial^{\alpha}v=0,
\]
because one factor \(\partial_j^m\) already annihilates \(v\).

So every derivative of total order \(K\) vanishes identically on \(v\). By Lemma~\ref{lem:vanishing_high_derivatives_polynomial}, \(v\) is a polynomial of total degree at most \(K-1\). Thus there exists a polynomial \(p\) such that
\[
 u(x)=e^{\langle \mu,x\rangle}p(x).
\]
We have shown that every element of every joint generalized eigenspace is an exponential polynomial, with a degree bound depending only on \(m\).

Let \(u\in V_{\mu}\) be nonzero. For every nonzero scalar \(t\), define
\[
 u_t(x):=u(tx).
\]
Because dilations belong to \(\Aff(n)\), affine invariance implies \(u_t\in V_{\C}\). We now compute how the generalized eigenvalue tuple changes under dilation. For one derivative,
\[
 \partial_j u_t(x)=t\,(\partial_j u)(tx).
\]
Hence
\[
 (D_j-t\mu_j I)u_t(x)=t\bigl((D_j-\mu_j I)u\bigr)(tx).
\]
Repeating this identity \(m\) times gives
\[
 (D_j-t\mu_j I)^m u_t(x)=t^m\bigl((D_j-\mu_j I)^m u\bigr)(tx)=0.
\]
So \(u_t\in V_{t\mu}\). If \(\mu\neq 0\), then the tuples \(t\mu\) are all distinct as \(t\) ranges over nonzero real numbers. Since \(u_t\neq 0\), we would obtain nonzero vectors in infinitely many distinct joint generalized eigenspaces \(V_{t\mu}\). That is impossible, as we have produced only finitely many such eigenspaces. Therefore every joint generalized eigenvalue tuple must be zero:
\[
 \Sigma=\{0\}.
\]
Since \(\Sigma=\{0\}\), every element of \(V_{\C}\) is of the form
\[
 e^{\langle 0,x\rangle}p(x)=p(x),
\]
where \(p\) is a polynomial of total degree at most \(K-1\). In particular,
\[
 V\subset \Poly{K-1}(\R^n).
\]
So every element of \(V\) is a polynomial, and there is a uniform degree bound across the whole space.

Fix \(f\in V\). Since \(f\) is a polynomial of degree at most \(K-1\), it has a unique decomposition into homogeneous pieces:
\[
 f=f_0+f_1+\cdots+f_{K-1},
\]
where \(f_d\) is homogeneous of degree \(d\) and some pieces may be zero.

For every nonzero scalar \(t\), dilation invariance gives
\[
 f(tx)=f\circ S_t(x)\in V.
\]
But
\[
 f(tx)=f_0(x)+t f_1(x)+\cdots+t^{K-1}f_{K-1}(x).
\]
Choose \(K\) distinct nonzero scalars \(t_0,\dots,t_{K-1}\). Then the functions
\[
 f(t_0x),\dots,f(t_{K-1}x)
\]
all lie in \(V\), and their coefficients with respect to \(f_0,\dots,f_{K-1}\) form the Vandermonde matrix
\[
 [t_r^d]_{0\le r,d\le K-1}.
\]
Because the \(t_r\) are distinct, this Vandermonde matrix is invertible. Therefore each homogeneous component \(f_d\) is a linear combination of the functions \(f(t_r x)\), hence belongs to \(V\).

Define
\[
 U_d:=V\cap \mathcal H_d,
\]
where \(\mathcal H_d\) is the space of homogeneous degree-\(d\) polynomials.

Because \(V\neq 0\) and all elements of \(V\) are polynomials, there is a largest integer \(\ell\) such that \(U_{\ell}\neq\{0\}\). The space \(U_{\ell}\) is nonzero and \(\GL(n)\)-invariant, because \(V\) is affine-invariant and homogeneous degree is preserved by invertible linear changes of variables. By Lemma~\ref{lem:homogeneous_irreducibility_constructive},
\[
 U_{\ell}=\mathcal H_{\ell}.
\]

We now show that all lower homogeneous spaces also lie in \(V\). Let \(x^{\beta}\) be any monomial of degree \(\ell-1\). Choose an index \(j\in\{1,\dots,n\}\), and set \(\alpha=\beta+e_j\), so that \(|\alpha|=\ell\). Then
\[
 \partial_j x^{\alpha}=(\beta_j+1)x^{\beta}.
\]
Since \(x^{\alpha}\in \mathcal H_{\ell}\subset V\) and derivatives preserve \(V\), the monomial \(x^{\beta}\) lies in \(V\). Because the monomials of degree \(\ell-1\) span \(\mathcal H_{\ell-1}\), we obtain
\[
 \mathcal H_{\ell-1}\subset V.
\]
Applying the same argument repeatedly gives
\[
 \mathcal H_d\subset V
 \qquad\text{for every }0\le d\le \ell.
\]
Therefore
\[
 \Poly{\ell}(\R^n)=\mathcal H_0\oplus\cdots\oplus\mathcal H_{\ell}\subset V.
\]
On the other hand, by the definition of \(\ell\), no nonzero homogeneous component of degree greater than \(\ell\) can occur in \(V\). So
\[
 V\subset \Poly{\ell}(\R^n).
\]
Combining the two inclusions yields
\[
 V=\Poly{\ell}(\R^n).
\]
This proves case (2) of the theorem.

For the converse direction, the zero space is trivially finite-dimensional and affine-invariant. Now fix \(\ell\ge 0\). The space \(\Poly{\ell}(\R^n)\) is finite-dimensional because it has a finite monomial basis. It is affine-invariant because if
\[
 p\in\Poly{\ell}(\R^n)
 \qquad\text{and}\qquad
 g(x)=Ax+b\in\Aff(n),
\]
then substituting the affine-linear expressions \(Ax+b\) into \(p\) produces another polynomial of total degree at most \(\ell\). Hence
\[
 p\circ g\in\Poly{\ell}(\R^n).
\]
So \(\Poly{\ell}(\R^n)\) is finite-dimensional and affine-invariant. This completes the proof.
\end{proof}

\paragraph{Proof sketch.}
Translation invariance gives closure under derivatives, and the joint generalized-eigenspace decomposition says that finite-dimensional translation-invariant spaces are built from exponential-polynomial pieces. Full affine invariance is much stricter: dilations send a character \(\mu\) to the whole ray \(t\mu\), so finite-dimensionality forces every nonzero character to disappear. Once only the zero character remains, every function is polynomial. The final \(\GL(n)\)-invariance argument says that one nonzero homogeneous degree drags in the whole homogeneous space of that degree, and differentiation then fills in all lower degrees.

\subsubsection{Proof of Corollary~\ref{cor:vector_valued_probe_family}}

\corVectorValued*

\begin{proof}
Let \(\mathcal F\subset C(\R^n,\R^k)\)
be a finite-dimensional affine-invariant linear subspace. For each coordinate index \(r\in\{1,\dots,k\}\), define the scalar projection
\[
 \pi_r:\R^k\to\R,
 \qquad
 \pi_r(y_1,\dots,y_k)=y_r.
\]
Now define the scalar coordinate space
\[
 V_r:=\{\pi_r\circ f : f\in\mathcal F\}\subset C(\R^n).
\]
Because \(\mathcal F\) is finite-dimensional, each \(V_r\) is finite-dimensional. Because \(\mathcal F\) is affine-invariant under precomposition, each \(V_r\) is also affine-invariant: if \(f\in\mathcal F\) and \(g\in\Aff(n)\), then
\[
 \pi_r\circ(f\circ g)=(\pi_r\circ f)\circ g
\]
and \(f\circ g\in\mathcal F\).

Therefore Theorem~\ref{thm:scalar_affine_classification} applies to each \(V_r\). For each coordinate \(r\), exactly one of the following happens:
\begin{itemize}[leftmargin=2em]
    \item \(V_r=\{0\}\), in which case every \(r\)-th coordinate is identically zero;
    \item \(V_r=\Poly{\ell_r}(\R^n)\) for some \(\ell_r\ge 0\).
\end{itemize}
Let \(I\subset\{1,\dots,k\}\) be the set of indices for which \(V_r\neq\{0\}\). If \(I\) is empty, then every coordinate is identically zero, so \(\mathcal F=\{0\}\), and the corollary is immediate.

Assume now that \(I\) is nonempty, and define
\[
 \ell:=\max\{\ell_r:r\in I\}.
\]
Then every nonzero coordinate function \(\pi_r\circ f\) of every \(f\in\mathcal F\) is a polynomial of total degree at most \(\ell\), and every remaining coordinate is identically zero. This means exactly that every \(f\in\mathcal F\) is a vector-valued polynomial map of degree at most \(\ell\):
\[
 f\in \Poly{\ell}(\R^n,\R^k).
\]
Hence
\[
 \mathcal F\subset \Poly{\ell}(\R^n,\R^k).
\]
This proves the corollary.
\end{proof}

\paragraph{Proof sketch.}
A vector-valued probe is just a stack of scalar probes. Once each scalar coordinate is forced to be polynomial, the whole vector-valued map is polynomial coordinatewise.

\subsubsection{Proof of Proposition~\ref{prop:low_rank_invariance}}

\propLowRank*

\begin{proof}
Let
\[
 p(x)=\sum_{r=1}^R \alpha_r\prod_{j=1}^m \ip{v_{r,j}}{x}
\]
be a rank-\(R\) homogeneous CP probe, and let \(A\in\GL(n)\). Then
\begin{align*}
 (p\circ A)(x)
 &= p(Ax)\\
 &= \sum_{r=1}^R \alpha_r\prod_{j=1}^m \ip{v_{r,j}}{Ax}\\
 &= \sum_{r=1}^R \alpha_r\prod_{j=1}^m \ip{A^\top v_{r,j}}{x}.
\end{align*}
This has the same CP-rank bound \(R\): it is obtained by replacing each factor vector \(v_{r,j}\) with \(A^\top v_{r,j}\). Therefore the family is closed under precomposition by \(\GL(n)\).

Now, we show that monomial sparsity is not stable in general under \(\GL(n)\) transformations. It is enough to provide a counterexample. Work in dimension \(n=2\) and degree \(2\). Consider the 1-sparse polynomial
\[
 q(x_1,x_2)=x_1^2.
\]
Now apply the invertible linear transformation
\[
 A=
 \begin{bmatrix}
 1 & 1\\
 0 & 1
 \end{bmatrix},
 \qquad
 A(x_1,x_2)^\top=(x_1+x_2,\,x_2)^\top.
\]
Then
\[
 (q\circ A)(x_1,x_2)=q(x_1+x_2,x_2)=(x_1+x_2)^2=x_1^2+2x_1x_2+x_2^2.
\]
The transformed polynomial has three nonzero monomial coefficients instead of one. Therefore monomial sparsity is not preserved in general under invertible linear reparameterizations. This proves the proposition.
\end{proof}

\paragraph{Proof sketch.}
Low-rank tensor structure is geometric: it is expressed through linear forms, and linear changes of coordinates simply transform those forms. Monomial sparsity is coordinate-dependent: once the coordinates are mixed, a sparse monomial expansion typically becomes dense.

\subsubsection{A complementary one-dimensional translation theorem}

The next proposition is not needed for the main classification theorem but provides additional intuition. It describes finite-dimensional translation-invariant subspaces of \(C(\R)\) without assuming affine invariance.

\begin{proposition}[One-dimensional translation-invariant spaces]
\label{prop:oned_translation}
Let \(E\subset C(\R)\) be a finite-dimensional real vector space such that
\[
 f\in E \implies f(\cdot+a)\in E
 \qquad\text{for every }a\in\R.
\]
Then every element of \(E\) is an exponential polynomial. More precisely, there exist complex numbers
\[
 \mu_1,\dots,\mu_r\in\C
\]
and nonnegative integers \(d_1,\dots,d_r\) such that
\[
 E=\Span_{\R}\Bigl\{e^{\lambda_j x}x^k\cos(\omega_j x),\; e^{\lambda_j x}x^k\sin(\omega_j x)
 : 1\le j\le r,\; 0\le k\le d_j\Bigr\},
\]
where \(\mu_j=\lambda_j+i\omega_j\).
\end{proposition}

\begin{proof}
By Lemma~\ref{lem:translation_implies_smooth}, the space \(E\) consists of smooth functions and is closed under differentiation. Consider the differentiation operator
\[
 D:=\frac{d}{dx}:E\to E.
\]
Complexify \(E\) to obtain \(E_{\C}\), and decompose \(E_{\C}\) into generalized eigenspaces of \(D\):
\[
 E_{\C}=U_{\mu_1}\oplus\cdots\oplus U_{\mu_r}.
\]
For each eigenvalue \(\mu\), the generalized eigenspace consists of solutions of
\[
 (D-\mu I)^N f=0
\]
for some \(N\). Write \(f(x)=e^{\mu x}g(x)\). Then
\[
 (D-\mu I)(e^{\mu x}g(x))=e^{\mu x}g'(x).
\]
Applying this identity repeatedly shows that
\[
 (D-\mu I)^N (e^{\mu x}g(x))=e^{\mu x}g^{(N)}(x).
\]
Thus \((D-\mu I)^N f=0\) if and only if \(g^{(N)}=0\), which means \(g\) is a polynomial of degree at most \(N-1\). Therefore each generalized eigenspace is spanned over \(\C\) by functions of the form \(x^k e^{\mu x}.\) Taking real and imaginary parts gives the claimed real basis.
\end{proof}

\paragraph{Proof sketch.}
Translation invariance alone allows exponentials, exponentials multiplied by polynomials, and oscillatory versions of the same objects. The full affine group is more restrictive because scaling kills every exponential mode except the trivial one, leaving only ordinary polynomials.

\section{Detailed Proofs for Section~\ref{sec:shared_space}}
\label{app:shared_space_proofs}

This appendix mirrors the progression of Section~\ref{sec:shared_space}. We first prove the finite-dimensional evaluation lemma, then the exact shared-space theorem, and finally the approximate finite-bank results used to interpret the transfer experiments. The accompanying proof sketches are meant to keep the conceptual thread visible as the algebra becomes more detailed.

\subsection{Proof of Lemma~\ref{lem:evaluations_span}}

\lemEvaluations*

\begin{proof}
Let
\[
 W:=\Span\{\operatorname{ev}_x : x\in X\}\subset \mathcal C^*.
\]
We want to show that \(W=\mathcal C^*\). Because \(\mathcal C\) is finite-dimensional, it is enough to show that the annihilator of \(W\) inside \(\mathcal C\) is trivial. Recall that
\[
 W^{\perp}:=\{c\in\mathcal C : \phi(c)=0\text{ for all }\phi\in W\}.
\]
Take any \(c\in W^{\perp}\). Since every evaluation functional \(\operatorname{ev}_x\) belongs to \(W\), we have
\[
 \operatorname{ev}_x(c)=0
 \qquad\text{for every }x\in X.
\]
But \(\operatorname{ev}_x(c)=c(x)\), so this means
\[
 c(x)=0
 \qquad\text{for every }x\in X.
\]
Hence \(c\) is the zero function. Therefore
\[
 W^{\perp}=\{0\}.
\]
In a finite-dimensional vector space, a subspace has trivial annihilator if and only if it is the whole dual space. Therefore \(W=\mathcal C^*.\) This proves the lemma.
\end{proof}

\paragraph{Proof sketch.}
If evaluations did not span the dual, there would be a nonzero concept invisible to every pointwise evaluation. That is impossible, because a nonzero function must take a nonzero value somewhere.

\subsection{Proof of Theorem~\ref{thm:common_abstract_space}}

\thmCommonSpace*

\begin{proof}
Fix \(i\in\{1,2\}\). Because the concept families coincide, we may regard
\[
 E_i:V_i\to\mathcal C,
 \qquad
 E_i(\ell)=\ell\circ h_i,
\]
as a linear map into the common space \(\mathcal C\). By assumption, \(E_i\) is injective, and its image is exactly \(\mathcal C\). Therefore \(E_i\) is a linear isomorphism.

For each \(v\in H_i\), define a functional \(Q_i(v)\in\mathcal C^*\) by the rule
\[
 Q_i(v)(E_i(\ell)):=\ell(v)
 \qquad\text{for every }\ell\in V_i.
\]
This is well-defined because every \(c\in\mathcal C\) can be written uniquely as \(c=E_i(\ell)\).

We next prove linearity in \(v\). Let \(v,w\in H_i\) and \(a,b\in\R\). For any \(\ell\in V_i\),
\begin{align*}
 Q_i(av+bw)(E_i(\ell))
 &= \ell(av+bw)\\
 &= a\ell(v)+b\ell(w)\\
 &= aQ_i(v)(E_i(\ell))+bQ_i(w)(E_i(\ell)).
\end{align*}
Since the vectors \(E_i(\ell)\) span \(\mathcal C\), this implies
\[
 Q_i(av+bw)=aQ_i(v)+bQ_i(w).
\]
So \(Q_i\) is linear. The same formula also gives uniqueness. Suppose \(\widetilde Q_i:H_i\to\mathcal C^*\) is another linear map such that
\[
 \widetilde Q_i(v)(E_i(\ell))=\ell(v)
 \qquad\text{for every }v\in H_i\text{ and every }\ell\in V_i.
\]
Then for every \(v\in H_i\) the functionals \(Q_i(v)\) and \(\widetilde Q_i(v)\) agree on the spanning set \(E_i(V_i)=\mathcal C\), hence they are equal. Therefore \(\widetilde Q_i=Q_i\).

Now, we verify the interpolation property on represented points. Fix \(x\in X\). Let \(c\in\mathcal C\), and write \(c=E_i(\ell)=\ell\circ h_i\). Then
\[
 Q_i(h_i(x))(c)
 =Q_i(h_i(x))(E_i(\ell))
 =\ell(h_i(x))
 =c(x)
 =\operatorname{ev}_x(c).
\]
Since this holds for every \(c\in\mathcal C\), we obtain
\[
 Q_i(h_i(x))=\operatorname{ev}_x.
\]
This proves the stated interpolation consequence.

Take \(v\in H_i\). By definition, \(Q_i(v)=0\) if and only if
\[
 Q_i(v)(c)=0
 \qquad\text{for every }c\in\mathcal C.
\]
Because every \(c\in\mathcal C\) has the form \(c=E_i(\ell)\) for a unique \(\ell\in V_i\), this is equivalent to
\[
 \ell(v)=0
 \qquad\text{for every }\ell\in V_i.
\]
That is exactly the condition that
\[
 v\in\bigcap_{\ell\in V_i}\Ker\ell=K(V_i).
\]
Therefore \(\Ker Q_i = K(V_i).\) This proves claim (2). Now, we previously showed that
\[
 Q_i(h_i(x))=\operatorname{ev}_x
 \qquad\text{for all }x\in X.
\]
Hence the image of \(Q_i\) contains every evaluation functional \(\operatorname{ev}_x\). By Lemma~\ref{lem:evaluations_span}, those evaluation functionals span the whole dual space \(\mathcal C^*\). Therefore
\[
 \Image Q_i = \mathcal C^*.
\]
So \(Q_i\) is surjective. Since \(\Ker Q_i = K(V_i)\), the map \(Q_i\) factors uniquely through the quotient:
\[
 H_i \xrightarrow{\pi_i} H_i/K(V_i) \xrightarrow{\overline Q_i} \mathcal C^*,
\]
where \(\pi_i\) is the quotient projection and \(Q_i=\overline Q_i\circ \pi_i\). The map \(\overline Q_i\) is injective because its kernel is trivial: if \(\overline Q_i([v])=0\), then \(Q_i(v)=0\), so \(v\in K(V_i)\), hence \([v]=0\) in the quotient. It is surjective because \(Q_i\) is surjective. Therefore \(\overline Q_i\) is a linear isomorphism. This proves claim (3).

Finally, because both quotients are canonically isomorphic to \(\mathcal C^*\), they are canonically isomorphic to each other:
\[
 H_1/K(V_1)\cong \mathcal C^* \cong H_2/K(V_2).
\]
This completes the proof.
\end{proof}

\paragraph{Proof sketch.}
The common concept family \(\mathcal C\) is the real shared object. A hidden vector \(v\in H_i\) acts on the concept family by first expressing a concept as a probe \(\ell\in V_i\) and then evaluating that probe on \(v\). That is exactly what the formula for \(Q_i\) encodes. Directions invisible to every probe vanish under this map, so the quotient by those directions is exactly the space of evaluations.

\subsection{Approximate version of Theorem~\ref{thm:common_abstract_space}}
\label{subsec:approx_shared_space}

Theorem~\ref{thm:common_abstract_space} is exact and basis-free: if two models realize exactly the same finite-dimensional concept family, then their probe-visible quotients are canonically the same abstract space. The transfer experiments, however, operate in a more concrete regime. One chooses a finite bank of probes, realizes the quotient through an SVD of the bank matrix, aligns paired activations from the two models on the same inputs, and asks how far the resulting quotient coordinates can drift when concept-family agreement is only approximate. The results below give a quantitative companion to Theorem~\ref{thm:common_abstract_space} in exactly that finite-bank setting.

Throughout this subsection, let \(\mu\) be a probability distribution on \(X\). For a vector-valued function \(g:X\to \R^m\), write
\[
\norm{g}_{L^2(\mu)}^2 := \E_{x\sim\mu}\norm{g(x)}_2^2.
\]
For each \(i\in\{1,2\}\), assume \(H_i=\R^{d_i}\), choose a finite probe bank
\[
\mathcal B_i=\{\ell_{i,1},\dots,\ell_{i,k_i}\}\subset V_i,
\]
and let \(W_i\in\R^{k_i\times d_i}\) be the corresponding bank matrix, so that the \(j\)-th row of \(W_i\) represents the linear functional \(\ell_{i,j}\). We assume
\[
r_i:=\rank(W_i)\ge 1.
\]
To avoid conflict with the probe spaces \(V_i\), write a thin SVD in the form
\[
W_i = U_i \Sigma_i R_i^\top,
\]
where \(U_i\in\R^{k_i\times r_i}\) and \(R_i\in\R^{d_i\times r_i}\) have orthonormal columns, and
\[
\Sigma_i = \diag(\sigma_{i,1},\dots,\sigma_{i,r_i}),
\qquad
\sigma_{i,1}\ge \cdots \ge \sigma_{i,r_i}>0.
\]
Define the bank-visible quotient coordinates by
\[
z_i(x):=R_i^\top h_i(x)\in\R^{r_i}.
\]
Equivalently,
\[
W_i h_i(x)=U_i\Sigma_i z_i(x).
\]
Thus \(z_i(x)\) is the concrete SVD realization of the quotient \(Z(\Span\,\mathcal B_i)\): it keeps exactly the directions seen by the bank and discards the bank-invisible directions. For simplicity we state the exact thin-SVD version. The same formulas apply verbatim to the truncated SVD obtained after discarding numerically tiny singular directions, which is the version used in practice.

Now let \(M:\R^{k_2}\to\R^{k_1}\) be a linear map between the two bank-score spaces. We define the associated \emph{bank-visible concept-family mismatch} by
\[
\delta(M)^2
:=
\E_{x\sim\mu}
\norm{
U_1^\top\bigl(W_1h_1(x)-MW_2h_2(x)\bigr)
}_2^2,
\]
and then set
\[
\delta_{\mathrm{cf}}:=\inf_M \delta(M).
\]
The projection by \(U_1^\top\) is deliberate: it measures mismatch only in the source-visible score directions. Any component orthogonal to \(\Image(W_1)\) is invisible to every source-bank probe and should not contribute to the source-side transfer error. Finally, write
\[
\sigma_{\min}^+(W_i):=\sigma_{i,r_i},
\qquad
\sigma_{\max}^+(W_i):=\sigma_{i,1},
\qquad
\kappa^+(W_i):=\frac{\sigma_{\max}^+(W_i)}{\sigma_{\min}^+(W_i)}.
\]

\begin{theorem}[Approximate shared space in finite-bank coordinates]
\label{thm:approx_shared_space}
With notation as above, fix any linear map \(M:\R^{k_2}\to\R^{k_1}\) and define
\[
T_M := \Sigma_1^{-1}U_1^\top M U_2\Sigma_2
\;:\;
\R^{r_2}\to\R^{r_1}.
\]
Then for every \(x\in X\),
\[
z_1(x)-T_M z_2(x)
=
\Sigma_1^{-1}U_1^\top\bigl(W_1h_1(x)-MW_2h_2(x)\bigr).
\]
Consequently,
\[
\norm{z_1-T_Mz_2}_{L^2(\mu)}
\le
\frac{\delta(M)}{\sigma_{\min}^+(W_1)}.
\]
In particular,
\[
\inf_{T\in\Hom(\R^{r_2},\R^{r_1})}
\norm{z_1-Tz_2}_{L^2(\mu)}
\le
\frac{\delta_{\mathrm{cf}}}{\sigma_{\min}^+(W_1)}.
\]
\end{theorem}

\begin{proof}
From the thin SVD \(W_i=U_i\Sigma_i R_i^\top\), we get
\[
U_i^\top W_i
=
U_i^\top U_i \Sigma_i R_i^\top
=
\Sigma_i R_i^\top,
\]
since the columns of \(U_i\) are orthonormal. Multiplying by \(h_i(x)\) gives
\[
U_i^\top W_i h_i(x)
=
\Sigma_i R_i^\top h_i(x)
=
\Sigma_i z_i(x).
\]
Because \(\Sigma_i\) is invertible on the visible subspace, we can solve for \(z_i(x)\):
\[
z_i(x)
=
\Sigma_i^{-1}U_i^\top W_i h_i(x).
\]
Thus, the quotient coordinates are obtained by reading off the bank scores and then inverting the bank only on the directions that the bank actually sees.

By definition of \(T_M\),
\begin{align*}
T_M z_2(x)
&=
\Sigma_1^{-1}U_1^\top M U_2\Sigma_2 z_2(x) \\
&=
\Sigma_1^{-1}U_1^\top M\bigl(U_2\Sigma_2 z_2(x)\bigr).
\end{align*}
But \(U_2\Sigma_2 z_2(x)=W_2h_2(x)\), so
\[
T_M z_2(x)
=
\Sigma_1^{-1}U_1^\top M W_2h_2(x).
\]
Applying the previous steps to model \(1\), we also have
\[
z_1(x)
=
\Sigma_1^{-1}U_1^\top W_1h_1(x).
\]
Thus, we obtain
\begin{align*}
z_1(x)-T_Mz_2(x)
&=
\Sigma_1^{-1}U_1^\top W_1h_1(x)
-
\Sigma_1^{-1}U_1^\top M W_2h_2(x) \\
&=
\Sigma_1^{-1}U_1^\top\bigl(W_1h_1(x)-MW_2h_2(x)\bigr).
\end{align*}
This is the stated pointwise identity. From the identity above,
\[
\norm{z_1(x)-T_Mz_2(x)}_2
\le
\norm{\Sigma_1^{-1}}_{\mathrm{op}}
\,
\norm{U_1^\top\bigl(W_1h_1(x)-MW_2h_2(x)\bigr)}_2.
\]
Since \(\Sigma_1^{-1}\) is diagonal with entries \(1/\sigma_{1,1},\dots,1/\sigma_{1,r_1}\), its operator norm is
\[
\norm{\Sigma_1^{-1}}_{\mathrm{op}}
=
\frac{1}{\sigma_{\min}^+(W_1)}.
\]
Therefore
\[
\norm{z_1(x)-T_Mz_2(x)}_2
\le
\frac{1}{\sigma_{\min}^+(W_1)}
\norm{U_1^\top\bigl(W_1h_1(x)-MW_2h_2(x)\bigr)}_2.
\]
Square both sides and average with respect to \(\mu\):
\[
\norm{z_1-T_Mz_2}_{L^2(\mu)}^2
\le
\frac{1}{(\sigma_{\min}^+(W_1))^2}
\E_{x\sim\mu}
\norm{U_1^\top\bigl(W_1h_1(x)-MW_2h_2(x)\bigr)}_2^2.
\]
By the definition of \(\delta(M)\), this is
\[
\norm{z_1-T_Mz_2}_{L^2(\mu)}^2
\le
\frac{\delta(M)^2}{(\sigma_{\min}^+(W_1))^2}.
\]
Taking square roots gives
\[
\norm{z_1-T_Mz_2}_{L^2(\mu)}
\le
\frac{\delta(M)}{\sigma_{\min}^+(W_1)}.
\]

Now, we optimize over transports. For every linear map \(M\), the construction above produces a corresponding linear transport \(T_M\). Hence
\[
\inf_{T\in\Hom(\R^{r_2},\R^{r_1})}
\norm{z_1-Tz_2}_{L^2(\mu)}
\le
\inf_M \norm{z_1-T_Mz_2}_{L^2(\mu)}
\le
\frac{\inf_M \delta(M)}{\sigma_{\min}^+(W_1)}.
\]
By definition, \(\inf_M\delta(M)=\delta_{\mathrm{cf}}\). This proves the theorem.
\end{proof}

\paragraph{Proof sketch.}
The theorem separates the approximation problem into two parts. The numerator \(\delta_{\mathrm{cf}}\) measures how well the two models agree on the \emph{bank-visible concept family} after the best linear re-basing of bank scores. The denominator \(\sigma_{\min}^+(W_1)\) measures how stably the source bank coordinatizes those concepts. Small concept-family mismatch and good bank conditioning imply small quotient transfer error.

\paragraph{Exact case.}
If the two finite banks are two coordinate realizations of the same bank-visible concept family, then there exists a linear map \(M\) for which \(\delta(M)=0\). Theorem~\ref{thm:approx_shared_space} then yields \(z_1=T_Mz_2\) in \(L^2(\mu)\), recovering the concrete finite-bank form of Theorem~\ref{thm:common_abstract_space}.

\begin{corollary}[Transferred probes inherit the same quantitative error]
\label{cor:approx_probe_transfer}
Under the hypotheses of Theorem~\ref{thm:approx_shared_space}, let \(a\in\R^{r_1}\) and define
\[
f_a(x):=a^\top z_1(x),
\qquad
\widetilde f_{a,M}(x):=a^\top T_Mz_2(x)=(T_M^\top a)^\top z_2(x).
\]
Then
\[
\norm{f_a-\widetilde f_{a,M}}_{L^2(\mu)}
\le
\norm{a}_2\,\frac{\delta(M)}{\sigma_{\min}^+(W_1)}.
\]
In particular, if \(M\) is chosen so that \(\delta(M)\) is close to \(\delta_{\mathrm{cf}}\), then every source-bank probe transfers with error controlled by the same concept-family mismatch and the same source-bank conditioning factor.
\end{corollary}

\begin{proof}
The difference between the source probe and the transferred probe is
\[
f_a(x)-\widetilde f_{a,M}(x)
=
a^\top\bigl(z_1(x)-T_Mz_2(x)\bigr).
\]
By the Cauchy--Schwarz inequality,
\[
|f_a(x)-\widetilde f_{a,M}(x)|
\le
\norm{a}_2\,\norm{z_1(x)-T_Mz_2(x)}_2.
\]
Square both sides and average over \(x\sim\mu\):
\[
\norm{f_a-\widetilde f_{a,M}}_{L^2(\mu)}^2
\le
\norm{a}_2^2
\norm{z_1-T_Mz_2}_{L^2(\mu)}^2.
\]
Now apply Theorem~\ref{thm:approx_shared_space}:
\[
\norm{f_a-\widetilde f_{a,M}}_{L^2(\mu)}
\le
\norm{a}_2\,\frac{\delta(M)}{\sigma_{\min}^+(W_1)}.
\]
This is the desired bound.
\end{proof}

\paragraph{Proof sketch.}
Corollary~\ref{cor:approx_probe_transfer} is the approximate analogue of Corollary~\ref{cor:transporting_probes}. Once the quotient coordinates are close, every linear functional on the source quotient is close to its transported version on the target quotient, with an extra factor of \(\norm{a}_2\) coming from the size of the probe itself.

\begin{proposition}[Transport size and bank conditioning]
\label{prop:approx_transport_conditioning}
With notation as in Theorem~\ref{thm:approx_shared_space},
\[
\norm{T_M}_{\mathrm{op}}
\le
\frac{\sigma_{\max}^+(W_2)}{\sigma_{\min}^+(W_1)}
\norm{M}_{\mathrm{op}}.
\]
If \(r_1=r_2\) and \(U_1^\top M U_2\) is invertible, then
\[
\kappa(T_M)
\le
\kappa(U_1^\top M U_2)\,\kappa^+(W_1)\,\kappa^+(W_2).
\]
\end{proposition}

\begin{proof}
For the operator-norm bound, use the definition of \(T_M\) and submultiplicativity:
\begin{align*}
\norm{T_M}_{\mathrm{op}}
&=
\norm{\Sigma_1^{-1}U_1^\top M U_2\Sigma_2}_{\mathrm{op}} \\
&\le
\norm{\Sigma_1^{-1}}_{\mathrm{op}}
\,
\norm{U_1^\top M U_2}_{\mathrm{op}}
\,
\norm{\Sigma_2}_{\mathrm{op}}.
\end{align*}
Because \(U_1\) and \(U_2\) have orthonormal columns,
\[
\norm{U_1^\top M U_2}_{\mathrm{op}}
\le
\norm{M}_{\mathrm{op}}.
\]
Also,
\[
\norm{\Sigma_1^{-1}}_{\mathrm{op}}
=
\frac{1}{\sigma_{\min}^+(W_1)},
\qquad
\norm{\Sigma_2}_{\mathrm{op}}
=
\sigma_{\max}^+(W_2).
\]
Combining these three bounds gives
\[
\norm{T_M}_{\mathrm{op}}
\le
\frac{\sigma_{\max}^+(W_2)}{\sigma_{\min}^+(W_1)}
\norm{M}_{\mathrm{op}}.
\]

Now assume \(r_1=r_2\) and \(U_1^\top M U_2\) is invertible. Then \(T_M\) is invertible, and its inverse is
\[
T_M^{-1}
=
\Sigma_2^{-1}(U_1^\top M U_2)^{-1}\Sigma_1.
\]
Therefore
\begin{align*}
\norm{T_M^{-1}}_{\mathrm{op}}
&\le
\norm{\Sigma_2^{-1}}_{\mathrm{op}}
\,
\norm{(U_1^\top M U_2)^{-1}}_{\mathrm{op}}
\,
\norm{\Sigma_1}_{\mathrm{op}} \\
&=
\frac{\sigma_{\max}^+(W_1)}{\sigma_{\min}^+(W_2)}
\norm{(U_1^\top M U_2)^{-1}}_{\mathrm{op}}.
\end{align*}
Multiply the exact factorization bound for \(\norm{T_M}_{\mathrm{op}}\) by this bound for \(\norm{T_M^{-1}}_{\mathrm{op}}\):
\begin{align*}
\kappa(T_M)
&=
\norm{T_M}_{\mathrm{op}}\,
\norm{T_M^{-1}}_{\mathrm{op}} \\
&\le
\Bigl(
\frac{\sigma_{\max}^+(W_2)}{\sigma_{\min}^+(W_1)}
\norm{U_1^\top M U_2}_{\mathrm{op}}
\Bigr)
\Bigl(
\frac{\sigma_{\max}^+(W_1)}{\sigma_{\min}^+(W_2)}
\norm{(U_1^\top M U_2)^{-1}}_{\mathrm{op}}
\Bigr) \\
&=
\kappa(U_1^\top M U_2)\,\kappa^+(W_1)\,\kappa^+(W_2).
\end{align*}
This proves the proposition.
\end{proof}

\paragraph{Proof sketch.}
Theorem~\ref{thm:approx_shared_space} is directional, so only the source smallest singular value appears explicitly in the coordinate error bound. Proposition~\ref{prop:approx_transport_conditioning} explains where the conditioning of \emph{both} banks enters: it controls the size and numerical stability of the induced transport map itself. Near-duplicate probe directions shrink \(\sigma_{\min}^+(W_i)\), inflate \(\kappa^+(W_i)\), and amplify otherwise small score-level mismatches.

\begin{corollary}[A simple margin bound for classification transfer]
\label{cor:approx_margin_transfer}
Under the hypotheses of Corollary~\ref{cor:approx_probe_transfer}, consider the source classifier \(x\mapsto \operatorname{sign}(f_a(x))\). Then for every \(\gamma>0\),
\[
\Pr_{x\sim\mu}\!\bigl[\operatorname{sign}(f_a(x))\neq \operatorname{sign}(\widetilde f_{a,M}(x))\bigr]
\le
\Pr_{x\sim\mu}\!\bigl[|f_a(x)|\le \gamma\bigr]
+
\frac{\norm{a}_2^2\,\delta(M)^2}{\gamma^2(\sigma_{\min}^+(W_1))^2}.
\]
\end{corollary}

\begin{proof}
A sign disagreement can happen in only two ways. Either the source score \(f_a(x)\) already lies within \(\gamma\) of zero, or the transferred score differs from the source score by at least \(\gamma\). Formally,
\[
\bigl\{\operatorname{sign}(f_a)\neq \operatorname{sign}(\widetilde f_{a,M})\bigr\}
\subset
\bigl\{|f_a|\le \gamma\bigr\}
\cup
\bigl\{|f_a-\widetilde f_{a,M}|\ge \gamma\bigr\}.
\]
Indeed, if the signs disagree and \(|f_a(x)|>\gamma\), then moving from \(f_a(x)\) to \(\widetilde f_{a,M}(x)\) must cross zero, so
\[
|f_a(x)-\widetilde f_{a,M}(x)|
\ge
|f_a(x)|
>
\gamma.
\]

Now take probabilities and apply the union bound:
\[
\Pr\!\bigl[\operatorname{sign}(f_a)\neq \operatorname{sign}(\widetilde f_{a,M})\bigr]
\le
\Pr\!\bigl[|f_a|\le \gamma\bigr]
+
\Pr\!\bigl[|f_a-\widetilde f_{a,M}|\ge \gamma\bigr].
\]
For the second term, Markov's inequality applied to the nonnegative random variable \(|f_a-\widetilde f_{a,M}|^2\) gives
\[
\Pr\!\bigl[|f_a-\widetilde f_{a,M}|\ge \gamma\bigr]
\le
\frac{\norm{f_a-\widetilde f_{a,M}}_{L^2(\mu)}^2}{\gamma^2}.
\]
Finally, use Corollary~\ref{cor:approx_probe_transfer}:
\[
\norm{f_a-\widetilde f_{a,M}}_{L^2(\mu)}^2
\le
\frac{\norm{a}_2^2\,\delta(M)^2}{(\sigma_{\min}^+(W_1))^2}.
\]
Substituting this bound yields the claimed inequality.
\end{proof}

\paragraph{Proof sketch.}
The corollary isolates the two familiar ingredients of successful classification transfer: the source classifier should have nontrivial margin away from zero, and the transferred score should stay close enough that it does not often flip the sign. Theorem~\ref{thm:approx_shared_space} controls the second term.

The optimization over \(M\) removes arbitrary choices of basis and scaling in score space. If the two banks differ only by a re-basing of the same bank-visible concepts, then some \(M\) makes \(\delta(M)=0\). If no such \(M\) exists, then \(\delta_{\mathrm{cf}}\) measures the residual disagreement between the two realized concept families on the paired inputs. In that sense, \(\delta_{\mathrm{cf}}\) is the finite-bank quantity that plays the role of approximate equality of concept families.

\paragraph{Connection to in-span fraction.}
We next connect this analysis to the empirical in-span-fraction results. Suppose a held-out source concept direction decomposes as \(u=u_{\parallel}+u_{\perp}\), where \(u_{\parallel}\) lies in the span of the bank and \(u_{\perp}\) is orthogonal to it. If
\[
\rho:=\frac{\norm{u_{\parallel}}_2^2}{\norm{u}_2^2}
\]
is the in-span fraction, then the uncovered component has size
\[
\norm{u_{\perp}}_2=\sqrt{1-\rho}\,\norm{u}_2.
\]
In this simple geometric picture, the irreducible part of \(\delta_{\mathrm{cf}}\) scales like \(\sqrt{1-\rho}\), so Theorem~\ref{thm:approx_shared_space} suggests transfer error on the order of
\[
\frac{\sqrt{1-\rho}}{\sigma_{\min}^+(W_1)}.
\]
In the synthetic \(\theta\)-sweep, \(\rho=\cos^2\theta\), so the same heuristic becomes an error scale of order
\[
\frac{\sin\theta}{\sigma_{\min}^+(W_1)}.
\]
This heuristic is not formalized as a theorem, since the exact constant depends on the data distribution and on how the held-out concept couples to the chosen bank.

\paragraph{Empirical version.}
All expectations above can be replaced by empirical averages over a finite paired alignment set \(\{x_1,\dots,x_N\}\); the proofs are identical. In practice, one can either fit \(M\) in score space and then form \(T_M\), or fit a transport \(T\) directly in quotient coordinates by least squares. Theorem~\ref{thm:approx_shared_space} shows that any good score-space matching induces a good quotient transport, while the empirical Ridge fits used in Section~\ref{subsec:transfer_experiments} can be viewed as estimating this transport directly.

\paragraph{Relation to the experiments.}
These bounds give a compact explanation of several empirical patterns in Section~\ref{subsec:transfer_experiments}. First, in the synthetic quotient-transfer experiment (Table~\ref{tab:exp_h}), nuisance dimensions do not appear explicitly in the bound unless they increase the bank-visible mismatch \(\delta_{\mathrm{cf}}\), which explains why quotient transfer remains stable while PCA collapses; full-state OLS, by contrast, is not quotient-restricted and therefore has no reason to preserve selectivity. Second, Proposition~\ref{prop:approx_transport_conditioning} explains the bank-conditioning effect behind Table~\ref{tab:track_k} and the controlled redundancy ablation in Appendix~\ref{subsec:conditioning_ablation}: redundant probe banks can shrink \(\sigma_{\min}^+(W_i)\), inflate \(\kappa^+(W_i)\), and thereby amplify alignment error--though in our experiments the implementation's SVD thresholding and Ridge regularization absorb much of the numerical conditioning effect, making span loss the dominant empirical failure mode. Third, the smooth degradation seen in the synthetic \(\theta\)-sweep (Table~\ref{tab:theta_sweep}) and the leave-one-out real-data study (Table~\ref{tab:isf_scatter}) is exactly the behavior one expects when the concept-family mismatch grows continuously while the bank itself is held fixed.

\paragraph{Connection to Theorem~\ref{thm:robust_alignment}.}
The algebraic pattern here is the same as in Theorem~\ref{thm:robust_alignment}. There, hidden-state error is bounded by output mismatch divided by a smallest singular value of the readout map. Here, quotient-coordinate error is bounded by bank-score mismatch divided by the smallest nonzero singular value of the source bank. The difference is conceptual rather than algebraic: the present theorem lives on the probe-visible quotient, which is the object the transfer experiments are designed to preserve.

\subsection{Proof of Corollary~\ref{cor:transporting_probes}}

\corTransporting*

\begin{proof}
By Theorem~\ref{thm:common_abstract_space}, for each \(i\in\{1,2\}\) there is a linear isomorphism
\[
 \overline Q_i:H_i/K(V_i)\to\mathcal C^*.
\]
Take any concept \(c\in\mathcal C\). Since \(\mathcal C\) is finite-dimensional, the natural bidual identification gives
\[
 c\in\mathcal C\cong (\mathcal C^*)^*.
\]
Thus \(c\) may be viewed as a linear functional on \(\mathcal C^*\). Compose this functional with \(\overline Q_i\) and with the quotient projection \(\pi_i:H_i\to H_i/K(V_i)\). We obtain the linear functional
\[
 h\mapsto c\bigl(\overline Q_i(\pi_i(h))\bigr)
\]
on \(H_i\). Now evaluate this linear functional on a represented point \(h_i(x)\):
\[
 c\bigl(\overline Q_i(\pi_i(h_i(x)))\bigr)
 =
 c(Q_i(h_i(x)))
 =
 c(\operatorname{ev}_x)
 =
 \operatorname{ev}_x(c)
 =
 c(x).
\]
So the pulled-back functional realizes exactly the concept \(c\) in model \(i\). Because the same functional \(c\in(\mathcal C^*)^*\) is used for both models, this gives a canonical way to transport probes through the shared abstract space.
\end{proof}

\paragraph{Proof sketch.}
A concept is a linear functional on the shared evaluation space. Once that space has been identified, both models realize the same concept simply by pulling the functional back along their own quotient maps.

\section{Experimental Roadmap}
\label{app:experiment_narrative}

Appendix~\ref{app:experiment_details} describes the models and procedures used in our experiments. It covers activation extraction and probe training, followed by quotient construction and alignment. Table~\ref{tab:notation_summary} links each experiment to the relevant theoretical claims. The remaining experimental appendices follow the paper's argument. They first validate the polynomial hierarchy and identify stable, structured degree-2 probe families. They then test quotient-space transfer and examine whether the transferred monitors are suitable for deployment.

\subsection{From affine symmetry to the polynomial hierarchy}

The degree-hierarchy experiments are the most direct tests of Theorem~\ref{thm:scalar_affine_classification}. Appendix~\ref{app:degree_validation_appendix} begins with targets whose required polynomial degree is known by construction. XOR and circular parity show that lower-degree probes cannot represent these targets even with additional data. These targets lie outside the function class available to lower-degree probes. The Pythia Boolean-composition experiment then repeats the same idea inside real language-model representations by training primitive probes for simple truth variables and fitting polynomial heads over their scores. The important point is not just that higher-degree heads can fit harder targets, but that the recovered minimum degree is stable across layers and model sizes when the primitive features are available. This is the empirical counterpart of the theorem's claim that the hierarchy is a property of the target composition and the symmetry class, not a particular coordinate basis.

Appendices~\ref{subsubsec:exact_reparam}, \ref{subsec:basis_stability_appendix}, and~\ref{subsubsec:exp_c} isolate the coordinate-stability side of the same claim. The exact reparameterization test isolates the function class from optimizer behavior. A full quadratic can be transported analytically across affine-equivalent coordinates. A sparse monomial probe cannot because its sparsity pattern depends on the coordinate basis. The Pythia basis-stability experiment then asks what survives once optimization, regularization, and real hidden states are reintroduced. CP rank-1 retains the degree-2 interaction signal while showing substantially lower basis variance than less structured alternatives. The softmax-symmetry check verifies the readout-level symmetry numerically, connecting the empirical setup back to Proposition~\ref{prop:softmax_affine_symmetry} and Corollary~\ref{cor:affine_head_equiv}.

Appendices~\ref{subsubsec:exp_a}, \ref{subsec:quadratic_scaling}, and~\ref{subsubsec:cp_rank} explain why the paper emphasizes structured degree-2 probes rather than arbitrary quadratics. Area regression is the regression analogue of XOR. Affine probes underfit its degree-2 target, while degree-2 probes recover the product. The scaling experiment shows that full quadratics become expensive or infeasible even at hidden dimensions that are small by modern language-model standards. The CP rank sweep supplies the remedy. When the interaction has low intrinsic tensor rank, CP probes recover most of the quadratic signal with far fewer parameters, and Proposition~\ref{prop:low_rank_invariance} explains why this compression is compatible with coordinate changes.

\subsection{From primitive probes to score-space composition}

The score-space experiments test a practical consequence of the hierarchy. If primitive linear probes produce scores \(s_1(h),\ldots,s_k(h)\), then a degree-\(\ell\) polynomial head on those scores is still a degree-\(\ell\) polynomial probe on the original hidden state. One can therefore build nonlinear monitors by first learning simple probe-visible features and then fitting a structured polynomial head on the low-dimensional score space.

Appendices~\ref{subsubsec:gating}, \ref{subsubsec:composition}, \ref{subsec:higher_degree_appendix}, and~\ref{subsec:cross_tok_appendix} study this idea on linguistic interactions. Subject--verb agreement is a degree-2 relation between features at two tokens, and the score-space CP head captures it more effectively than either a linear score head or a bilinear probe over raw concatenated hidden states. The higher-degree UD experiments show that degree~2 often captures the useful interaction signal. Degree~3 adds little on the tested conjunctions. Thus, the hierarchy is not a license to increase degree arbitrarily; it is a way to choose the lowest structured degree that matches the target interaction.

Appendix~\ref{app:degree2_composition} gives the safety analogue. Policy mismatch is naturally an XOR between prompt-side risk and response-side refusal. Harmful compliance and benign over-refusal are both failures that occupy opposite corners of the primitive score space. A linear head cannot express this interaction cleanly, while a quadratic or CP rank-1 head can. This experiment links the linguistic composition results to monitoring. Degree-2 score-space probes are useful whenever the deployment concept involves an interaction between simpler probe-visible quantities.

\subsection{From shared quotients to monitor transfer}

The quotient-transfer experiments test Theorem~\ref{thm:common_abstract_space} and its finite-bank approximation in Appendix~\ref{subsec:approx_shared_space}. The theorem says that a probe bank does not need the full hidden state to align across models. It only needs the part of the hidden state visible to the bank, namely \(Z(V)=H/K(V)\). This changes the transfer problem from ``match all activations'' to ``match the probe-visible quotient.''

The synthetic transfer experiment in Section~\ref{subsec:transfer_experiments} is the most controlled test of this claim. Source and target activations share a low-dimensional concept but contain growing nuisance subspaces. Quotient transfer remains stable because the nuisance directions are invisible to the source bank. Its selectivity comes from projecting away concepts outside the bank span. These concepts therefore transfer at chance. Full-state OLS can transfer those out-of-span concepts, but this is precisely the problem--it gives no label-free signal that the concept was unsupported by the bank.

Appendix~\ref{subsubsec:track_s} and Section~\ref{subsec:transfer_experiments} move from synthetic data to real model families. Sentiment transfer across Pythia sizes provides a favorable setting. An approximately linear concept shared across related models transfers with almost complete AUROC recovery. Safety transfer from Qwen-7B to Qwen, Qwen-Coder, and Mistral targets is more demanding because it tests safety-relevant concepts across different model families without target labels. The main result is partial but meaningful portability. Some concepts, such as sentiment and toxicity-like monitors, transfer strongly; others, such as some jailbreaking directions in cross-architecture transfer, are weaker. These results support coverage-aware quotient alignment, while individual monitors still require validation before deployment.

Appendix~\ref{subsec:safety_transfer_additional} clarifies the conditions behind successful transfer. Alignment text must activate the relevant quotient directions. SST-2 alignment transfers sentiment but not safety monitors. Safety-domain text transfers safety concepts. Alignment budget also matters, because the map between quotient coordinates has to be estimated from paired unlabeled activations. In practice, target labels are not needed for quotient alignment, but unlabeled paired inputs should come from the deployment domain.

\subsection{Coverage, conditioning, and failure modes}

The conditioning and coverage experiments explain how to decide whether a transferred monitor should be trusted. Appendix~\ref{subsec:conditioning_ablation} shows that probe-bank composition matters more than raw bank size. Redundant probes do not help if they merely duplicate directions already in the bank. Replacing independent directions with redundant probes shrinks the effective visible span and degrades transfer. These findings correspond to the finite-bank bounds in Theorem~\ref{thm:approx_shared_space} and Proposition~\ref{prop:approx_transport_conditioning}. The bank should cover diverse concept directions and avoid unnecessary ill-conditioning.

Appendix~\ref{subsec:continuum} turns coverage into a continuous diagnostic. In the synthetic \(\theta\)-sweep, a held-out concept rotates from fully in-span to fully out-of-span, and quotient transfer degrades smoothly with its in-span fraction. Full-state OLS again transfers everything and therefore gives no coverage signal. The real-data ISF analyses in Appendix~\ref{subsubsec:isf_pearson_primary_bank} provide an imperfect but useful counterpart. ISF is a label-free diagnostic computed on the source model. It correlates with target transfer quality across the primary safety bank. It is not a semantic oracle -- correlated concepts can sometimes transfer despite low geometric overlap -- but it is a principled warning sign that the quotient route lacks direct coverage.

Appendices~\ref{app:coverage_loo}, \ref{app:coverage_abstention}, and~\ref{app:theorem_prediction} convert this warning sign into deployment-style tests. Leave-one-out removal shows that full-state transfer can silently preserve a removed concept, while quotient transfer often drops when the removed concept is outside the remaining bank span. Coverage-aware abstention deploys a transferred monitor only when its ISF exceeds a threshold. This rule costs a few AUROC points on covered concepts, but it eliminates a class of silent failures where full-state transfer appears successful without a coverage certificate. Finally, the coverage-deficit analysis tests the theorem's prediction that transfer drop should correlate with \(1-\mathrm{ISF}\) for the quotient route. This relationship need not hold for full-state OLS. The empirical correlations match that pattern.

\subsection{Lexical baselines and behavioral deployment}

Appendices~\ref{subsubsec:lexical} and~\ref{app:lexical_robustness} separate representation-level transfer from surface-text shortcuts. TF-IDF baselines are strong on some standard safety benchmarks, especially when distinctive keywords appear in both train and test sets. The robustness tests show substantial TF-IDF degradation under keyword holdout, lexically matched evaluation, or low-label budgets. These experiments do not claim that neural probes are always better than lexical baselines. Rather, they identify which concepts are mostly lexical in the benchmark and which require more compositional or representation-level information.

Appendix~\ref{app:reranking} asks whether transferred monitors can affect model behavior, not just classify cached activations. The reranking experiment uses the transferred safety bank to choose among multiple sampled completions from a target model. With zero target labels, the monitor reduces harmful compliance on both within-family and cross-architecture targets at low benign false-refusal rates. Under the full bank, quotient and full-state rerankers often agree on rankings because the deployed concepts are in-span; the significance is that a source-trained monitor can change a target model's behavior without target supervision.

Appendix~\ref{app:refusal_audit} documents the refusal/compliance judging protocol used by the behavioral experiments. The behavioral conclusions in Appendices~\ref{app:reranking} and~\ref{app:degree2_composition} depend on consistently classifying completions as comply, refuse, or ambiguous. The judge protocol is therefore part of the experimental claim, not a peripheral detail.

\subsection{Practical deployment implications}

Taken together, the experiments suggest the following workflow. First, identify primitive concepts with simple probes and use the degree hierarchy to decide whether the target is likely linear or interactional. Second, if a degree-2 interaction is needed, prefer structured coordinate-stable families such as CP over basis-dependent sparse monomials or full quadratics at high hidden dimension. Third, for cross-model portability, align the probe-visible quotient rather than the full hidden state. Fourth, before deploying a transferred monitor, inspect the bank span, conditioning, ISF, and alignment-domain coverage. Monitor transfer can work with no target labels, but should be deployed with coverage diagnostics rather than as an unconditional full-state stitching procedure.

\section{Implementation, Reproducibility, and Experimental Setup}
\label{app:experiment_details}

This appendix supplies the operational details behind Section~\ref{sec:experiments}. It moves from shared infrastructure and notation to experiment-specific configurations, and then to the layer-pair diagnostic that checks that the transfer results are not tied to a single hand-picked layer choice.

\subsection{Infrastructure and methodology}
\label{subsec:implementation}

All experiments use the same infrastructure, documented here for reproducibility.

\paragraph{Models.}
Pythia-70m, 160m, and 410m~\citep{biderman2023pythia} with hidden dimensions $d\in\{512,768,1024\}$.
Qwen-2.5-3B, 7B-Instruct, 14B, and Coder-7B with hidden dimensions $d\in\{2048,3584,5120,3584\}$.
Mistral-7B-Instruct-v0.3 with $d=4096$.
All models are loaded in bfloat16 from HuggingFace with no quantization.

\paragraph{Tokenization and alignment.}
Sentence-level tokenization with \texttt{return\_offsets\_mapping} for correct subword-to-word alignment.
For UD tasks, each token's hidden state is taken as the last subword piece.
For sentence-level tasks (SST-2, safety), we use the last token's hidden state.

\paragraph{Probes.}
Classification probes use scikit-learn's \texttt{LogisticRegression(C=1.0, max\_iter=5000)}.
Regression probes use Ridge regression with $\alpha=10^{-4}$.
We train CP probes using PyTorch L-BFGS with strong-Wolfe line search. We use 30--50 random restarts depending on the experiment and select the best run by validation AUROC. The CP training objective is binary cross-entropy for classification (\texttt{BCEWithLogitsLoss}) and squared loss for regression. For the synthetic regression experiment in Appendix~\ref{subsubsec:exp_a}, CP probes are fit with Adam (lr $10^{-3}$, 500 epochs, patience 30) with 2--3 random restarts, matching the configuration used for the synthetic regression experiment.

\paragraph{Quotient construction.}
Given the probe weight matrix $W\in\R^{k\times d}$ (with $k$ probes), compute the SVD $W=U\Sigma V^\top$ and retain singular vectors with $\sigma_i > 10^{-3}\cdot\sigma_1$.
The quotient projection is $\pi_V = V_r V_r^\top$, where $V_r$ contains the retained right singular vectors.
The quotient typically has 5--29 dimensions depending on the number and diversity of probes.

\paragraph{Alignment.}
We fit multi-output Ridge regression ($\alpha=10^{-4}$) to unlabeled paired activations drawn exclusively from \textbf{training splits}. Evaluation uses held-out test splits that are completely disjoint from the alignment pool.
For Pythia, both models process the same input sentences and we extract hidden states from matching layers. We use the middle layer by default.
For cross-family transfer, both models process the same prompts.
Pythia sentiment transfer uses 2{,}000 alignment samples. Cross-family safety transfer concatenates all 16 concept training sets into 76K rows containing approximately 21K unique texts (Section~\ref{subsec:track_k_additional}). The train/test splits are: ToxicChat~\citep{lin2023toxicchat} and BeaverTails~\citep{ji2024beavertails} use their native HuggingFace splits; SST-2~\citep{socher2013sst} uses train/validation; OpenAI moderation and harmful-combined use random 80/20 splits with fixed seed. UD English-EWT~\citep{silveira2014gold} provides the token-level linguistic annotations for the Pythia experiments.

\paragraph{Seeds and statistical reporting.}
All experiments use seeds $\{42,137,256,314,999\}$.
We report mean $\pm$ standard deviation where variation is informative, and mean only where standard deviations are negligible ($<0.005$).

\paragraph{Models and data summary.} Synthetic experiments use controlled latent structure in \(\R^d\). The real-model experiments use Pythia~\citep{biderman2023pythia} at 70M, 160M, and 410M parameters on Universal Dependencies English-EWT, and Qwen-2.5/Mistral variants for cross-model safety transfer. Unless otherwise stated, classification probes are logistic regression (\(C=1.0\)) with intercept, regression probes use Ridge with \(\alpha=10^{-4}\), and results are averaged over 5 seeds \(\{42,137,256,314,999\}\). AUROC is the primary classification metric; balanced accuracy is reported where it is more informative. 

\paragraph{Computational resources.}
All experiments were conducted on a shared HPC cluster equipped with NVIDIA A100 GPUs (80GB). All code, seeds, and Slurm job scripts are available in the supplementary material.

\paragraph{Notation summary.}
Table~\ref{tab:notation_summary} summarizes the correspondence between experiments and the theoretical results they test.

\begin{table}[h]
\centering
\small
\caption{Mapping from experiments to theoretical results.}
\label{tab:notation_summary}
\begin{tabular}{l l l}
\toprule
Experiment & Tests & Key prediction \\
\midrule
Synthetic XOR & Thm~\ref{thm:scalar_affine_classification} & Degree-1 $\not\ni$ degree-2 concepts \\
Circular parity & Thm~\ref{thm:scalar_affine_classification} & Hierarchy is tight at every degree \\
Area regression & Thm~\ref{thm:scalar_affine_classification}, Prop~\ref{prop:low_rank_invariance} & Degree-2 regression; CP vs.\ sparse \\
Pythia interactions & Thm~\ref{thm:scalar_affine_classification} & Degree hierarchy in real LMs \\
Sparse vs.\ CP & Prop~\ref{prop:low_rank_invariance} & Low-rank invariant, sparsity not \\
Softmax symmetry & Prop~\ref{prop:softmax_affine_symmetry} & Softmax invariance to machine $\eps$ \\
Synthetic quotient & Thm~\ref{thm:common_abstract_space} & Stable + selective transfer \\
SST-2 transfer & Thm~\ref{thm:common_abstract_space} & Within-family zero-label transfer \\
Safety transfer & Thm~\ref{thm:common_abstract_space} & Cross-architecture portability \\
\bottomrule
\end{tabular}
\end{table}

\subsection{Replication details}

This appendix provides full replication details for all experiments in Section~\ref{sec:experiments}.

\subsection{Models and activation extraction}

\paragraph{Pythia models.} We use EleutherAI/pythia-\{70m, 160m, 410m\} with hidden dimensions 512, 768, 1024 and 6, 12, 24 layers respectively. Activations are extracted at layers \{2,3,4\}, \{4,6,8\}, \{8,12,16\} (early, middle, late). For most experiments we use the middle layer.

\paragraph{Qwen/Mistral models.} For safety transfer we use Qwen2.5-\{3B, 7B, 14B, Coder-7B\}-Instruct and Mistral-7B-Instruct-v0.3. Their hidden dimensions are 2048, 3584, 5120, 3584, and 4096 respectively. We extract activations at 75\% of model depth.

\paragraph{Tokenization.} UD English-EWT text is tokenized as full sentences using \texttt{tokenizer(sentence, return\_offsets\_mapping=True)}. Character-level offsets map subtokens back to UD word boundaries. Each sentence is a separate forward pass.

\paragraph{Centering.} All hidden states are centered by subtracting the training-split mean. Validation and test splits use the training mean for centering.

\subsection{Probe training}

All classification probes use scikit-learn 1.7.2's \texttt{LogisticRegression(C=1.0, max\_iter=5000, solver="lbfgs")} on centered features. The area-regression experiment in Appendix~\ref{subsubsec:exp_a} uses Ridge regression with \(\alpha=10^{-4}\). We use random seeds \{42, 137, 256, 314, 999\}.

\subsection{Quotient construction}

The probe-visible quotient is computed via SVD of the probe weight matrix \(W\in\R^{r\times d}\):
\[
W = U\Sigma V^\top, \qquad Q = V_k^\top \text{ where } k = |\{i : \sigma_i > 10^{-3}\sigma_1\}|.
\]
The quotient coordinates are \(z = Qh\). For efficiency, we use multi-output Ridge regression to fit an alignment map \(T\) satisfying \(h_{\text{tgt}}T^\top \approx z_{\text{src}}\).

\subsection{Per-experiment configuration}
\label{subsec:per_experiment_config}

\paragraph{Synthetic experiments.} The synthetic experiments use \(d=64\) unless otherwise noted. We sample affine transforms with condition numbers bounded by 50. The training, validation, and test sets contain 5{,}000, 1{,}000, and 2{,}000 samples respectively. These experiments use bias-augmented probe forms (e.g., \([z;\,1]\)) with explicit intercept terms and do not apply centering, so that the full affine reparameterization including translations is tested directly.

\paragraph{Pythia interaction experiments.} The Pythia probe bank is constructed from 5 Universal Dependencies morphosyntactic annotation categories on English-EWT: part-of-speech tag, morphological number (Sing/Plur), morphological tense (Past/Pres/None), punctuation class, and coarse dependency relation. Each category is expanded into binary probes via one-vs-rest encoding: for a category with $C$ classes, we train $C$ independent binary logistic regression probes (one per class vs.\ all others). The resulting bank has 30 probes total (one per distinct class label across the 5 categories). The quotient basis $Q$ is obtained by SVD of the $30 \times d$ stacked weight matrix; one singular value falls below the relative SVD threshold ($10^{-3}\sigma_1$) and is discarded, leaving $k = 29$. For each target, balanced subsampling matches the class counts to the minority-class count. We draw independent subsamples using five seeds. For the PCA-64 full quadratic, we fit PCA only on the training split and apply it to the validation and test splits. Applying \texttt{PolynomialFeatures(degree=2)} to the resulting 64 components yields 2{,}144 features.

\paragraph{Synthetic quotient transfer.} A shared latent $c\sim\mathcal{N}(0,I_8)$ is embedded into source and target hidden spaces via iid Gaussian matrices $A_s\in\R^{64\times 8}$ and $A_t\in\R^{128\times 8}$, drawn per seed. Nuisance bases $B_s,B_t$ of rank $k_s=k_t\in\{0,8,24,56\}$ and isotropic observation noise $\sigma=0.01$ are added to each activation, and paired activations share $c$ across the two models but use independent nuisance and noise draws. Five primitive concept directions $w_i\in\R^8$ are drawn as random unit vectors, with binary labels $y_i=\mathbf{1}[w_i\cdot c>0]$. Two held-out in-span concepts are random unit-norm linear combinations of the $w_i$, and one held-out out-of-span concept is a unit vector orthogonal to their span. Training and validation use $5{,}000$ and $1{,}000$ samples. The source quotient basis is obtained by SVD of the stacked primitive-probe weight matrix, with typical dimension $k=5$, and alignment is multi-output Ridge ($\alpha=10^{-4}$) fit on paired training activations. The PCA and random-projection baselines project each model to eight dimensions before OLS alignment, while full-state OLS fits Ridge $h_t\to h_s$ directly; Procrustes is inapplicable because $d_s\neq d_t$.

\paragraph{Safety transfer.} The alignment data concatenate all 16 concept datasets into 76K rows containing approximately 21K unique texts. Duplicates arise because multiple concepts share the same underlying dataset. For Qwen-14B (\(d=5120\)), this gives \(n/d\approx 14.9\), which is adequate for Ridge conditioning. We test six alignment methods, including quotient Ridge and affine alignment using Ridge with an intercept. The remaining methods are Procrustes, CCA, 1-shot threshold calibration, and full-state OLS.

\paragraph{Random-init control.} Qwen-7B architecture initialized with random weights via \texttt{AutoModelForCausalLM.from\_config()} (no pretraining). Same tokenizer and extraction procedure as pretrained models.

\subsection{Layer-pair transfer heatmap}

To verify that transfer quality is not an artifact of one specific layer choice, we evaluate all source$\times$target layer pairs. The Qwen-7B source layers are $\{7, 14, 21, 28\}$. Each target model contributes four layers at 25\%, 50\%, 75\%, and 100\% depth. For each pair, we build the quotient from source probes and learn the alignment before transferring without target labels.

\begin{table}[h]
\centering
\small
\caption{Layer-pair transfer heatmap for Qwen-7B source probes to Qwen-2.5-14B and Mistral-7B targets. Each cell reports mean transfer AUROC across the five safety concepts of Table~\ref{tab:track_k} for the specified (source layer, target layer) pair.}
\label{tab:layer_pair_heatmap}
\begin{tabular}{l cccc}
\toprule
& \multicolumn{4}{c}{Qwen-14B target layer} \\
\cmidrule(lr){2-5}
Source layer & L12 & L24 & L36 & L48 \\
\midrule
L7 & 0.815 & 0.872 & 0.850 & 0.633 \\
L14 & 0.838 & \textbf{0.929} & 0.914 & 0.858 \\
L21 & 0.821 & 0.916 & 0.911 & 0.855 \\
L28 & 0.846 & 0.923 & 0.914 & 0.919 \\
\bottomrule
\end{tabular}
\quad
\begin{tabular}{l cccc}
\toprule
& \multicolumn{4}{c}{Mistral target layer} \\
\cmidrule(lr){2-5}
Source layer & L8 & L16 & L24 & L32 \\
\midrule
L7 & 0.629 & 0.593 & 0.626 & 0.606 \\
L14 & 0.857 & 0.855 & 0.876 & 0.851 \\
L21 & 0.836 & 0.848 & 0.867 & 0.839 \\
L28 & 0.905 & 0.918 & \textbf{0.925} & 0.923 \\
\bottomrule
\end{tabular}
\end{table}

This pattern extends beyond a single layer choice. A broad region of mid-to-late source layers paired with mid-to-late target layers achieves AUROC $>0.85$. The early source layer (L7) transfers poorly to late target layers, but mid and late source layers are consistently strong. For Mistral (cross-architecture), the last source layer (L28) is best, achieving 0.91--0.93 across all target layers.

\section{Extended Experimental Results}
\label{app:extended_experiments}

This appendix provides additional experiments organized around the paper's three theoretical claims: the degree hierarchy, affine invariance, and quotient-space transfer. Some results strengthen a main-text claim directly (for example the scaling and CP-rank studies), while others delimit its scope (for example lexical baselines, alignment-domain mismatch, and higher-degree composites).

\subsection{Degree hierarchy validation}
\label{app:degree_validation_appendix}
\label{subsubsec:exp_f}
\label{subsubsec:exp_b}

We test two predictions of Theorem~\ref{thm:scalar_affine_classification}: that degree-2 is necessary for XOR-like concepts, and that the hierarchy is tight at every degree.

\paragraph{XOR (degree-1 fails, degree-2 succeeds).}
We embed two binary inputs $a,b\in\{-1,+1\}$ into $\R^{64}$ via a random affine map and define $y=\mathbf{1}[a=b]$, a degree-2 function. Across 5 seeds, a linear probe scores $0.501 \pm 0.010$ (chance) while a full quadratic probe scores $1.000 \pm 0.000$, with zero degradation under affine reparameterization. Since $ab\in\Poly{2}\setminus\Poly{1}$ and $\Poly{1}$ is closed under $\Aff(d)$, no amount of data can help a degree-1 probe represent this concept.

\paragraph{Degree recovery across layers and tasks.}
The affine-invariance theory predicts that the minimum probe degree $d^\star(T)$ required for a task $T$ is intrinsic to the task, not to the network or the layer at which the representation is read. We test this prediction on a pretrained language model by recovering $d^\star$ at multiple layers of Pythia-160m and Pythia-410m across five Boolean tasks with known minimum polynomial-threshold decision degrees over the primitive scores. Primitive probes (for $A$, $B$, $C$ truth) are trained on the last-token hidden state of template sentences ``$A$ is \{true/false\}.\ $B$ is \{true/false\}.\ $C$ is \{true/false\}.''; each target task combines the primitive scores into a single binary label. We then fit a sequence of degree-$d$ composition heads on the primitive-score vector for $d \in \{0,1,2,3\}$, defining $d^\star(T, L) = \min\{d : \text{test AUROC}_d(T, L) \ge 0.99\}$. Primitive-probe AUROCs are $1.000$ at every tested layer, so the recovered $d^\star$ reflects the degree of the composition, not a primitive-detection failure.

\begin{figure}[h]
\centering
\includegraphics[width=0.48\linewidth]{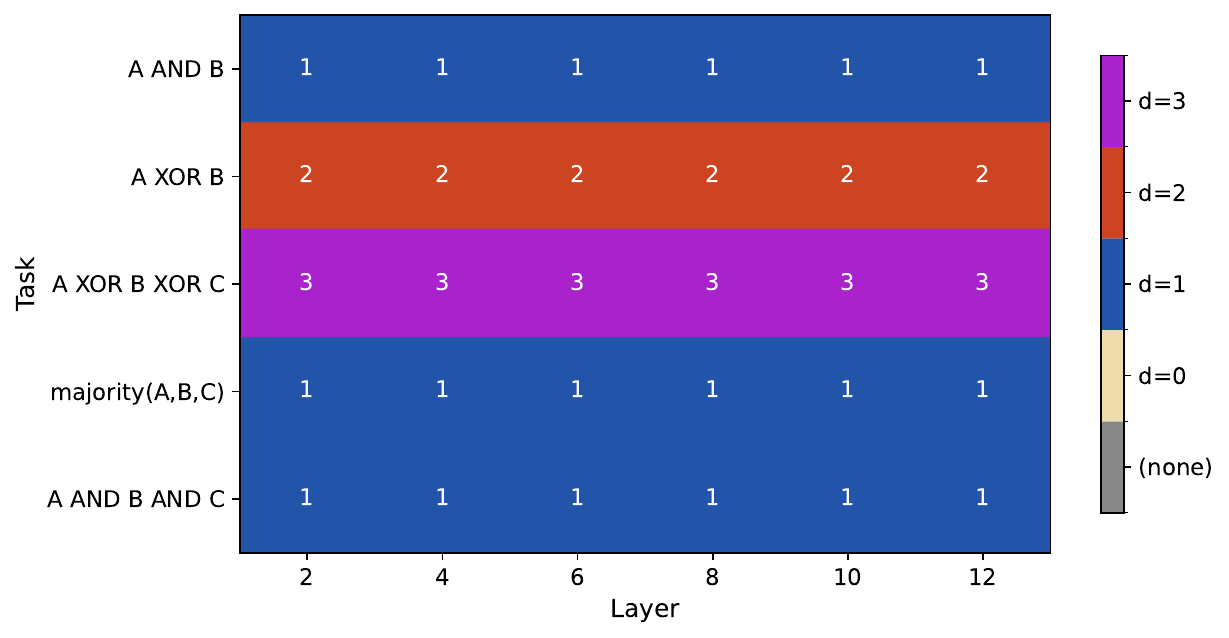}\hfill
\includegraphics[width=0.48\linewidth]{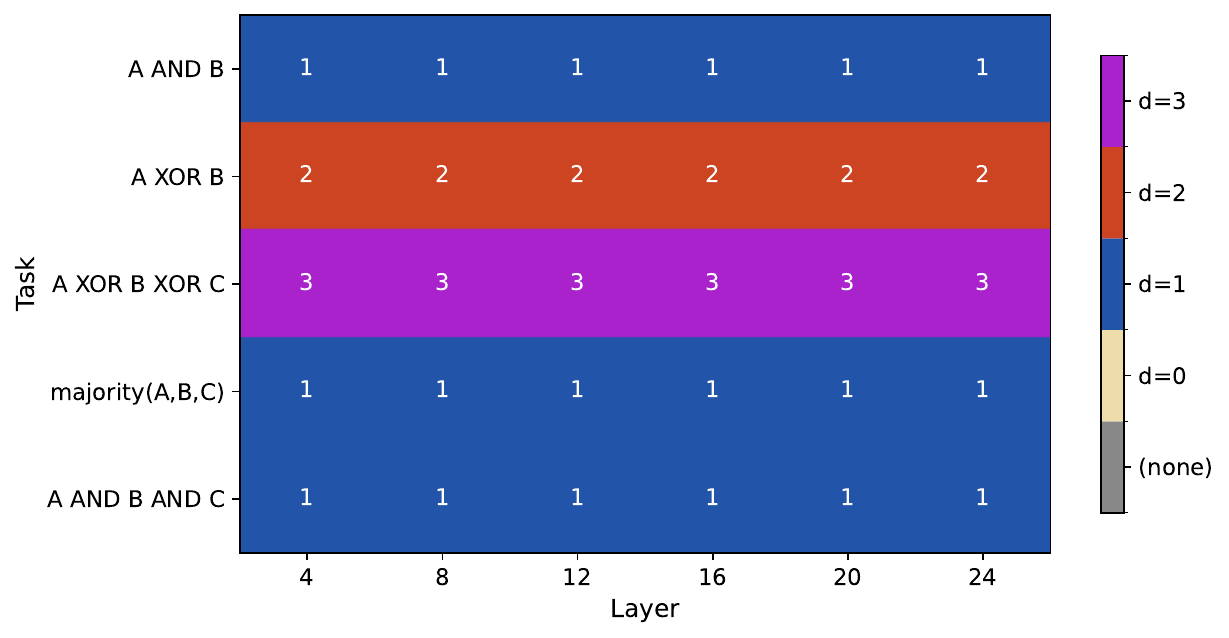}
\caption{Recovered minimum probe degree $d^\star(T, L)$ for five Boolean tasks at six layers of Pythia-160m (left) and Pythia-410m (right). Cells are shaded by $d^\star$. ``Degree'' here refers to the minimum polynomial-threshold decision degree over the primitive scores (whether the label is separable by a polynomial threshold of degree $d$ applied to the $(\widehat{A}, \widehat{B}, \widehat{C})$ score vector), not to the degree of the Boolean function as a polynomial over $\{0,1\}$. Under that definition the conjunctive and majority tasks have ground-truth degree $1$ and the XOR / parity-3 tasks have degree $2$ and $3$. The recovered degree matches ground truth at every tested layer of both models.}
\label{fig:degree_multitask}
\end{figure}

Figure~\ref{fig:degree_multitask} shows the recovered $d^\star$ heatmap. The three conjunctive or majority tasks ($A \land B$, $\text{majority}(A,B,C)$, $A \land B \land C$) recover $d^\star = 1$ at every tested layer of both models. The XOR task $A \oplus B$ recovers $d^\star = 2$, and the parity-3 task $A \oplus B \oplus C$ recovers $d^\star = 3$. The parity-3 case is the sharpest check: degree-1 stays at chance ($0.42$--$0.48$ across the twelve tiles) and degree-2 saturates below threshold ($\le 0.68$), while degree-3 reaches AUROC $0.999$--$1.000$ at every layer. The same pattern holds whether we sweep early or late layers, confirming that $d^\star$ is a property of the task rather than the representation depth.

\paragraph{Circular parity (the hierarchy is tight).}
To verify that degree~$N$ is both necessary \emph{and} sufficient for every $N$, we place $2N$ points uniformly on the unit circle with alternating labels (Figure~\ref{fig:main}a). For every $N\in\{2,\dots,8\}$, degree $N{-}1$ is at chance ($\approx 0.50$) while degree $N$ achieves test accuracy $0.985$--$0.995$ on 200 held-out points. This is not a sample-efficiency effect: degree $N{-}1$ cannot even interpolate the $2N$ training points. The polynomial spaces $\Poly{1}\subset\Poly{2}\subset\cdots$ form a strict filtration with no gaps.

\subsection{Exact reparameterization robustness test}
\label{subsubsec:exact_reparam}

Proposition~\ref{prop:low_rank_invariance} makes a sharp prediction: complete quadratic families survive affine coordinate change, whereas sparsity in a monomial basis does not. We test that distinction in a tightly controlled setting -- no retraining and no regularization changes -- so the result reflects the function class rather than optimizer behavior. We generate a degree-2 regression target $y=\ip{a}{z}\ip{b}{z}+\ip{c}{z}+\eps$ in $\R^{64}$ and fit two probes in the original coordinates: a full quadratic (all monomials up to degree~2, Ridge $\alpha=10^{-6}$) and a sparse quadratic (top-50 monomials by correlation, same Ridge). We then apply $T=20$ random well-conditioned affine transforms $z'=Az+t$ and evaluate each probe in the new coordinates \emph{without retraining}. For the full quadratic, the coefficients can be transported analytically because $\Poly{2}$ is affine-stable. For the sparse probe, the same rewrite becomes dense in the new basis, but the probe is still restricted to the original 50 monomial slots, which now correspond to different functions.

Across 5 seeds $\times$ 20 transforms, the transported full quadratic reproduces scores to machine precision (maximum error $1.14\times 10^{-11}$). The sparse probe, evaluated with its original sparsity pattern in the new coordinates, fails under reparameterization ($R^2=-18.3\pm3.5$). The distinction is exactly the one predicted by the theory: completeness of the polynomial family is coordinate-free; sparsity of monomials is not. Appendix~\ref{subsubsec:exp_a} gives the empirical analogue with retraining across affine-equivalent spaces.

\subsection{Basis stability on Pythia}
\label{subsec:basis_stability_appendix}

The exact reparameterization test in Appendix~\ref{subsubsec:exact_reparam} measures the analytic symmetry of the function class under affine reparameterization. We complement that test with an empirical basis-stability measurement on real Pythia hidden states, where regularization, optimization, and parameterization choices interact with the underlying symmetry.

\paragraph{Setup.}
On real Pythia-70m quotient coordinates ($k=29$ obtained from a 30-probe UD morphosyntactic bank; Appendix~\ref{subsec:per_experiment_config}), we apply 20 random affine basis changes and retrain each probe family after every change. Table~\ref{tab:basis_stability_main} measures the stability of the learned procedure -- regularization, optimization, and parameterization included -- not just the analytic symmetry of the underlying function class.

\begin{table}[h]
\centering
\small
\caption{Basis stability on Pythia-70m quotient coordinates (20 random affine transforms, 5 interaction targets). CP rank-1 achieves near-top degree-2 accuracy with the lowest basis variance.}
\label{tab:basis_stability_main}
\begin{tabular}{l c c}
\toprule
Probe & Mean AUROC & Basis std \\
\midrule
Linear & \pmci{0.969}{.025} & 0.0008 \\
\textbf{CP rank-1} & \textbf{\pmci{0.985}{.006}} & \textbf{0.0011} \\
CP rank-2 & \pmci{0.986}{.007} & 0.0027 \\
Diag.\ quadratic & \pmci{0.978}{.013} & 0.0030 \\
Full quadratic & \pmci{0.985}{.007} & 0.0040 \\
MLP (matched params) & \pmci{0.848}{.252} & 0.1883 \\
\bottomrule
\end{tabular}
\end{table}

CP rank~1 is the best overall compromise: it achieves 0.985 AUROC (within 0.001 of the top degree-2 probe) while having the smallest basis variance (0.0011 versus 0.0040 for the full quadratic). Higher CP ranks add little or no accuracy and make the probe more basis-sensitive, again suggesting that the interactions tested here are effectively rank~1. Appendix~\ref{subsubsec:cp_rank} shows the same compression pattern in the synthetic regression setting, and Appendix~\ref{subsubsec:exp_c} verifies the softmax symmetry numerically.

\subsection{Area as degree-2 regression}
\label{subsubsec:exp_a}

We begin with a controlled synthetic test of the degree hierarchy. Because area $y = wh$ is a degree-2 polynomial, degree-1 probes must incur irreducible error, while Proposition~\ref{prop:low_rank_invariance} predicts that CP rank--but not monomial sparsity--is preserved under $\GL(d)$ reparameterization. To verify this, we sample $w,h\sim\text{Uniform}[0,1]$, set $y=wh$, and embed into~$\R^{64}$ with a nuisance subspace:
\begin{equation}
\label{eq:area_embed}
    z = B\begin{bmatrix}w\\h\end{bmatrix} + b + N\eta,
    \qquad B\in\R^{64\times 2},\;\; b\in\R^{64},\;\; \eta\sim\mathcal{N}(0,\sigma^2 I_{64}),\;\; \sigma=0.1.
\end{equation}
We create $M=4$ equivalent hidden spaces via $z^{(m)}=A_m z + t_m$, $A_m\in\GL(64)$, $t_m\sim\mathcal{N}(0,I_{64})$ (5{,}000 train / 1{,}000 test samples) and fit five probe families: (i)~affine Ridge ($[z;\,1]$, 65 params), (ii)~full quadratic Ridge ($[\text{vech}(zz^\top);\,z;\,1]$, 2{,}145 params), (iii)~polynomial kernel Ridge with $k(z,z')=(\ip{z}{z'}+1)^2$, (iv)~low-rank CP with rank $R$ cross-validated from $\{1,2,4,8,16\}$ (rank-1 has 196 params; the CV-selected rank-16 has 2{,}161), and (v)~L1-regularized sparse quadratic ($\le 50$ nonzero monomials). We measure $R^2$ on the test set and transfer $R^2$ when probes are transported across the four equivalent spaces.

\begin{table}[h]
\centering
\small
\caption{Area regression: probe performance and transfer across 4 affine-equivalent spaces ($d=64$, 5 seeds). The CP probe (cross-validated rank, 2{,}161 params) nearly matches full quadratic while transferring perfectly. Sparse and kernel probes break on transfer.}
\label{tab:exp_a}
\begin{tabular}{l r c c}
\toprule
Probe family & \# Params & In-space $R^2$ & Transfer $R^2$ \\
\midrule
Affine (degree-1) & 65 & \pmci{0.853}{.006} & \pmci{0.853}{.006} \\
Full quadratic & 2{,}145 & \pmci{1.000}{.000} & \pmci{1.000}{.000} \\
Low-rank CP (CV) & 2{,}161 & \pmci{0.991}{.001} & \pmci{0.991}{.001} \\
Poly kernel Ridge & implicit & \pmci{1.000}{.000} & $<0$ \\
Sparse quadratic & 496 & \pmci{0.915}{.087} & $<0$ \\
\bottomrule
\end{tabular}
\end{table}

Table~\ref{tab:exp_a} confirms both predictions. Degree-2 probes capture the product perfectly ($R^2=1.000$), and the CP probe recovers 99.1\% of full quadratic performance while transferring perfectly across all equivalent spaces. Although cross-validation selects rank~16, Table~\ref{tab:cp_sweep} shows that rank~1 alone captures 98.4\% of full-quadratic $R^2$ with $11\times$ fewer parameters. Sparse quadratic and polynomial kernel probes, by contrast, achieve reasonable in-space accuracy but fail on transfer ($R^2<0$). The monomial basis in which sparsity is defined is coordinate-dependent, and the kernel's feature map is not equivariant under $\GL(d)$. This is consistent with the theoretical prediction that CP rank is a geometric invariant while monomial sparsity is a coordinate artifact.

\subsection{Quadratic scaling on real hidden dimensions}
\label{subsec:quadratic_scaling}
\label{subsubsec:scaling}

The theoretical completeness of $\Poly{2}$ raises an immediate practical question: can full quadratic probes, with their $O(d^2)$ parameter count, actually be trained at the hidden dimensions of real language models? To find out, we fit four probe families--affine ($d+1$ params), diagonal quadratic ($2d+1$), full quadratic ($\tfrac{d(d+1)}{2}+d+1$), and PCA-reduced quadratic (project to $k\in\{16,32,64\}$ dimensions first)--on morphological number (Sing/Plur) across Pythia-70m ($d{=}512$), 160m ($d{=}768$), and 410m ($d{=}1024$).

\begin{table}[h]
\centering
\small
\caption{Probe scaling on morphological number (Sing/Plur). Full quadratic becomes infeasible at $d\ge 768$ within our computational budget due to $O(d^2)$ parameter count. Morphological number is an approximately linear concept, so affine probes suffice here.}
\label{tab:scaling}
\begin{tabular}{l c r r c}
\toprule
Probe & $d$ & \# Params & Time & Bacc \\
\midrule
Affine & 512 & 513 & 38s & 0.919 \\
Diag Quad & 512 & 1{,}025 & 6min & 0.920 \\
Full Quad & 512 & 131{,}841 & 44min & 0.917 \\
PCA-64 Quad & 512 & 2{,}145 & 8min & 0.900 \\
\midrule
Affine & 768 & 769 & 84s & 0.926 \\
Full Quad & 768 & 296{,}065 & \multicolumn{2}{c}{\emph{infeasible}} \\
PCA-64 Quad & 768 & 2{,}145 & 6min & 0.912 \\
\midrule
Affine & 1{,}024 & 1{,}025 & 67s & 0.929 \\
Full Quad & 1{,}024 & 525{,}825 & \multicolumn{2}{c}{\emph{infeasible}} \\
PCA-64 Quad & 1{,}024 & 2{,}145 & 6min & 0.855 \\
\bottomrule
\end{tabular}
\end{table}

As Table~\ref{tab:scaling} shows, the full quadratic probe is feasible only at $d{=}512$, where it already takes 44 minutes to train. At $d{=}768$ (296K parameters) it fails to converge within resource limits. Importantly, morphological number is an approximately linear concept: the affine probe achieves equal or higher balanced accuracy at every model scale, suggesting that morphological number is well-captured by degree-1 features. This demonstrates that the degree-2 advantage observed in score-space composition (Section~\ref{subsec:degree_hierarchy}) is concept-specific--some interactions genuinely require degree-2, but many individual features do not--and motivates the need for structured approximations. The Pythia family is small by modern standards ($d\le 1024$); for models like Qwen-14B ($d=5{,}120$), the full quadratic feature vector would have over 13 million components. The CP decomposition (Section~\ref{subsubsec:cp_rank}) and quotient-level quadratic (Section~\ref{subsubsec:composition}) provide scalable alternatives that preserve affine invariance.

\subsection{Low-rank CP: efficient quadratic approximation}
\label{subsubsec:cp_rank}

Given the scaling limits just identified, we now characterize the accuracy--compression trade-off of CP-decomposed quadratic probes. Using the same area-regression setup (Section~\ref{subsubsec:exp_a}, $\R^{64}$), we train CP probes at ranks $r\in\{1,2,4,8,16\}$ and compare to the full quadratic. The rank-$r$ CP probe parameterizes the quadratic form as
\begin{equation}
\label{eq:cp_decomp}
    \hat{y} = \sum_{j=1}^{r} \alpha_j\bigl(\ip{u_j}{z}+u_{j,0}\bigr)\bigl(\ip{v_j}{z}+v_{j,0}\bigr) + \ip{w}{z} + w_0,
\end{equation}
with $2r(d+1)+r+(d+1)$ parameters total (each rank term contributes $2(d+1)$ from the vectors $u_j,v_j$ and biases, plus one scalar weight).

\begin{table}[h]
\centering
\small
\caption{CP rank sweep on area regression ($d=64$, 5 seeds). Rank-1 captures 98.4\% of full quadratic $R^2$ with $11\times$ compression.}
\label{tab:cp_sweep}
\begin{tabular}{r r c}
\toprule
CP Rank & \# Params & $R^2$ \\
\midrule
1 & 196 & \pmci{0.984}{.003} \\
2 & 327 & \pmci{0.986}{.002} \\
4 & 589 & \pmci{0.989}{.002} \\
8 & 1{,}113 & \pmci{0.991}{.002} \\
16 & 2{,}161 & \pmci{0.992}{.002} \\
\midrule
Full & 2{,}145 & \pmci{1.000}{.000} \\
\bottomrule
\end{tabular}
\end{table}

Because area has exact tensor rank~1, the CP-1 probe recovers it almost perfectly ($R^2=0.984$; the residual is likely due to optimization noise), and higher ranks provide diminishing returns (Table~\ref{tab:cp_sweep}). At rank~16 the CP probe has comparable parameter count to the full quadratic but slightly lower $R^2$--an expected consequence of the non-convex CP optimization landscape. When the concept has low intrinsic tensor rank, CP probes provide orders-of-magnitude compression with negligible accuracy loss, and unlike sparse probes, this compression is preserved under affine reparameterization (Proposition~\ref{prop:low_rank_invariance}).

\subsection{Softmax symmetry verification}
\label{subsubsec:exp_c}

We next turn from the degree hierarchy to affine invariance, beginning with a direct numerical check of the softmax symmetry. Proposition~\ref{prop:softmax_affine_symmetry} asserts that softmax output probabilities are invariant under $z_*=A^{-\top}z$, $\Lambda_*=\Lambda A^\top + \mathbf{1}c^\top$ for any $A\in\GL(d)$ and $c\in\R^d$. To verify this, we fix $d=32$ and $n=10$ classes, generate a random logit matrix $\Lambda\in\R^{10\times 32}$ and 1{,}000 hidden states $z\sim\mathcal{N}(0,I_{32})$, compute baseline probabilities $p_i(z)=\text{softmax}(\Lambda z)_i$, and then produce $T=100$ equivalent parameterizations:
\begin{equation}
\label{eq:softmax_equiv}
    A_t\in\GL(32),\quad c_t\in\R^{32},\qquad
    z_*^{(t)} = A_t^{-\top}z,\qquad
    \Lambda_*^{(t)} = \Lambda A_t^\top + \mathbf{1}c_t^\top.
\end{equation}
We measure the maximum absolute discrepancy $\max_{t,i,z} |p_i(z) - p_i^{(t)}(z_*^{(t)})|$ across all 100 reparameterizations, 10 classes, and 1{,}000 data points. The result is $3.43\times 10^{-14}$--machine epsilon for float64--and the common-logit-shift vector $c$ is recovered exactly in all 100 cases. This confirms Proposition~\ref{prop:softmax_affine_symmetry} to machine precision: under $\GL(d)$ reparameterization $z_*=A^{-\top}z$ together with the softmax common-logit shift $\Lambda_*=\Lambda A^\top + \mathbf{1}c^\top$ ($c\in\R^d$), the softmax output is unchanged.

\subsection{SST-2 sentiment transfer across Pythia sizes}
\label{subsubsec:track_s}

Having validated affine invariance on synthetic data, we now test quotient transfer on a real NLP task. Sentiment is expected to be an approximately linear concept shared across Pythia scales, so Theorem~\ref{thm:common_abstract_space} predicts that quotient alignment should fully recover a probe trained on one model when applied to another. We train a sentiment probe (SST-2, positive/negative) on Pythia-70m ($d{=}512$) using the middle layer's hidden states and transfer without target labels to Pythia-160m ($d{=}768$) and Pythia-410m ($d{=}1024$) via quotient alignment learned from 2{,}000 unlabeled paired activations. We compare the transferred probe to a fresh probe trained on the target model with full labels (upper bound), measuring balanced accuracy, AUROC, and the recovery ratio $\text{AUROC}(\text{quotient})/\text{AUROC}(\text{target-trained})$.

\begin{table}[h]
\centering
\small
\caption{Zero-label sentiment transfer from Pythia-70m via quotient alignment. AUROC columns report 95\% bootstrap CIs ($\pm$ half-width).}
\label{tab:track_s}
\begin{tabular}{l cc cc c}
\toprule
& \multicolumn{2}{c}{Balanced accuracy} & \multicolumn{2}{c}{AUROC} & \\
\cmidrule(lr){2-3} \cmidrule(lr){4-5}
Transfer & Quotient & Target & Quotient & Target & Recovery \\
\midrule
70m $\to$ 160m & \pmci{0.771}{.03} & \pmci{0.786}{.03} & \pmci{0.846}{.03} & \pmci{0.859}{.03} & \textbf{98.5\%} \\
70m $\to$ 410m & \pmci{0.801}{.03} & \pmci{0.818}{.03} & \pmci{0.875}{.02} & \pmci{0.885}{.02} & \textbf{98.9\%} \\
\bottomrule
\end{tabular}
\end{table}

As Table~\ref{tab:track_s} shows, both transfers achieve $>$98\% AUROC recovery--98.5\% for 160m and 98.9\% for 410m--with tight bootstrap CIs. Balanced accuracy recovery is similarly high (98.1\% for 160m, 97.9\% for 410m), with quotient and target-trained CIs overlapping. The small AUROC gap for 70m$\to$410m (0.875 vs 0.885) may reflect richer sentiment encoding in the larger model. Operationally, for linear concepts, zero-label transfer across model sizes works with no target labels at all, requiring only unlabeled paired activations for alignment.

\subsection{Lexical baselines and the role of neural representations.}
\label{subsubsec:lexical}

A natural concern is whether the concepts we probe genuinely require model internals or could be detected from surface text alone. To distinguish the two cases, we train TF-IDF logistic regression (unigrams + bigrams, 10K features) on the raw text of each evaluation dataset--a baseline that uses no model activations--and compare its AUROC to the quotient-transferred AUROC from the safety-transfer experiments (Section~\ref{subsec:transfer_experiments}).

\begin{table}[h]
\centering
\small
\caption{Lexical TF-IDF baseline (unigrams + bigrams, 10K features, logistic regression on raw text, no model activations) versus zero-label quotient transfer. Four of the five safety-transfer concepts are shown (toxicity, jailbreaking, sentiment, moderation); harmful-content is excluded because its evaluation set is constructed from the same ToxicChat-plus-BeaverTails vocabulary captured by the toxicity row, and a separate TF-IDF evaluation on that composite would not provide an independent lexical signal.}
\label{tab:lexical_baselines}
\begin{tabular}{lccc}
\toprule
Concept & TF-IDF & Quotient (Qwen-3B) & Quotient (Mistral) \\
\midrule
Toxicity & 0.903 & \pmci{0.952}{.010} & \pmci{0.926}{.013} \\
Jailbreaking & 0.981 & \pmci{0.932}{.021} & \pmci{0.669}{.054} \\
Sentiment & 0.834 & \textbf{\pmci{0.966}{.012}} & \textbf{\pmci{0.972}{.011}} \\
Moderation & 0.874 & \pmci{0.836}{.046} & \pmci{0.820}{.050} \\
\bottomrule
\end{tabular}
\end{table}

Toxicity and jailbreaking show high lexical signal: TF-IDF achieves 0.90--0.98 AUROC from word patterns alone, suggesting these detection tasks are well-solved by distinctive vocabulary (slurs, injection patterns, harmful keywords). Quotient transfer to Qwen-3B is competitive on these concepts (0.93--0.95) but offers only a modest advantage on this standard test set; transfer to Mistral is weaker for jailbreaking (0.669). Sentiment, by contrast, is where neural representations clearly add value: quotient transfer outperforms TF-IDF by $+$13pp (0.97 vs.\ 0.83), because sentiment depends on compositional meaning (``not bad'' = positive, sarcasm, negation scope) that bag-of-words features cannot capture.

\paragraph{TF-IDF robustness.}
The TF-IDF AUROC for toxicity in Table~\ref{tab:label_eff} (0.836 at \(B=500\)) is lower than in the lexical-baselines table above (0.904 with all ${\approx}$5{,}082 training labels). The gap reflects label budget: TF-IDF is trained on raw text rather than model activations, and its performance scales with the number of labeled training examples. At \(B=500\) with toxicity's ${\approx}$7.5\% positive rate in ToxicChat, stratified subsampling yields ${\approx}$38 toxic examples in the training set, enough to learn some distinctive vocabulary but fewer than the ${\approx}$383 positives available in the full training set. Both tables use the identical vectorizer (word 1--2-grams, \texttt{max\_features}${=}$10{,}000, \(\mathrm{sublinear\_tf}{=}\mathrm{True}\)) and the same logistic regression head with intercept, so the difference is entirely training-set size. These results do not change the comparison with quotient transfer: zero-label quotient transfer reaches 0.952 AUROC with no target labels at all, exceeding both TF-IDF@500 (0.836) and TF-IDF with all labels (0.904) on the same test set, and widening the gap further at the smaller label budgets in Table~\ref{tab:tfidf_budget_robustness}.

However, the standard-test TF-IDF results also overstate the method's robustness. Two additional analyses reveal their fragility.

\subsection{Lexical robustness: keyword-holdout and distribution-shift tests}
\label{app:lexical_robustness}

The standard test set contains the same distinctive keywords that TF-IDF learns during training--if ``kill'' appears in both training and test toxicity examples, TF-IDF can rely on that lexical shortcut without detecting the underlying concept. To test whether TF-IDF's strength is genuine or keyword-dependent, we construct two controlled test sets per concept:

\begin{itemize}[leftmargin=2em,itemsep=2pt]
\item \textbf{Keyword-holdout}: Remove from the test set all examples containing any of the 50 most predictive TF-IDF unigrams for that concept. This forces evaluation on texts where the most distinctive vocabulary is absent.
\item \textbf{Lexically-matched}: Create a balanced test set (50\% positive, 50\% negative) where positive and negative examples have similar vocabulary distributions, eliminating lexical shortcuts entirely.
\end{itemize}

\begin{table}[h]
\centering
\small
\caption{AUROC under distribution shift (TF-IDF trained on all labels; quotient is zero-label with 95\% bootstrap CIs). TF-IDF degrades when keywords are removed or vocabulary is balanced; quotient transfer is more robust.}
\label{tab:lexical_robustness}
\begin{tabular}{l l ccc}
\toprule
& & Standard & Keyword-holdout & Lexically-matched \\
\midrule
\multirow{2}{*}{Toxicity}
& TF-IDF (all labels) & 0.904 & 0.751 ($-$15pp) & 0.706 ($-$20pp) \\
& Quotient (Qwen-3B) & 0.951\,[.94,.96] & 0.969\,[.95,.99] & 0.852\,[.82,.88] \\
\midrule
\multirow{2}{*}{Sentiment}
& TF-IDF (all labels) & 0.845 & 0.702 ($-$14pp) & 0.831 ($-$1pp) \\
& Quotient (Qwen-3B) & 0.966\,[.95,.98] & 0.944\,[.90,.98] & 0.969\,[.96,.98] \\
\midrule
\multirow{2}{*}{Jailbreaking}
& TF-IDF (all labels) & 0.982 & --- & 0.828 ($-$15pp) \\
& Quotient (Qwen-3B) & 0.934\,[.91,.95] & --- & 0.778\,[.71,.85] \\
\midrule
\multirow{2}{*}{Moderation}
& TF-IDF (all labels) & 0.867 & 0.784 ($-$8pp) & 0.725 ($-$14pp) \\
& Quotient (Qwen-3B) & 0.836\,[.79,.88] & 0.907\,[.72,1.0] & 0.767\,[.70,.82] \\
\bottomrule
\end{tabular}
\end{table}

For toxicity, TF-IDF drops from 0.904 to 0.706 ($-$20pp) on the lexically-matched test, while quotient transfer drops from 0.951 to 0.852 ($-$10pp). For sentiment, TF-IDF loses 14pp on keyword-holdout while quotient transfer loses only 2pp (0.966$\to$0.944). For jailbreaking, removing the top-50 TF-IDF cue words from the test set eliminates all positive examples; AUROC is therefore undefined for both TF-IDF and the quotient probe in that cell (reported as \textbf{---}), and we do not report a number. This degeneracy is itself informative: for jailbreaking in ToxicChat, the top-50 cue words are essentially coextensive with the positive class -- every positive test example contained at least one of them -- suggesting that jailbreaking in this benchmark has a strongly lexical character at the label level.

\paragraph{Low-data fragility.}
The fragility of TF-IDF becomes more extreme when training labels are also limited. Table~\ref{tab:tfidf_budget_robustness} shows the full budget $\times$ test-variant matrix.

\begin{table}[h]
\centering
\footnotesize
\caption{TF-IDF AUROC at different label budgets and test conditions (mean $\pm$ std over 10 subsamples). At $B{=}25$, TF-IDF is near chance on toxicity and sentiment but can still be high on jailbreaking's standard test (0.84) due to distinctive lexical patterns. For reference, zero-label quotient transfer ($B{=}0$) achieves 0.88--0.96 on the same standard tests.}
\label{tab:tfidf_budget_robustness}
\begin{tabular}{l l ccc}
\toprule
Concept & Budget & Standard & Keyword-holdout & Lexically-matched \\
\midrule
\multirow{4}{*}{Toxicity}
& $B{=}25$  & \pmci{0.668}{.04} & \pmci{0.508}{.06} & \pmci{0.576}{.03} \\
& $B{=}100$ & \pmci{0.740}{.03} & \pmci{0.560}{.04} & \pmci{0.604}{.03} \\
& $B{=}500$ & \pmci{0.836}{.01} & \pmci{0.668}{.04} & \pmci{0.645}{.02} \\
& All       & 0.904              & 0.751              & 0.706 \\
\midrule
\multirow{4}{*}{Sentiment}
& $B{=}25$  & \pmci{0.542}{.03} & \pmci{0.528}{.03} & \pmci{0.547}{.05} \\
& $B{=}100$ & \pmci{0.614}{.03} & \pmci{0.542}{.04} & \pmci{0.586}{.03} \\
& $B{=}500$ & \pmci{0.710}{.02} & \pmci{0.569}{.03} & \pmci{0.691}{.02} \\
& All       & 0.845              & 0.702              & 0.831 \\
\midrule
\multirow{4}{*}{Jailbreaking}
& $B{=}25$  & \pmci{0.842}{.13} & --- & \pmci{0.669}{.08} \\
& $B{=}100$ & \pmci{0.934}{.03} & --- & \pmci{0.719}{.06} \\
& $B{=}500$ & \pmci{0.967}{.00} & --- & \pmci{0.766}{.02} \\
& All       & 0.982              & ---                & 0.828 \\
\bottomrule
\end{tabular}
\end{table}

Several patterns emerge. First, on the standard test set, TF-IDF at $B{=}25$ achieves 0.668 on toxicity and 0.542 on sentiment--both near chance--while zero-label quotient transfer ($B{=}0$) achieves 0.918 and 0.960 respectively. TF-IDF needs hundreds of labels to become competitive; quotient transfer needs none. Second, on the keyword-holdout test, the jailbreaking cell is \textbf{---} (undefined) for TF-IDF at every budget and for the quotient probe: removing the top-50 TF-IDF cue words strips all positive examples from the ToxicChat jailbreak test set, so AUROC cannot be computed by either method. The degeneracy itself is diagnostic: jailbreaking positives in ToxicChat are almost entirely characterized by a small distinctive vocabulary. Third, the variance of TF-IDF at low budgets is large ($\pm$0.05--0.13) on the standard and lexically-matched tests, reflecting instability in the learned vocabulary patterns, while quotient transfer (being deterministic at $B{=}0$) has zero variance.

Together, these results suggest that TF-IDF's competitive performance on standard benchmarks is at least partly explained by lexical overlap between training and test sets. In realistic deployment settings--where the vocabulary distribution shifts and labeled data is scarce--quotient transfer is substantially more reliable.

\subsection{Agreement gating: deployment with polynomial probes}
\label{subsubsec:gating}

Degree-2 probes may improve practical deployment decisions even when the bilinear-linear gap on raw hidden states is small. To test this, we train linear, bilinear, and quotient-quadratic gates for subject-verb agreement detection, sweeping a threshold $\tau$ on the gate's output probability to produce flag-rate vs.\ miss-rate curves, and evaluate AUROC for detecting disagreement.

\begin{table}[h]
\centering
\small
\caption{Subject--verb disagreement detection AUROC for linear, bilinear, and quotient-quadratic gates on Pythia-70m / 160m / 410m. Intervals are 95\% bootstrap half-widths over 5 seeds.}
\label{tab:agreement_gating}
\begin{tabular}{lccc}
\toprule
Model & Linear AUROC & Bilinear AUROC & Quotient AUROC \\
\midrule
Pythia-70m & \pmci{0.748}{.01} & \pmci{0.778}{.01} & \textbf{\pmci{0.795}{.01}} \\
Pythia-160m & \pmci{0.738}{.01} & \pmci{0.780}{.01} & \textbf{\pmci{0.808}{.01}} \\
Pythia-410m & \pmci{0.735}{.01} & \pmci{0.786}{.01} & \textbf{\pmci{0.869}{.01}} \\
\bottomrule
\end{tabular}
\end{table}

The bilinear gate modestly outperforms linear ($+$3--5pp AUROC), but the quotient gate outperforms both ($+$5--13pp over linear). This is likely because quotient compression to $\sim$29 concept-relevant dimensions removes noise that hurts the full-dimensional bilinear product. This demonstrates a practical consequence of the quotient construction: it is not only theoretically correct but also empirically beneficial as a preprocessing step for higher-degree probing.

\subsection{Probe composition through polynomial heads}
\label{subsubsec:composition}

Because agreement is a degree-2 function of primitive features (number of subject $\times$ number of verb), Theorem~\ref{thm:scalar_affine_classification} predicts that degree-2 is necessary. We test whether a polynomial head on primitive probe outputs can compose this concept more effectively than a bilinear probe on raw hidden states. Given primitive probes for 5 UD concepts (POS, number, tense, punctuation, dependency relation), we compare six composition methods on Pythia-410m ($d{=}1024$, middle layer). To avoid sentence-identity confounds, we use cross-sentence pairs for both classes: positives pair a subject with a verb from a different sentence with the same number, negatives with different number.
\begin{enumerate}[leftmargin=2em,itemsep=2pt]
    \item \textbf{Hard rule:} deterministic threshold on primitive probe scores.
    \item \textbf{Linear on scores:} logistic regression on the 5 primitive probe scores.
    \item \textbf{Quadratic on scores:} logistic regression on $[s_{\text{subj}} \odot s_{\text{verb}};\, s_{\text{subj}};\, s_{\text{verb}}]$ where $s$ are the 3 morphological-number probe scores (9 features).
    \item \textbf{Linear on raw concat:} logistic regression on concatenated hidden states $[h_{\text{subj}};\,h_{\text{verb}}]$ ($2d$ features).
    \item \textbf{Bilinear on raw:} logistic regression on $[h_{\text{subj}} \odot h_{\text{verb}};\, h_{\text{subj}};\, h_{\text{verb}}]$ ($3d$ features).
    \item \textbf{Quotient-level quadratic:} project each hidden state through the quotient, then fit a quadratic on quotient coordinates.
\end{enumerate}

\begin{table}[h]
\centering
\small
\caption{Probe composition for agreement detection (Pythia-410m, layer 16, cross-sentence evaluation, 10 resamples). Quadratic heads on primitive probe scores outperform bilinear probes on raw hidden states.}
\label{tab:composition}
\begin{tabular}{l c}
\toprule
Method & Bacc \\
\midrule
Random baseline & \pmci{0.498}{.02} \\
Hard rule & \pmci{0.805}{.01} \\
Linear on scores & \pmci{0.705}{.00} \\
Quadratic on scores & \textbf{\pmci{0.863}{.01}} \\
Linear on raw concat & \pmci{0.695}{.01} \\
Bilinear on raw & \pmci{0.730}{.01} \\
Quotient-level quadratic & \pmci{0.746}{.01} \\
\bottomrule
\end{tabular}
\end{table}

Table~\ref{tab:composition} shows that the quadratic head on probe scores (bacc $= 0.863 \pm 0.007$) outperforms the bilinear probe on raw hidden states, though these operate on different feature representations (5 pre-trained probe scores vs $3d$ raw dimensions) ($0.730 \pm 0.006$), despite having only 9 features versus $3d=3{,}072$. Composition through primitive probe outputs is both more effective and more practical than working in the full hidden space. The quotient-level quadratic ($0.746 \pm 0.007$) operates on $\sim\!29$ concept-relevant dimensions--far fewer than the raw bilinear's $3d$. Notably, the hard rule ($0.805 \pm 0.007$) outperforms both linear methods, suggesting that the primitive probes already extract most of the relevant information; the degree-2 head's role is to capture the multiplicative interaction between number features, which a linear combination of scores cannot represent.

\paragraph{Summary of quotient-space experiments.}
Across all three settings--synthetic (Section~\ref{subsec:transfer_experiments}), within-family (Section~\ref{subsubsec:track_s}), and cross-family (Section~\ref{subsec:transfer_experiments})--the quotient construction consistently exhibits two properties.
First, \emph{transfer fidelity}: quotient alignment recovers 98.5--98.9\% of target-trained AUROC for concepts within the probe bank (Table~\ref{tab:track_s}).
Second, \emph{coverage awareness}: the quotient tends to reject out-of-span concepts (Section~\ref{subsec:transfer_experiments}) and fails on random-init models (Section~\ref{subsec:transfer_experiments}), providing coverage diagnostics that standard full-state alignment methods do not provide. On real data, the in-span fraction is a useful but imperfect predictor--some low-ISF concepts still transfer via correlations with bank members, while some moderate-ISF concepts fail (Section~\ref{subsec:continuum}).
The practical implication is that organizations can build a probe bank once and port it across model upgrades with approximate coverage diagnostics, a workflow that the theory (Theorem~\ref{thm:common_abstract_space}) both motivates and grounds.

\subsection{Higher-degree score-space composition}
\label{subsec:higher_degree_appendix}

Section~\ref{subsec:degree_hierarchy} introduces score-space composition for cross-token interactions. Here we apply the same methodology to single-token 3-way conjunctions on Universal Dependencies, asking whether degree-3 adds anything beyond degree-2.

\paragraph{UD 3-way conjunctions.}
We construct 3-way AND targets (e.g., NOUN $\wedge$ SING $\wedge$ NSUBJ) and compose primitive probe scores with degree-1 through degree-3 heads across all three Pythia scales. Table~\ref{tab:higher_degree_ud} reports all 3 targets across all 3 models.

\begin{table}[h]
\centering
\small
\caption{Higher-degree score-space composition on 3-way UD conjunctions (AUROC, mean $\pm$ std, 5 seeds). Degree-2 provides consistent gains; degree-3 adds at most 0.1pp.}
\label{tab:higher_degree_ud}
\begin{tabular}{ll cccc}
\toprule
Model & Target & deg1 lin & deg2 full & deg2 CP-1 & deg3 full \\
\midrule
70m & noun\_plur\_obj & \pmci{.988}{.003} & \pmci{.993}{.002} & \pmci{.992}{.002} & \pmci{.993}{.002} \\
70m & noun\_sing\_nsubj & \pmci{.956}{.006} & \pmci{.976}{.003} & \pmci{.974}{.004} & \pmci{.977}{.003} \\
70m & nominal\_plur\_obj & \pmci{.991}{.001} & \pmci{.993}{.001} & \pmci{.993}{.001} & \pmci{.994}{.001} \\
\midrule
160m & noun\_plur\_obj & \pmci{.991}{.002} & \pmci{.996}{.002} & \pmci{.995}{.002} & \pmci{.996}{.002} \\
160m & noun\_sing\_nsubj & \pmci{.968}{.003} & \pmci{.983}{.001} & \pmci{.982}{.002} & \pmci{.984}{.001} \\
160m & nominal\_plur\_obj & \pmci{.993}{.001} & \pmci{.996}{.001} & \pmci{.996}{.001} & \pmci{.996}{.001} \\
\midrule
410m & noun\_plur\_obj & \pmci{.990}{.002} & \pmci{.996}{.001} & \pmci{.995}{.001} & \pmci{.996}{.001} \\
410m & noun\_sing\_nsubj & \pmci{.973}{.004} & \pmci{.987}{.002} & \pmci{.984}{.002} & \pmci{.987}{.001} \\
410m & nominal\_plur\_obj & \pmci{.993}{.002} & \pmci{.996}{.000} & \pmci{.995}{.000} & \pmci{.996}{.001} \\
\bottomrule
\end{tabular}
\end{table}

Degree-2 provides a statistically significant AUROC gain over degree-1 on every target and model, ranging from $+$0.2pp on easy targets (nominal\_plur\_obj) to $+$2.0pp on the hardest (noun\_sing\_nsubj). The magnitude scales with target difficulty, consistent with pre-linearization: easy targets are nearly saturated by degree-1, hard ones retain more structure for degree-2 to capture. CP rank-1 captures most of the degree-2 signal. Degree-3 adds at most 0.1pp, consistent with pairwise interactions dominating.

\subsection{Cross-token score-space composition: full per-target breakdown}
\label{subsec:cross_tok_appendix}

Section~\ref{subsec:degree_hierarchy} reports the headline cross-token result on subject--verb agreement (Table~\ref{tab:cross_tok_main}). Here we give the full per-target breakdown: agreement, AND/XOR variants combining agreement with verb tense, a 3-way AND target, and \texttt{subj\_plur\_AND}, evaluated across three Pythia scales. The setup matches the main text: primitive probe scores from subject and verb tokens are concatenated into a 60-dimensional score space (30 probes $\times$ 2 tokens) and composition heads are trained on this space, with cross-sentence pairing for both classes to avoid sentence-identity confounds. Raw-activation baselines (linear concat and bilinear on $2d$-dimensional hidden states) are reported in the rightmost two columns.

\begin{table}[h]
\centering
\footnotesize
\caption{Cross-token score-space composition (AUROC, mean $\pm$ std over 10 resamples). CP rank-1 outperforms the linear head on every cross-token target and matches or exceeds full quadratic; the largest gains over both baselines are on agreement and the XOR variants, where the linear baseline is far from saturation.}
\label{tab:cross_tok}
\begin{tabular}{ll ccccc}
\toprule
Model & Target & Linear & Full quad & CP-1 & Raw lin & Raw bilin \\
\midrule
70m & agreement & \pmci{.733}{.005} & \pmci{.819}{.013} & \pmci{.901}{.004} & \pmci{.677}{.009} & \pmci{.703}{.012} \\
70m & agree\_AND\_v.past & \pmci{.937}{.015} & \pmci{.942}{.011} & \pmci{.969}{.009} & \pmci{.934}{.011} & \pmci{.937}{.014} \\
70m & agree\_AND\_v.pres & \pmci{.860}{.009} & \pmci{.885}{.009} & \pmci{.935}{.004} & \pmci{.855}{.010} & \pmci{.857}{.010} \\
70m & agree\_XOR\_v.past & \pmci{.635}{.009} & \pmci{.684}{.016} & \pmci{.798}{.007} & \pmci{.609}{.011} & \pmci{.598}{.011} \\
70m & 3-way AND & \pmci{.937}{.015} & \pmci{.936}{.010} & \pmci{.966}{.008} & \pmci{.935}{.015} & \pmci{.938}{.012} \\
70m & subj\_plur\_AND & \pmci{.990}{.004} & \pmci{.991}{.005} & \pmci{.994}{.004} & \pmci{.992}{.002} & \pmci{.991}{.004} \\
\midrule
160m & agreement & \pmci{.740}{.003} & \pmci{.855}{.007} & \pmci{.930}{.002} & \pmci{.639}{.012} & \pmci{.694}{.013} \\
160m & agree\_AND\_v.past & \pmci{.949}{.015} & \pmci{.954}{.008} & \pmci{.980}{.004} & \pmci{.927}{.015} & \pmci{.940}{.014} \\
160m & agree\_AND\_v.pres & \pmci{.892}{.010} & \pmci{.917}{.012} & \pmci{.958}{.005} & \pmci{.854}{.009} & \pmci{.874}{.011} \\
160m & agree\_XOR\_v.past & \pmci{.612}{.010} & \pmci{.689}{.014} & \pmci{.800}{.006} & \pmci{.583}{.015} & \pmci{.598}{.014} \\
160m & subj\_plur\_AND & \pmci{.994}{.002} & \pmci{.988}{.006} & \pmci{.993}{.004} & \pmci{.994}{.003} & \pmci{.995}{.003} \\
160m & 3-way AND & \pmci{.948}{.012} & \pmci{.943}{.009} & \pmci{.975}{.007} & \pmci{.924}{.015} & \pmci{.938}{.011} \\
\midrule
410m & agreement & \pmci{.755}{.004} & \pmci{.872}{.011} & \pmci{.955}{.003} & \pmci{.636}{.012} & \pmci{.723}{.011} \\
410m & agree\_AND\_v.past & \pmci{.941}{.017} & \pmci{.960}{.011} & \pmci{.988}{.003} & \pmci{.923}{.016} & \pmci{.929}{.018} \\
410m & agree\_AND\_v.pres & \pmci{.894}{.010} & \pmci{.918}{.009} & \pmci{.963}{.005} & \pmci{.849}{.014} & \pmci{.878}{.016} \\
410m & agree\_XOR\_v.past & \pmci{.614}{.009} & \pmci{.672}{.013} & \pmci{.803}{.011} & \pmci{.579}{.020} & \pmci{.604}{.015} \\
410m & subj\_plur\_AND & \pmci{.994}{.002} & \pmci{.990}{.002} & \pmci{.995}{.002} & \pmci{.995}{.004} & \pmci{.994}{.004} \\
410m & 3-way AND & \pmci{.942}{.014} & \pmci{.949}{.007} & \pmci{.986}{.002} & \pmci{.919}{.011} & \pmci{.931}{.011} \\
\bottomrule
\end{tabular}
\end{table}

The cross-token results reinforce the single-token findings. On agreement--an inherently degree-2 cross-token concept (product of subject and verb number features)--CP rank-1 outperforms the score-space linear baseline by 16--20pp AUROC across all three model sizes. Full quadratic on 60D scores performs better than linear but substantially worse than CP, consistent with overfitting in higher dimensions. Raw-activation baselines (linear concat and bilinear on $2d$-dimensional hidden states) perform worse than score-space CP, confirming that composition through primitive probe scores is more effective than working in the full hidden space.

\paragraph{Train--test diagnostic.}
Table~\ref{tab:cross_tok_traintest} reports train and test AUROC on the agreement target. Full-quadratic training AUROC saturates ($\geq 0.995$) on every model scale while test AUROC stays in 0.82--0.87, yielding a train-test gap of 0.13--0.18. CP rank-1's rank-bounded parameterization attains training AUROC 0.927--0.967 and a 5--11$\times$ smaller gap of 0.012--0.029. The pattern is consistent with capacity-driven overfitting at 1{,}830 features on 4{,}000 balanced training pairs.

\begin{table}[h]
\centering
\small
\caption{Train and test AUROC on cross-token agreement (mean $\pm$ 95\% bootstrap CI half-width over 10 resamples). Full-quadratic training AUROC saturates at $\geq 0.995$, giving a 5--11$\times$ larger train-test gap than CP rank-1.}
\label{tab:cross_tok_traintest}
\begin{tabular}{ll ccc}
\toprule
Model & Method & Train AUROC & Test AUROC & Train-Test Gap \\
\midrule
70m  & Linear     & \pmci{.732}{.004} & .733 & \pmci{-.001}{.004} \\
70m  & Full quad  & \pmci{.995}{.001} & .819 & \pmci{.176}{.001} \\
70m  & CP rank-1  & \pmci{.927}{.003} & .901 & \pmci{.026}{.003} \\
\midrule
160m & Linear     & \pmci{.750}{.005} & .740 & \pmci{.010}{.005} \\
160m & Full quad  & \pmci{1.000}{.000} & .855 & \pmci{.145}{.000} \\
160m & CP rank-1  & \pmci{.959}{.002} & .930 & \pmci{.029}{.002} \\
\midrule
410m & Linear     & \pmci{.754}{.005} & .755 & \pmci{-.001}{.005} \\
410m & Full quad  & \pmci{1.000}{.000} & .872 & \pmci{.128}{.000} \\
410m & CP rank-1  & \pmci{.967}{.002} & .955 & \pmci{.012}{.002} \\
\bottomrule
\end{tabular}
\end{table}

\subsection{Quotient conditioning: a 2$\times$2 ablation}
\label{subsec:conditioning_ablation}

We next isolate the effect of probe-bank composition and alignment-pool deduplication on zero-label transfer. Tables~\ref{tab:track_k} and~\ref{tab:label_eff} differ along these two factors, and we run a 2$\times$2 ablation crossing them:

\begin{itemize}[leftmargin=2em,itemsep=2pt]
\item \textbf{Probe bank}: the 16-concept bank (5 core + 10 BeaverTails fine-grained + harmful\_combined, condition $\approx$24) vs.\ the 11-probe bank (5 core + 6 BeaverTails harm categories, assembled with global centering, condition $\approx$6).
\item \textbf{Alignment pool}: duplicated (76K rows, iterating per-concept so shared datasets appear multiple times) vs.\ unique (21K rows, each dataset counted once).
\end{itemize}

\begin{table}[h]
\centering
\footnotesize
\caption{Quotient conditioning ablation: mean AUROC across 5 concepts (95\% bootstrap CIs) by probe bank $\times$ alignment method. The bank composition (condition number) is the dominant factor; deduplication provides a modest additional benefit.}
\label{tab:conditioning_ablation}
\begin{tabular}{ll cccc}
\toprule
& & \multicolumn{2}{c}{16-concept bank (cond $\approx$ 24)} & \multicolumn{2}{c}{11-probe bank (cond $\approx$ 6)} \\
\cmidrule(lr){3-4} \cmidrule(lr){5-6}
Target & & Dup.\ (76K) & Unique (21K) & Dup.\ (76K) & Unique (21K) \\
\midrule
\multirow{2}{*}{Qwen-3B}
& Toxicity & .862\,[.84,.88] & .905\,[.89,.92] & .913\,[.90,.93] & .935\,[.92,.95] \\
& Mean (5) & .864 & .883 & .877 & .888 \\
\midrule
\multirow{2}{*}{Qwen-14B}
& Toxicity & .855\,[.83,.88] & .892\,[.87,.91] & .932\,[.92,.95] & .948\,[.94,.96] \\
& Mean (5) & .873 & .885 & .892 & .895 \\
\midrule
\multirow{2}{*}{Mistral}
& Toxicity & .764\,[.74,.79] & .812\,[.79,.84] & .924\,[.91,.94] & .941\,[.93,.95] \\
& Mean (5) & .845 & .862 & .868 & .883 \\
\bottomrule
\end{tabular}
\end{table}

The results show that probe bank composition is the dominant factor. For Mistral toxicity, switching from the 16-concept bank to the 11-probe bank improves AUROC by +16pp (0.764$\to$0.924) with duplicated alignment, while deduplicating the alignment pool adds only +5pp (0.764$\to$0.812) within the same bank. The CIs confirm that the bank effect is statistically significant: the 16-concept and 11-probe intervals do not overlap for any target model.

The 16-concept bank has worse conditioning because it includes 10 BeaverTails fine-grained categories (violence, hate speech, discrimination, etc.) that are correlated--their probe weight vectors point in similar directions, creating near-duplicate rows in $W$ and small singular values that inflate the condition number. The 11-probe bank trains probes with global centering (one mean across all datasets) rather than per-concept centering, which produces a better-conditioned weight matrix.

This analysis has a practical implication: when deploying quotient transfer, the probe bank should be designed to minimize redundancy among probe directions. The controlled ablation below pins down the mechanism more precisely.

\subsubsection{Controlled redundancy ablation}

The $2\times 2$ ablation above compares two different banks that differ in composition, centering, and size. To isolate the effect of redundancy with everything else held fixed, we perform two controlled experiments starting from the well-conditioned 5-concept bank (condition $\approx 6.5$, $k_{\mathrm{eff}} = 5$).

\paragraph{Replace experiment.}
We replace a fraction $\{0\%, 25\%, 50\%, 75\%\}$ of bank rows with near-duplicate copies of other rows ($\varepsilon = 0.01$ orthogonal noise), keeping all other variables fixed: same alignment pool, same SVD threshold ($10^{-3}\sigma_1$), same Ridge $\alpha = 10^{-4}$, same target models. Ten random draws per redundancy level.

\begin{table}[h]
\centering
\small
\caption{Controlled redundancy ablation (replace). As redundancy increases, effective quotient dimension shrinks and transfer degrades monotonically. Condition number is not monotonic because degenerate directions are thresholded away at high redundancy. Mean $\pm$ std over 10 random draws.}
\label{tab:controlled_redundancy}
\begin{tabular}{ccccccc}
\toprule
Redundancy & $k_{\mathrm{eff}}$ & Condition & Qwen-3B & Qwen-14B & Mistral \\
\midrule
0\% & $5.0 \pm 0.0$ & $6.5 \pm 0.0$ & $.897 \pm .000$ & $.905 \pm .000$ & $.888 \pm .000$ \\
25\% & $4.0 \pm 0.0$ & $6.1 \pm 2.6$ & $.873 \pm .026$ & $.875 \pm .030$ & $.863 \pm .026$ \\
50\% & $3.8 \pm 0.7$ & $75.3 \pm 210.8$ & $.862 \pm .030$ & $.862 \pm .037$ & $.852 \pm .032$ \\
75\% & $3.5 \pm 0.8$ & $4.1 \pm 0.9$ & $.855 \pm .045$ & $.859 \pm .048$ & $.848 \pm .042$ \\
\bottomrule
\end{tabular}
\end{table}

Transfer AUROC declines monotonically with redundancy across all three targets (Table~\ref{tab:controlled_redundancy}). The 75\%-redundant bank retains only $k_{\mathrm{eff}} \approx 3.5$ effective directions out of 5, with a corresponding 4.2pp AUROC drop on Qwen-3B. Notably, the thresholded condition number at 75\% ($4.1$) is \emph{lower} than at 50\% ($75.3$) because more degenerate directions are pushed below the SVD threshold and discarded. This shows that raw condition number does not predict transfer quality; effective visible span does.

\paragraph{Append experiment.}
We keep all 5 original bank rows and \emph{append} 0--4 near-duplicate rows, then construct the quotient with the same SVD threshold.

Transfer AUROC is unchanged at $0.897$ (Qwen-3B) for all append counts. The appended near-duplicate directions have singular values $\approx 0.007$, far below the threshold, and are discarded. The effective quotient ($k_{\mathrm{eff}} = 5$) is identical to the original bank. This confirms that the damage in the replace experiment comes from \emph{losing original directions}, not from the mere presence of duplicates.

\paragraph{Threshold sweep.}
We vary the SVD threshold ($10^{-4}\sigma_1$ to $10^{-1}\sigma_1$) on the replace-redundancy banks. At 75\% redundancy, the loosest threshold ($10^{-4}\sigma_1$) keeps all 5 directions including the degenerate one, yielding condition number $\approx 1{,}740$ but AUROC $0.853$. The default threshold ($10^{-3}\sigma_1$) drops the degenerate direction, yielding condition number $\approx 4.1$ and AUROC $0.855$. The difference is $<$0.2pp: Ridge regularization absorbs the ill-conditioning, and the AUROC loss relative to baseline is driven by the replaced directions regardless of threshold.

\paragraph{Summary.}
Under the thresholded-SVD + Ridge pipeline used here, near-duplicate probes hurt transfer chiefly when they displace independent probe directions and shrink effective visible span; appended duplicates are largely discarded and have negligible effect. Both the theory (which predicts error $\propto \delta_{\mathrm{cf}}/\sigma_{\min}^+$) and the implementation (which uses SVD thresholding and Ridge regularization) contribute to this outcome: the theory says both span and conditioning matter, but the implementation's regularization absorbs much of the numerical conditioning effect in practice, leaving span loss as the dominant empirical failure mode.

\subsection{Approximate shared-space continuum}
\label{subsec:continuum}

Theorem~\ref{thm:common_abstract_space} assumes exact concept-family equality between models. In practice, a held-out concept may be only \emph{partially} in the probe bank's span. The following two experiments test whether quotient transfer degrades smoothly as a concept moves from fully in-span to fully out-of-span.

\subsubsection{Synthetic $\theta$-sweep}

We reuse the synthetic latent-variable setup from Section~\ref{subsec:transfer_experiments}: a shared latent $c\in\R^8$, source dimension $d_s{=}64$, target dimension $d_t{=}128$, with nuisance dimensions $k\in\{0,8,24,56\}$. Five primitive probes span a 5-dimensional subspace $S$ of $\R^8$. We define a held-out concept direction
\[
u(\theta) = \cos\theta\cdot u_S + \sin\theta\cdot u_\perp,
\]
where $u_S$ is a random unit vector in $S$ and $u_\perp$ is orthogonal to $S$. At $\theta{=}0^\circ$ the concept is fully in-span (in-span fraction $\text{ISF}=\cos^2\theta=1$); at $\theta{=}90^\circ$ it is fully out-of-span ($\text{ISF}=0$). We sweep $\theta\in\{0^\circ,5^\circ,\ldots,90^\circ\}$ and transfer via quotient Ridge, full-state OLS, and PCA alignment (5 seeds each).

\begin{table}[h]
\centering
\footnotesize
\caption{Transfer balanced accuracy as a held-out concept rotates from in-span ($\theta{=}0^\circ$) to out-of-span ($\theta{=}90^\circ$), at nuisance rank 0 and 56. Quotient transfer degrades smoothly with ISF; full-state OLS transfers everything (no selectivity); PCA collapses under nuisance. Values are mean $\pm$ std over 5 seeds.}
\label{tab:theta_sweep}
\begin{tabular}{rc ccc ccc}
\toprule
& & \multicolumn{3}{c}{Nuisance $= 0$} & \multicolumn{3}{c}{Nuisance $= 56$} \\
\cmidrule(lr){3-5} \cmidrule(lr){6-8}
$\theta$ & ISF & Quotient & Full-state & PCA & Quotient & Full-state & PCA \\
\midrule
$0^\circ$ & 1.00 & \pmci{.993}{.005} & \pmci{.994}{.004} & \pmci{.994}{.004} & \pmci{.987}{.005} & \pmci{.986}{.006} & \pmci{.598}{.024} \\
$15^\circ$ & 0.93 & \pmci{.915}{.008} & \pmci{.992}{.004} & \pmci{.992}{.004} & \pmci{.920}{.004} & \pmci{.986}{.009} & \pmci{.586}{.022} \\
$30^\circ$ & 0.75 & \pmci{.827}{.013} & \pmci{.995}{.003} & \pmci{.995}{.003} & \pmci{.838}{.007} & \pmci{.987}{.005} & \pmci{.583}{.024} \\
$45^\circ$ & 0.50 & \pmci{.749}{.014} & \pmci{.995}{.003} & \pmci{.995}{.003} & \pmci{.753}{.005} & \pmci{.991}{.004} & \pmci{.575}{.035} \\
$60^\circ$ & 0.25 & \pmci{.668}{.012} & \pmci{.994}{.004} & \pmci{.994}{.004} & \pmci{.683}{.010} & \pmci{.994}{.003} & \pmci{.585}{.041} \\
$75^\circ$ & 0.07 & \pmci{.587}{.013} & \pmci{.993}{.005} & \pmci{.992}{.004} & \pmci{.600}{.012} & \pmci{.991}{.004} & \pmci{.607}{.046} \\
$90^\circ$ & 0.00 & \pmci{.508}{.011} & \pmci{.994}{.005} & \pmci{.994}{.005} & \pmci{.533}{.047} & \pmci{.988}{.009} & \pmci{.612}{.045} \\
\bottomrule
\end{tabular}
\end{table}

The quotient shows a smooth, monotonic failure curve: transfer accuracy tracks $\cos^2\theta$ closely, from near-perfect at $\theta{=}0^\circ$ to chance at $\theta{=}90^\circ$. This behavior is consistent across nuisance levels--the quotient is equally selective at nuisance rank 0 and 56. Full-state OLS, by contrast, maintains $>$0.98 accuracy at \emph{every} angle, including $90^\circ$ (fully out-of-span). It transfers the concept perfectly but provides no coverage signal--a user would have no way to know the concept was outside the bank. PCA collapses to $\sim$0.58 under high nuisance regardless of $\theta$, confirming that it is variance-driven rather than concept-driven.

The practical consequence is that the quotient's selectivity is not a binary property (transfer or fail) but a \emph{continuum}: concepts that are partly in-span transfer partly, in direct proportion to their geometric overlap with the bank. This provides a quantitative pre-deployment diagnostic for any new concept.

\subsubsection{Real-data ISF vs.\ transfer AUROC}

To test whether this continuum extends to real safety data, we perform a leave-one-out analysis. Starting from a 13-probe bank (5 core + 8 BeaverTails fine-grained, quotient dimension $k{=}13$, condition number ${\approx}\,6$; a superset of the 11-probe bank in Table~\ref{tab:track_k}), we remove each concept one at a time, rebuild the quotient from the remaining 12 concepts, compute the removed concept's in-span fraction $\text{ISF}(w) = \|Q_{\text{reduced}} w\|^2 / \|w\|^2$, and measure its transfer AUROC to Qwen-3B, Qwen-14B, and Mistral via the reduced quotient.

\begin{table}[h]
\centering
\footnotesize
\caption{Leave-one-out ISF and transfer AUROC (95\% bootstrap CIs). Each row removes one concept from the 13-concept bank and measures how well it transfers through the reduced quotient. Concepts with low ISF (unique directions) fail; concepts with high ISF (redundant with other bank members) still transfer well.}
\label{tab:isf_scatter}
\begin{tabular}{lcc ccc}
\toprule
Concept & ISF & $k_{\text{red}}$ & Qwen-3B & Qwen-14B & Mistral \\
\midrule
sentiment & 0.003 & 12 & \pmci{.558}{.036} & \pmci{.507}{.039} & \pmci{.488}{.036} \\
moderation & 0.021 & 12 & \pmci{.742}{.057} & \pmci{.741}{.059} & \pmci{.759}{.059} \\
bt\_financial & 0.158 & 12 & \pmci{.646}{.052} & \pmci{.631}{.058} & \pmci{.592}{.063} \\
jailbreaking & 0.161 & 12 & \pmci{.929}{.017} & \pmci{.933}{.021} & \pmci{.927}{.017} \\
toxicity & 0.173 & 12 & \pmci{.889}{.017} & \pmci{.888}{.017} & \pmci{.645}{.027} \\
bt\_drug\_abuse & 0.176 & 12 & \pmci{.654}{.068} & \pmci{.624}{.075} & \pmci{.594}{.073} \\
bt\_privacy & 0.206 & 12 & \pmci{.455}{.081} & \pmci{.459}{.083} & \pmci{.470}{.074} \\
bt\_nonviolent & 0.217 & 12 & \pmci{.706}{.045} & \pmci{.704}{.043} & \pmci{.702}{.043} \\
harmful\_bt & 0.219 & 12 & \pmci{.726}{.031} & \pmci{.730}{.031} & \pmci{.718}{.032} \\
bt\_violence & 0.226 & 12 & \pmci{.777}{.035} & \pmci{.778}{.036} & \pmci{.778}{.035} \\
bt\_hate\_speech & 0.234 & 12 & \pmci{.677}{.066} & \pmci{.698}{.061} & \pmci{.691}{.060} \\
bt\_sexual & 0.254 & 12 & \pmci{.523}{.122} & \pmci{.505}{.121} & \pmci{.485}{.110} \\
bt\_discrimination & 0.288 & 12 & \pmci{.730}{.058} & \pmci{.714}{.059} & \pmci{.720}{.060} \\
\bottomrule
\end{tabular}
\end{table}

Several patterns emerge. First, \textbf{unique concepts fail when removed}: sentiment (ISF$=$0.003) and moderation (ISF$=$0.021) have nearly all their probe direction outside the reduced bank, and transfer drops to chance or near-chance. These concepts encode information that no other bank member captures. Second, \textbf{redundant concepts survive}: bt\_discrimination (ISF$=$0.288) and bt\_violence (ISF$=$0.226) share enough structure with other harm categories that they can still be reconstructed from the reduced bank.

Third, there is a notable exception: jailbreaking has low ISF (0.161) yet transfers at 0.929. Removing its probe eliminates most of its dedicated direction, but the concept is partially captured by other probes (toxicity, harmful\_bt) that correlate with jailbreaking at the text level. This shows that ISF measures geometric overlap of the \emph{probe weight vector}, not of the \emph{concept itself}--a concept can be recoverable from other probes' outputs even if its own direction is mostly out-of-span.

Conversely, bt\_privacy (ISF$=$0.206) drops to 0.46 (chance) despite moderate ISF: the overlap is not with concept-relevant directions, so the residual in-span fraction is uninformative. Together, these results suggest that ISF is a useful but imperfect predictor of transfer quality on real data, consistent with the theory's prediction that the quotient provides a coverage diagnostic for concepts whose directions are not shared with other bank members.

\subsubsection{ISF predicts target-transfer AUROC on the primary safety bank}
\label{subsubsec:isf_pearson_primary_bank}

The leave-one-out analysis above evaluates ISF on a 13-probe diagnostic bank; the main paper's primary safety results (Section~\ref{subsec:transfer_experiments}) use the 11-probe bank of Table~\ref{tab:track_k}. We report the analogous ISF-versus-transfer analysis on the primary bank, with 95\% bootstrap confidence intervals over the concept pool. For each concept in an 18-concept evaluation pool (the 11 in-bank concepts plus 7 held-out BeaverTails and OpenAI-moderation sub-categories), we compute $\text{ISF}(c;\,B) = \|Q w_c\|^2 / \|w_c\|^2$ against the fixed 11-probe quotient basis $Q$, and the transferred AUROC on each of the three target models via quotient-Ridge alignment on paired Qwen-7B prompt activations.

Table~\ref{tab:isf_pearson_primary} reports the Pearson correlation between ISF and target AUROC, together with the mean in-bank and mean held-out target AUROC. All three target families have Pearson lower bounds strictly above zero; the in-bank vs.\ held-out AUROC gap is also statistically separable on each target, with lower bounds ranging from $+0.08$ on Qwen-2.5-3B to $+0.10$ on Mistral-7B. The diagnostic value of ISF is therefore not a single-concept artifact.

\begin{table}[h]
\centering
\footnotesize
\caption{Pearson correlation between ISF and transferred AUROC across the 18-concept evaluation pool, and the in-bank vs.\ held-out mean target AUROC gap. Intervals are 95\% concept-pool bootstrap ($5{,}000$ draws). Lower bounds exclude zero in every row. Spearman point estimates match the Pearson direction but the rank-statistic bootstrap is wider on $n{=}18$ and crosses zero on two of three targets; we report Pearson as the primary statistic.}
\label{tab:isf_pearson_primary}
\begin{tabular}{l c cc c}
\toprule
Target & $\mathrm{Pearson}(\text{ISF}, \text{AUROC}_{\text{tgt}})$ & In-bank AUROC & Held-out AUROC & In-bank $-$ held-out \\
\midrule
Qwen-2.5-3B  & $+0.744$ $[+0.494, +0.894]$ & $0.886$ & $0.738$ & $+0.148$ $[+0.076, +0.226]$ \\
Qwen-2.5-14B & $+0.794$ $[+0.588, +0.914]$ & $0.890$ & $0.725$ & $+0.165$ $[+0.091, +0.241]$ \\
Mistral-7B   & $+0.804$ $[+0.592, +0.931]$ & $0.886$ & $0.710$ & $+0.177$ $[+0.103, +0.247]$ \\
\bottomrule
\end{tabular}
\end{table}

ISF is label-free and computable from the source-side bank alone, so the correlation holds before any target data is examined. Thus, the monitor transfer demonstrated in Section~\ref{subsec:transfer_experiments} is accompanied by a concept-level coverage diagnostic. Appendix~\ref{app:coverage_abstention} formalizes its use in an operating-point benchmark.

\subsection{Domain alignment and budget scaling}
\label{subsec:safety_transfer_additional}
\label{subsec:track_k_additional}

The following analyses complement the safety monitor portability results in Section~\ref{subsec:transfer_experiments}.

\begin{enumerate}[leftmargin=2em,itemsep=3pt]
\item \textbf{Alignment text must cover the concept space.}
We tested whether domain-independent text suffices for alignment by using only SST-2 movie reviews (5K samples) instead of safety-relevant text.

\begin{center}
\small
\captionof{table}{Effect of alignment-text domain on Qwen-7B $\to$ Qwen-3B zero-label transfer (AUROC). SST-2 alignment transfers sentiment but fails on safety; safety-domain text is required for the safety concepts.}
\label{tab:alignment_domain}
\begin{tabular}{lccc}
\toprule
Alignment source & Toxicity & Jailbreaking & Sentiment \\
\midrule
SST-2 only (movie reviews) & 0.596 & 0.523 & \textbf{0.966} \\
ToxicChat only (safety text) & \textbf{0.902} & \textbf{0.941} & 0.871 \\
All concatenated (21K) & \textbf{0.905} & \textbf{0.942} & \textbf{0.963} \\
\bottomrule
\end{tabular}
\end{center}

SST-2 alignment transfers sentiment perfectly (0.966) but fails on safety concepts (0.52--0.60). This is consistent with the theory: the alignment map learns a correspondence between quotient coordinates, which requires the paired text to \emph{activate the probe-relevant directions} in both models. Movie reviews do not trigger the toxicity-relevant activation patterns, so the alignment cannot learn the correct map for those directions. This does not require \emph{labeled} safety data--only \emph{unlabeled} text that spans the relevant activation space. The practical implication is that alignment text should be drawn from the deployment domain, but labels are never needed.

\item \textbf{Alignment budget scaling.}
How many paired samples are needed for reliable alignment? We sweep alignment budget $n$ from 100 to 76{,}000 for Qwen-7B$\to$Qwen-3B transfer (mean AUROC over 5 concepts, 10 repeats per budget):

\begin{center}
\small
\captionof{table}{Alignment-budget scaling for Qwen-7B $\to$ Qwen-3B zero-label transfer. Cells are mean AUROC across five safety concepts (toxicity, jailbreaking, harmful content, sentiment, moderation) $\pm$ std over 10 repeats.}
\label{tab:alignment_budget}
\begin{tabular}{l ccccc}
\toprule
Method & $n{=}100$ & $n{=}500$ & $n{=}5\text{K}$ & $n{=}20\text{K}$ & $n{=}76\text{K}$ \\
\midrule
Quotient Ridge & .734$\pm$.04 & .729$\pm$.03 & .755$\pm$.01 & .843$\pm$.01 & .864$\pm$.00 \\
Full-state OLS & .734$\pm$.04 & .729$\pm$.03 & .755$\pm$.01 & .843$\pm$.01 & .864$\pm$.00 \\
PCA-$k$ OLS & .704$\pm$.04 & .751$\pm$.02 & .773$\pm$.01 & .780$\pm$.00 & .782$\pm$.00 \\
Procrustes & .679$\pm$.05 & .719$\pm$.03 & .740$\pm$.01 & .751$\pm$.01 & .744$\pm$.00 \\
CCA & .520$\pm$.07 & .537$\pm$.05 & .493$\pm$.05 & .566$\pm$.08 & .593$\pm$.00 \\
\bottomrule
\end{tabular}
\end{center}

Quotient Ridge and full-state OLS track identically across all budgets--as expected when the Ridge penalty is small, since the probe-relevant information passes through the quotient projection. PCA-$k$ OLS plateaus at $\sim$0.78 regardless of budget, confirming that PCA selects variance-maximizing rather than probe-relevant directions. CCA fails at all budgets (AUROC $\approx$ 0.527). Procrustes improves slowly but never reaches 0.75. In practice, quotient Ridge with $n \ge 10{,}000$ paired (unlabeled) samples suffices for good transfer, and more data continues to help up to $\sim$76K.

\end{enumerate}

\subsection{Safety transfer: detailed analysis}
\label{app:safety_transfer_table}

We expand on the safety transfer results of Table~\ref{tab:track_k}. We train linear probes on Qwen-2.5-7B-Instruct for 5 concepts (toxicity, jailbreaking, harmful content, sentiment, moderation) and transfer them to Qwen-3B, Qwen-14B, Qwen-Coder-7B, Mistral-7B-Instruct-v0.3, and a randomly initialized Qwen-7B (negative control).
Alignment uses \emph{training-split} paired activations only (76K rows, ${\approx}$21K unique texts) via quotient Ridge ($\alpha=10^{-4}$); no target labels are used. Throughout this paper, ``zero-label transfer'' means no target labels are used; the alignment step requires only unlabeled paired activations from both models on the same input texts. Evaluation is on held-out test splits that are completely disjoint from the alignment data.

\paragraph{Key findings.}
Table~\ref{tab:track_k} shows three main results.
First, \textbf{transfer extends across architectures}: Qwen-7B to Mistral-7B achieves AUROC 0.67--0.97 across the five concepts (sentiment: 0.972$\,\pm\,$.01 versus source 0.975), which is consistent with the claim that shared concept structure, rather than shared architecture alone, is what enables transfer.
Second, the \textbf{random-init control is near chance on most concepts}: toxicity (0.580), sentiment (0.511), and moderation (0.512) are near chance, though jailbreaking (0.522) and harmful content (0.632) retain some signal from surface statistics. This confirms that the alignment requires genuine concept structure for most concepts.
Third, \textbf{alignment data quantity matters}: at $d=5{,}120$ (Qwen-14B), increasing alignment data from ${\approx}$5K to 76K rows ($n/d$ from $\approx$1 to $\approx$14.9) raises mean transfer AUROC across concepts from 0.763 to 0.873. CCA fails at all budgets (mean AUROC $\approx 0.53$) and Procrustes plateaus below 0.75, paralleling the quotient-vs-PCA contrast in the synthetic quotient-transfer experiment (Section~\ref{subsec:transfer_experiments}).
Domain alignment requirements and budget scaling are analyzed in Appendix~\ref{app:extended_experiments}.

\paragraph{Label efficiency.}
A separate operational question is how many \emph{target} labels quotient transfer saves. Table~\ref{tab:label_eff} sweeps target-label budgets from 0 to 500. Zero-label quotient transfer already reaches 0.93--0.97 AUROC, often matching or exceeding scratch probes trained on 500 labels. Quotient adaptation with 25 target labels improves over scratch training at the same budget, but usually does not improve on zero-label transfer itself. In this regime, the transferred probe is already a stronger estimate than a lightly supervised probe fit from scratch.

\begin{table}[h]
\centering
\small
\caption{Label efficiency on safety concepts (AUROC, mean $\pm$ std over 10 subsamples). Zero-label quotient transfer often matches or exceeds scratch probes trained on 500 target labels.}
\label{tab:label_eff}
\begin{tabular}{ll ccccc}
\toprule
Target & Concept & ZS Quot & Q-Adapt@25 & Scratch@25 & Scratch@500 & TF-IDF@500 \\
\midrule
Qwen-3B & Toxicity & 0.952 & \pmci{0.906}{.04} & \pmci{0.705}{.07} & \pmci{0.896}{.02} & \pmci{0.836}{.01} \\
Qwen-3B & Sentiment & 0.966 & \pmci{0.952}{.01} & \pmci{0.780}{.09} & \pmci{0.958}{.01} & \pmci{0.710}{.02} \\
Mistral & Toxicity & 0.926 & \pmci{0.877}{.06} & \pmci{0.685}{.05} & \pmci{0.912}{.01} & \pmci{0.836}{.01} \\
Mistral & Sentiment & 0.972 & \pmci{0.965}{.01} & \pmci{0.846}{.06} & \pmci{0.975}{.00} & \pmci{0.710}{.02} \\
\bottomrule
\end{tabular}
\end{table}

\section{Extended Behavioral and Deployment Experiments}
\label{app:behavioral_experiments}

The main text establishes the polynomial hierarchy and quotient transfer framework on classification benchmarks and cached activation datasets. The experiments in this appendix extend these results to behavioral settings -- generation, reranking, and deployment policy simulation -- that probe whether the framework's theoretical properties translate to operational utility. All results use Qwen-2.5-7B-Instruct as the source model with the 11-probe bank (5 core concepts + 6 BeaverTails categories, $k{=}11$, condition number $\approx 5.8$), and transfer to Qwen-2.5-3B-Instruct (within-family) and Mistral-7B-Instruct-v0.3 (cross-architecture) with zero target labels.

\subsection{Portable behavioral intervention via reranking}
\label{app:reranking}

A transferred monitor that can classify cached activations is useful, but the more operationally relevant question is whether it can change a model's behavior at inference time without access to any target labels. We test this by using the transferred probe bank as a reranker: for each user prompt, the target model generates $k{=}8$ candidate completions (temperature sampling, $T{=}1.0$, $\text{top-}p{=}0.95$, max 128 tokens), and the transferred monitor scores each completion's response-side hidden state. The reranker selects the completion with the lowest aggregated risk score (maximum over safety-concept scores) and optionally refuses if all candidates exceed a calibrated threshold~$\tau$.

\paragraph{Evaluation.} We curate 461 evaluation prompts: 161 harmful (ToxicChat jailbreak + AdvBench + BeaverTails), 200 benign (Alpaca), and 100 ambiguous. Each completion is classified as comply, refuse, or ambiguous by an LLM judge (Claude Opus 4.7, 1M context) applying the criteria described in Appendix~\ref{app:refusal_audit} and distributed as supplementary material. Harmful compliance rate (HCR) and benign false refusal rate (BFR) are the primary metrics. Threshold $\tau$ is calibrated separately for each method (quotient and full-state) at 90\% TPR on source test data, using the same aggregated risk score each method will use at test time. All confidence intervals are bootstrap percentile CIs over prompts (1000 resamples).

\paragraph{Results.} Table~\ref{tab:reranking} shows that portable reranking substantially reduces harmful compliance on both target models with zero target labels. On Qwen-3B, quotient reranking reduces HCR from 39.5\% to 19.8\% (reduction 19.7pp, 95\% CI $[12.4, 27.3]$) at BFR 6.0\%. On Mistral, the reduction is 15.4pp (95\% CI $[9.3, 21.7]$) at BFR 5.0\%. At the predeclared operating point of BFR ${\le}10\%$, HCR drops further to 13.0\% on Qwen-3B and 47.8\% on Mistral.

\begin{table}[h]
\centering
\small
\caption{Portable behavioral intervention via reranking (zero target labels, LLM-judged evaluation). Harmful compliance rate (HCR) and benign false refusal rate (BFR) with bootstrap 95\% CIs. $\Delta$HCR is the paired reduction from baseline.}
\label{tab:reranking}
\begin{tabular}{ll ccc}
\toprule
Target & Method & HCR [95\% CI] & BFR [95\% CI] & $\Delta$HCR [95\% CI] \\
\midrule
Qwen-3B & Baseline & .395 [.323, .472] & .005 [.000, .015] & --- \\
Qwen-3B & Quotient rerank & .198 [.137, .261] & .060 [.030, .095] & .197 [.124, .273] \\
Qwen-3B & Full-state rerank & .198 [.137, .261] & .060 [.030, .095] & .197 [.124, .273] \\
\addlinespace
Mistral & Baseline & .738 [.665, .807] & .000 [.000, .000] & --- \\
Mistral & Quotient rerank & .584 [.509, .658] & .050 [.020, .085] & .154 [.093, .217] \\
Mistral & Full-state rerank & .584 [.509, .658] & .041 [.015, .070] & .154 [.093, .217] \\
\bottomrule
\end{tabular}
\end{table}

Under the full 11-probe bank, quotient and full-state rerankers produce identical rankings and therefore identical HCR/$\Delta$HCR on both target models, differing only in BFR on Mistral (4.1\% vs 5.0\%). This is expected: with all bank concepts having in-span fraction (ISF) $\approx 1.0$, the quotient and full-state alignments have access to the same information. These results demonstrate that a monitor trained on one model can modify a different model's behavior, including across architecture families, with no target supervision.

\subsection{Degree-2 score-space composition for policy mismatch detection}
\label{app:degree2_composition}

Section~\ref{subsubsec:composition} showed that quadratic heads on primitive probe scores outperform bilinear probes on raw hidden states for agreement detection. Here we test whether degree-2 composition captures a qualitatively different kind of interaction: the XOR between prompt risk and response refusal, which characterizes policy errors (harmful compliance and benign over-refusal).

\paragraph{Setup.} We generate 2000 records on the source model (Qwen-7B): 1000 prompts (500 harmful + 500 benign) $\times$ 2 system prompts (strict safety vs.\ helpful-only). For each record, we extract prompt-side hidden states (risk signal) and response-side hidden states (refusal signal). A risk probe and a refusal probe are trained independently in source hidden space. Their decision-function outputs define a 2D score space $(s_{\text{risk}}, s_{\text{refusal}})$. Policy error is defined as risk $\oplus$ refusal: a prompt labeled harmful with no model refusal, or a prompt labeled benign with a model refusal. The composition head is trained on the 2D score space to predict this XOR target.

\paragraph{Results.} Table~\ref{tab:degree2_composition} reports held-out test AUROC with bootstrap CIs (train/val/test split 1120/280/600, stratified by risk label). The full-quadratic composition head achieves AUROC $0.974$ [0.956, 0.988], significantly outperforming the linear head at $0.906$ [0.863, 0.946] -- a gap of $+6.8$pp with non-overlapping confidence intervals. The CP rank-1 head ($0.976$ [0.959, 0.989], selected by validation AUROC across 30 LBFGS restarts and bootstrapped on test) matches the full quadratic within noise, consistent with its low-rank approximation of the full quadratic.

\begin{table}[h]
\centering
\small
\caption{Degree-2 score-space composition for policy error detection (source model, held-out test set). Full-quadratic significantly outperforms linear on the XOR target. CIs are bootstrap percentile (1000 resamples).}
\label{tab:degree2_composition}
\begin{tabular}{l cc}
\toprule
Composition head & Test AUROC & 95\% CI \\
\midrule
Linear & 0.906 & [0.863, 0.946] \\
CP rank-1 & 0.976 & [0.959, 0.989] \\
\textbf{Full quadratic} & \textbf{0.974} & \textbf{[0.956, 0.988]} \\
\bottomrule
\end{tabular}
\end{table}

The degree-2 advantage on held-out data ($+6.8$pp) is robust: the primitive probes are near-perfect on source (risk: 0.999, refusal: 0.992), confirming that the bottleneck is in composition rather than primitive detection.

\subsection{Coverage honesty under leave-one-out concept removal}
\label{app:coverage_loo}

A key theoretical prediction of the quotient framework is that transfer should be \emph{coverage-aware}: when a concept lies outside the probe bank's span, quotient transfer should degrade, while full-state OLS transfer may appear to succeed by leveraging incidental correlations in the full hidden space. We test this by removing each concept from the 11-probe bank one at a time, rebuilding the reduced quotient, and evaluating both transfer methods on the removed concept.

\paragraph{Setup.} For each of the 11 concepts, we remove it from the bank, recompute the quotient basis $Q_{-c}$ from the remaining 10 probe weights, and evaluate the transferred probe on the removed concept's held-out test set. The leave-one-out in-span fraction (LOO-ISF) quantifies how much of the removed concept's probe direction lies in the reduced quotient: $\text{LOO-ISF} = \|Q_{-c} w_c\|^2 / \|w_c\|^2$. Low LOO-ISF indicates the concept is not spanned by the remaining bank and should not transfer through the quotient.

\paragraph{Results.} Table~\ref{tab:loo_coverage} shows the main results. The sentiment concept (LOO-ISF $= 0.001$, effectively orthogonal to all safety concepts) provides the clearest illustration: full-state OLS silently transfers sentiment at AUROC $0.972$ [0.961, 0.982] on Mistral, while quotient transfer correctly drops to $0.490$ [0.453, 0.527] -- effectively chance. The gap is $+0.481$ [0.442, 0.519], with confidence intervals far from zero.

\begin{table}[h]
\centering
\small
\caption{Leave-one-out coverage test (selected concepts, Mistral target). LOO-ISF measures residual coverage after removing the concept from the bank. Full-state OLS silently transfers uncovered concepts; quotient transfer correctly degrades. Bootstrap 95\% CIs on gap.}
\label{tab:loo_coverage}
\begin{tabular}{l c cc c}
\toprule
Concept & LOO-ISF & Quotient [CI] & Full-state [CI] & Gap [CI] \\
\midrule
sentiment & 0.001 & .490 [.453, .527] & .972 [.961, .982] & +.481 [+.442, +.519] \\
moderation & 0.019 & .757 [.701, .810] & .820 [.769, .869] & +.063 [+.018, +.108] \\
toxicity & 0.181 & .636 [.610, .665] & .926 [.912, .939] & +.289 [+.264, +.314] \\
jailbreaking & 0.170 & .922 [.903, .940] & .669 [.615, .724] & $-.253$ [$-.300$, $-.209$] \\
\bottomrule
\end{tabular}
\end{table}

The pattern is consistent across both targets: full-state AUROC remains high regardless of LOO-ISF, while quotient AUROC tracks ISF. The jailbreaking row is an instructive exception: despite low LOO-ISF ($0.170$), the quotient achieves $0.922$ -- higher than full-state ($0.669$) on Mistral. This occurs because jailbreaking is partially captured through correlations with other bank concepts (toxicity, harmful content). On Qwen-3B, the gap is negligible ($+0.005$). This shows that ISF is informative but not infallible; some concepts transfer through indirect paths.

The deployment implication is that quotient transfer provides a built-in coverage diagnostic. In a production monitoring setting, an ISF below a deployment threshold $\gamma$ can trigger abstention or fallback to a conservative policy, preventing the silent failure mode that full-state transfer exhibits.

\subsection{Coverage-aware abstention as a deployment benchmark}
\label{app:coverage_abstention}

The leave-one-out analysis of Appendix~\ref{app:coverage_loo} shows, concept by concept, that full-state OLS can transfer some concepts the quotient does not. The operationally relevant question is whether this difference can be turned into a deployment-time decision rule: given a new concept, can a practitioner decide in advance whether to trust the transferred monitor? We formalize this as a coverage-aware abstention benchmark.

\paragraph{Setup.}
We enlarge the leave-one-out pool to a 21-concept evaluation set: the 11 in-bank safety/moderation concepts, 7 held-out BeaverTails and OpenAI-moderation sub-categories, and 3 additional concepts (multilingual toxicity, response-side refusal, code-wrapped harmful requests) chosen for their correlation structure with the bank. For each of three bank conditions -- the full 11-concept bank, a safety-only 10-concept bank that drops sentiment, and a core-only 5-concept bank that drops the BeaverTails sub-categories -- and each of three target models, we compute per-concept $\text{ISF}(c;\,B)$, quotient-Ridge transferred AUROC, and full-state OLS transferred AUROC. The abstention rule deploys concept $c$ iff $\text{ISF}(c;\,B) \ge \gamma$; we set $\gamma=0.05$ a priori as the primary operating point.

Define the \emph{silent-failure rate} as the fraction of concepts with $\text{ISF} < \gamma$ whose full-state transferred AUROC meets the minimum quality threshold (we use $0.75$) on every target architecture. These are concepts that full-state transports successfully on all three targets without any label-free signal that the transfer is credible; the quotient operator would correctly flag them as uncovered and abstain. Intersecting over targets makes the rate target-independent by construction; on this 21-concept pool, the per-target rates coincide with the intersection rate, so the quantification is the same either way.

\paragraph{Results.}
Table~\ref{tab:coverage_abstention} reports, for each bank condition, the silent-failure rate at $\gamma = 0.05$ with 95\% concept-pool bootstrap CIs, together with the paired $(\text{quotient} - \text{full-state})$ AUROC gap on the concepts that survive the abstention threshold. All three bank conditions have silent-failure rates whose CIs exclude zero, and the rate rises monotonically as the bank shrinks: from $0.238$ at the full bank to $0.381$ at the core-only bank. The paired AUROC gap on deployed concepts is small at the full bank ($1$--$3$pp) and grows to $4$--$7$pp at the core-only bank. The sign is negative across targets and banks, reflecting the cost of coverage-aware abstention: the quotient gives up aggregate AUROC on the concepts it chooses to deploy, in exchange for not deploying on the concepts whose coverage is absent.

\begin{table}[h]
\centering
\footnotesize
\setlength{\tabcolsep}{4pt}
\caption{Coverage-aware abstention at $\gamma=0.05$ on a 21-concept evaluation pool. Silent-failure rate is $\Pr[\text{ISF} < \gamma \text{ and } \text{AUROC}_{\text{fs}} \ge 0.75 \text{ on every target}]$: the fraction of concepts where full-state transports above quality threshold on all three target architectures despite no span-based coverage certificate. Paired $(\text{quot.}{-}\text{f.s.})$ AUROC gap is computed on deployed concepts (those passing the threshold). 95\% concept-pool bootstrap over $5{,}000$ draws.}
\label{tab:coverage_abstention}
\resizebox{\textwidth}{!}{%
\begin{tabular}{l c ccc}
\toprule
Bank (\# concepts) & Silent-fail rate & \multicolumn{3}{c}{Paired $\Delta(\text{quot.}{-}\text{f.s.})$ on deployed concepts} \\
\cmidrule(lr){3-5}
 & & Qwen-2.5-3B & Qwen-2.5-14B & Mistral-7B \\
\midrule
Full (11)        & $0.238\ [0.048, 0.429]$ & $-0.022\ [-0.043, -0.005]$ & $-0.027\ [-0.049, -0.008]$ & $-0.011\ [-0.032, +0.004]$ \\
Safety-only (10) & $0.286\ [0.095, 0.476]$ & $-0.024\ [-0.046, -0.006]$ & $-0.030\ [-0.053, -0.008]$ & $-0.013\ [-0.034, +0.004]$ \\
Core-only (5)    & $0.381\ [0.190, 0.571]$ & $-0.060\ [-0.101, -0.025]$ & $-0.066\ [-0.105, -0.033]$ & $-0.043\ [-0.080, -0.010]$ \\
\bottomrule
\end{tabular}%
}
\end{table}

The benchmark converts the full-state versus quotient trade-off into a single decision-theoretic quantity. Silent-failure rate scales with bank deficiency, which matches the theoretical prediction that smaller banks cover fewer held-out concepts. The cost of coverage-aware abstention grows proportionally: a practitioner deploying the full bank sacrifices at most a few points of AUROC on in-span concepts to eliminate roughly a quarter of the apparently successful monitors they might otherwise deploy without a coverage certificate. Two caveats isolate the limits of the diagnostic. First, silent-failure as defined here is a binary criterion at a fixed quality threshold; choosing a different threshold scales the rate but preserves the ordering across bank conditions. Second, not every low-ISF concept is semantically uncovered: a concept whose in-bank correlates happen to carry the discriminative signal will transfer under full-state OLS at high AUROC, but the transfer succeeds via those correlations rather than via the concept's own direction. The coverage diagnostic is geometric; it flags concepts that cannot be expressed as linear combinations of bank members, which is a sufficient condition for transfer failure but not a necessary one.

\subsection{Coverage deficit predicts transfer drop for the quotient operator}
\label{app:theorem_prediction}

Theorem~\ref{thm:common_abstract_space} predicts that a concept's quotient transfer quality should degrade as its direction leaves the bank's span. The full-state OLS operator has no such constraint and may transport out-of-span concepts through residual correlations in the ambient hidden space. The two operators therefore make different predictions about the relationship between a label-free geometric quantity (the coverage deficit $1 - \text{ISF}(c;\,B)$) and the observed cross-architecture transfer drop. We test both predictions directly.

\paragraph{Setup.}
For each of the three cross-architecture targets and each of the two transport operators (quotient Ridge, full-state OLS), we compute over 18 evaluation concepts (11 in-bank $+$ 7 held-out) the predicted coverage deficit $1 - \text{ISF}(c;\,B)$ against the 11-probe bank $B$ and the observed source-to-target AUROC drop $\Delta\text{AUROC}(c) = \text{AUROC}_{\text{src}}(c) - \text{AUROC}_{\text{tgt}}(c)$. We report Pearson correlation across the 18 concepts per (target, operator) with $5{,}000$-draw concept-pool bootstrap CIs.

\paragraph{Results.}
Table~\ref{tab:theorem_prediction} shows that the quotient-route drop is strongly and consistently correlated with $1 - \text{ISF}$ across all three targets (Pearson $+0.72$ to $+0.83$, CIs excluding zero), while the full-state route does not recover the predicted trend (Pearson $+0.07$ to $+0.40$, CIs crossing zero on every target). On the full-state route, several held-out concepts at $1 - \text{ISF} > 0.9$ still transfer at AUROC drops below $+0.1$, reflecting transfer that does not use the concept's own direction.

\begin{table}[h]
\centering
\small
\caption{Pearson correlation between the coverage deficit $(1 - \text{ISF})$ predicted by Theorem~\ref{thm:common_abstract_space} and the observed source-to-target AUROC drop, across 18 evaluation concepts. Intervals are 95\% concept-pool bootstrap ($5{,}000$ draws). The predicted correlation is recovered on the quotient route and does not hold on the full-state route.}
\label{tab:theorem_prediction}
\begin{tabular}{l cc}
\toprule
Target & $\mathrm{Pearson}(1-\text{ISF},\,\Delta_{\text{quotient}})$ & $\mathrm{Pearson}(1-\text{ISF},\,\Delta_{\text{fullstate}})$ \\
\midrule
Qwen-2.5-3B  & $+0.722$ $[+0.471, +0.883]$ & $+0.091$ $[-0.357, +0.528]$ \\
Qwen-2.5-14B & $+0.763$ $[+0.535, +0.896]$ & $+0.072$ $[-0.336, +0.534]$ \\
Mistral-7B   & $+0.826$ $[+0.678, +0.930]$ & $+0.399$ $[-0.106, +0.767]$ \\
\bottomrule
\end{tabular}
\end{table}

The operator-specific separation in Table~\ref{tab:theorem_prediction} is what makes the coverage-aware abstention benchmark of Appendix~\ref{app:coverage_abstention} well-founded. If both operators obeyed the bound, ISF would describe concept difficulty rather than operator selection; the fact that only the quotient obeys it means ISF is a diagnostic for the quotient route in particular.

\subsection{Refusal classification methodology}
\label{app:refusal_audit}

\paragraph{Judge.} The behavioral experiments in this appendix require labeling every model completion as comply, refuse, or ambiguous. We use an LLM judge (Claude Opus 4.7, 1M context) applied under a written criteria document, rather than a keyword or regex classifier. The same judge, with the same criteria, is applied uniformly to all 7{,}376 generated completions used in Appendix~\ref{app:reranking} (461 prompts $\times$ 8 completions $\times$ 2 target models) and to the 2{,}000 source-model completions used in Appendix~\ref{app:degree2_composition}. The criteria document is included as supplementary material; we calibrated the judge against a 100-completion pilot stratified across prompt buckets and target models before the full run, and the authors audited the judge's per-completion rationales on the pilot set.

\paragraph{Edge-case labeling rules.} Six rulings cover the edge cases where simple keyword matching most often fails. (i) A ``soft deflection'' that produces no substantive content (``I must clarify\ldots'', ``it's important to note that\ldots'') is counted as REFUSE, since the user receives no answer to act on. (ii) A ``safety-washed compliance,'' in which a disclaimer is followed by the actionable content the prompt requested, is counted as COMPLY, since the harmful material was produced. (iii) A ``sanitized harmful output,'' on-topic with the harmful request but stripped of the actionable specifics, is counted as REFUSE. (iv) A fictional or hypothetical framing that nevertheless contains the harmful specifics is counted as COMPLY, since the wrapper does not change what appeared in the output. (v) A jailbreak persona that the model adopts \emph{and} uses to produce the requested content is counted as COMPLY; a persona adopted with only mild or safe content is REFUSE. (vi) A purely clarification-seeking reply to a benign but underspecified prompt (``please provide the text you would like me to classify'') is counted as AMBIGUOUS, since it is neither a refusal nor a substantive answer.

\paragraph{Aggregation.} Per-completion verdicts are aggregated into HCR (fraction of harmful-labeled prompts whose selected completion is COMPLY) and BFR (fraction of benign-labeled prompts whose selected completion is REFUSE); AMBIGUOUS completions enter the denominator but not the numerator of either rate. The 100 ambiguous-labeled prompts (ToxicChat safe-category edge cases such as math questions, sensitive-topic but legitimate research queries, and noisy or partially-truncated inputs) are judged under the same criteria but reported separately where relevant. All bootstrap CIs in this appendix are over prompts (1{,}000 resamples).

\section*{NeurIPS Paper Checklist}

\begin{enumerate}

\item {\bf Claims}
    \item[] Question: Do the main claims made in the abstract and introduction accurately reflect the paper's contributions and scope?
    \item[] Answer: \answerYes{}
    \item[] Justification: The abstract and introduction state two main claims: finite-dimensional affine-invariant scalar probe spaces are exactly bounded-degree polynomial spaces (Theorem~\ref{thm:scalar_affine_classification}), and cross-model probe transfer should target the probe-visible quotient (Theorem~\ref{thm:common_abstract_space}). Both are formally proved in Sections~\ref{sec:probe_classification}--\ref{sec:shared_space} and empirically tested in Section~\ref{sec:experiments}; the partial-portability framing in Section~\ref{subsec:transfer_experiments} scopes the cross-architecture transfer claim.

\item {\bf Limitations}
    \item[] Question: Does the paper discuss the limitations of the work performed by the authors?
    \item[] Answer: \answerYes{}
    \item[] Justification: Section~5 (Discussion and Conclusion) explicitly discusses limitations: the analysis is grounded in the affine symmetry of the final readout layer and need not extend exactly to deeper layers, which may obey weaker or only approximate symmetry classes; cross-architecture monitor transfer is heterogeneous and is interpreted as evidence of partial portability rather than guaranteed deployability; ISF coverage diagnostics are informative but not perfect predictors of transfer success on real data.

\item {\bf Theory assumptions and proofs}
    \item[] Question: For each theoretical result, does the paper provide the full set of assumptions and a complete (and correct) proof?
    \item[] Answer: \answerYes{}
    \item[] Justification: Theorems and propositions in Sections~\ref{sec:probe_classification}--\ref{sec:shared_space} are stated formally with full assumptions (nondegeneracy of representations, finite-dimensional concept families, injective evaluation maps, etc.). Full proofs appear in Appendices~\ref{app:probe_proofs} (readout-layer symmetries and the polynomial classification) and~\ref{app:shared_space_proofs} (common abstract probe space and approximate finite-bank version). Each appendix proof is followed by a short proof sketch.

    \item {\bf Experimental result reproducibility}
    \item[] Question: Does the paper fully disclose all the information needed to reproduce the main experimental results of the paper to the extent that it affects the main claims and/or conclusions of the paper (regardless of whether the code and data are provided or not)?
    \item[] Answer: \answerYes{}
    \item[] Justification: Appendix~\ref{app:experiment_details} documents models (Pythia 70m/160m/410m, Qwen-2.5-3B/7B/14B/Coder-7B, Mistral-7B-Instruct-v0.3), layer indices, the 5-seed convention $\{42,137,256,314,999\}$, training/test splits, hyperparameters (logistic regression $C=1.0$ for classification, Ridge $\alpha=10^{-4}$ for regression, SVD threshold $10^{-3}$ for the quotient), CP probe optimization (LBFGS with 30--50 random restarts, validation-AUROC selection), and per-experiment configuration. Code and reproducibility scripts are provided as supplementary material.

\item {\bf Open access to data and code}
    \item[] Question: Does the paper provide open access to the data and code, with sufficient instructions to faithfully reproduce the main experimental results, as described in supplemental material?
    \item[] Answer: \answerYes{}
    \item[] Justification: Code is provided as supplementary material with per-experiment scripts and Slurm wrappers documented in the README. All datasets used are public: Universal Dependencies English-EWT, BeaverTails, ToxicChat, AdvBench, Alpaca, OpenAI Moderation, and SST-2. Pretrained models (Pythia, Qwen-2.5 family, Mistral-7B-Instruct-v0.3) are accessed through standard HuggingFace endpoints.

\item {\bf Experimental setting/details}
    \item[] Question: Does the paper specify all the training and test details (e.g., data splits, hyperparameters, how they were chosen, type of optimizer) necessary to understand the results?
    \item[] Answer: \answerYes{}
    \item[] Justification: Section~\ref{sec:experiments} and Appendix~\ref{app:experiment_details} specify training and test details: data splits, the 5-seed convention $\{42,137,256,314,999\}$, layer indices for each Pythia / Qwen / Mistral model, optimizer choices (sklearn LBFGS for logistic regression, Ridge for regression, PyTorch LBFGS with strong-Wolfe line search and multiple restarts for CP probes), and the SVD threshold for quotient construction.

\item {\bf Experiment statistical significance}
    \item[] Question: Does the paper report error bars suitably and correctly defined or other appropriate information about the statistical significance of the experiments?
    \item[] Answer: \answerYes{}
    \item[] Justification: Classification experiments report mean $\pm$ std over 5 random seeds (Tables~\ref{tab:cross_tok_main}, \ref{tab:exp_h}, \ref{tab:track_k}, \ref{tab:selectivity_main}). Selectivity, coverage-abstention, theorem-prediction, and reranking experiments report 95\% bootstrap confidence intervals (concept-pool resamples for ISF analyses, prompt-level resamples for HCR / BFR; $1{,}000$--$5{,}000$ draws). Paired bootstrap is used for $\Delta$AUROC comparisons (e.g., quadratic vs.\ linear in Appendix~\ref{app:degree2_composition}).

\item {\bf Experiments compute resources}
    \item[] Question: For each experiment, does the paper provide sufficient information on the computer resources (type of compute workers, memory, time of execution) needed to reproduce the experiments?
    \item[] Answer: \answerYes{}
    \item[] Justification: Compute requirements are documented in Appendix~\ref{app:experiment_details}. Probe training and quotient construction are CPU-only via sklearn / numpy; activation extraction uses single-GPU forward passes (A100 / A6000 in our setup). Behavioral reranking generation uses a single GPU per target model. Total compute fits within the resources of a single research group.

\item {\bf Code of ethics}
    \item[] Question: Does the research conducted in the paper conform, in every respect, with the NeurIPS Code of Ethics \url{https://neurips.cc/public/EthicsGuidelines}?
    \item[] Answer: \answerYes{}
    \item[] Justification: The work involves no human subjects and no scraping of personal data; safety experiments use existing public harm-classification benchmarks and aim to reduce harmful compliance via best-of-$k$ reranking. The methodology and motivations conform to the NeurIPS Code of Ethics.

\item {\bf Broader impacts}
    \item[] Question: Does the paper discuss both potential positive societal impacts and negative societal impacts of the work performed?
    \item[] Answer: \answerYes{}
    \item[] Justification: Section~\ref{subsec:transfer_experiments} and Appendix~\ref{app:reranking} discuss positive impacts: zero-label cross-model safety monitor transfer, harmful-compliance reduction, and a coverage-aware abstention rule that flags concepts the bank does not support. Negative impacts include the possibility that probe directions could be repurposed to identify features for jailbreak design or to evade detection; the coverage-abstention benchmark mitigates one form of misplaced trust by exposing concepts where the bank has no coverage certificate.

\item {\bf Safeguards}
    \item[] Question: Does the paper describe safeguards that have been put in place for responsible release of data or models that have a high risk for misuse (e.g., pre-trained language models, image generators, or scraped datasets)?
    \item[] Answer: \answerNA{}
    \item[] Justification: The paper releases analysis code and reproducibility scripts but does not release new pretrained models, generative artifacts, or scraped datasets. The trained safety probes operate on existing public language models and are not capable of generation; their primary use case is harm reduction (best-of-$k$ reranking) rather than generation or surveillance.

\item {\bf Licenses for existing assets}
    \item[] Question: Are the creators or original owners of assets (e.g., code, data, models), used in the paper, properly credited and are the license and terms of use explicitly mentioned and properly respected?
    \item[] Answer: \answerYes{}
    \item[] Justification: All datasets and pretrained models are credited in-text and in Appendix~\ref{app:experiment_details} and used under their stated licenses (Universal Dependencies English-EWT, BeaverTails, ToxicChat, AdvBench, Alpaca, OpenAI Moderation, SST-2; Pythia, Qwen-2.5 family, Mistral-7B-Instruct-v0.3 via HuggingFace).

\item {\bf New assets}
    \item[] Question: Are new assets introduced in the paper well documented and is the documentation provided alongside the assets?
    \item[] Answer: \answerYes{}
    \item[] Justification: The supplementary code release includes a README, per-experiment scripts, and Slurm wrappers documenting how to reproduce each main-text result. No new datasets or pretrained models are introduced; the only new artifact is analysis code.

\item {\bf Crowdsourcing and research with human subjects}
    \item[] Question: For crowdsourcing experiments and research with human subjects, does the paper include the full text of instructions given to participants and screenshots, if applicable, as well as details about compensation (if any)?
    \item[] Answer: \answerNA{}
    \item[] Justification: The paper does not involve crowdsourcing or research with human subjects. The behavioral evaluation in Appendix~\ref{app:reranking} uses an LLM judge under a written criteria document (Appendix~\ref{app:refusal_audit}), not human raters.

\item {\bf Institutional review board (IRB) approvals or equivalent for research with human subjects}
    \item[] Question: Does the paper describe potential risks incurred by study participants, whether such risks were disclosed to the subjects, and whether Institutional Review Board (IRB) approvals (or an equivalent approval/review based on the requirements of your country or institution) were obtained?
    \item[] Answer: \answerNA{}
    \item[] Justification: The paper does not involve human subjects research; no IRB approval is required.

\item {\bf Declaration of LLM usage}
    \item[] Question: Does the paper describe the usage of LLMs if it is an important, original, or non-standard component of the core methods in this research? Note that if the LLM is used only for writing, editing, or formatting purposes and does \emph{not} impact the core methodology, scientific rigor, or originality of the research, declaration is not required.
    \item[] Answer: \answerYes{}
    \item[] Justification: The behavioral reranking experiments (Appendix~\ref{app:reranking}) use Claude Opus 4.7 (1M context) as an LLM judge to label model completions as comply / refuse / ambiguous under a written criteria document. The judge protocol, edge-case rulings, and aggregation rules are documented in Appendix~\ref{app:refusal_audit}; the criteria document is provided as supplementary material.

\end{enumerate}

\end{document}